\documentclass{article}

\usepackage{amsmath, amssymb, amsfonts}
\usepackage{amsthm, thmtools, thm-restate}
\usepackage{mathtools}
\usepackage{enumitem}

\usepackage{algorithm, algorithmicx, algpseudocode}

\usepackage{epsfig, graphicx, subfigure}
\usepackage{nicefrac}
\usepackage{bbold}
\usepackage[justification=centering]{caption}
\usepackage{times}
\usepackage{color}
\usepackage{upgreek}

\newtheorem{theorem}{Theorem}[section]
\newtheorem{proposition}[theorem]{Proposition}
\newtheorem{definition}[theorem]{Definition}
\newtheorem{lemma}[theorem]{Lemma}
\newtheorem{corollary}[theorem]{Corollary}

\newtheorem{fact}[theorem]{Fact}

\declaretheoremstyle[
notefont=\bfseries, notebraces={}{},
bodyfont=\normalfont\itshape,
headformat=\NAME \NOTE
]{nopar}

\let\originalparagraph\paragraph
\renewcommand{\paragraph}[2][.]{\originalparagraph{#2#1}}

\usepackage{tikz}
\usetikzlibrary{positioning}
\usetikzlibrary{shapes.geometric}
\usetikzlibrary{fit}
\usepackage{stackengine}
\usepackage{pgfplots}
\pgfplotsset{ticks=none}
\def\shft#1{\stackunder[43pt]{}{\kern112pt #1}}

\tikzset{
	module/.style={draw,thick,  shape=rectangle, rounded corners = 0.5ex, minimum width =2em, minimum height =2em},
	bigmodule/.style={draw, thick, shape=rectangle, rounded corners = 1ex, align=right, minimum width =4em, minimum height =3em, inner sep=2ex}
	}

\newcommand{\sigmoid}{
	\begin{tikzpicture}[declare function={sigma(\x)=1/(1+exp(-\x));}, scale=0.05]
		\begin{axis}%
			[
			line width= 0.5cm,
			xmin=-6,
			xmax=6,
			axis x line=bottom,
			ymax=1,
			axis y line=middle,
			samples=100,
			domain=-6:6,
			xticklabels={,,},
			yticklabels={,,},
			axis line style={draw=none}
			]
			\addplot[blue,mark=none]   (x,{sigma(x)});
			\legend{}
		\end{axis}
\end{tikzpicture}}

\newcommand{\linetikz}{
	\begin{tikzpicture}[declare function={sigma(\x)=x;}, scale=0.05]
		\begin{axis}%
			[
			line width= 0.5cm,
			xmin=-1,
			xmax=1,
			axis x line=bottom,
			ymax=1,
			axis y line=middle,
			samples=100,
			domain=-1:1,
			xticklabels={,,},
			yticklabels={,,},
			axis line style={draw=none}
			]
			\addplot[blue,mark=none]   (x,{sigma(x)});
			\legend{}
		\end{axis}
\end{tikzpicture}}
\newcommand{\reals}{{\mathbb R}}

\newcommand{\ones}{\operatorname{\mathbf 1}}
\newcommand{\idm}{\operatorname{I}}
\newcommand{\id}{\operatorname{I}}

\newcommand{\Rank}{\operatorname{\bf Rank}}
\newcommand{\Tr}{\operatorname{\bf Tr}}
\newcommand{\diag}{\operatorname{\bf diag}}

\newcommand{\Vect}{\operatorname{Vec}}

\newcommand{\Expect}{\operatorname{\mathbb E}}

\DeclareMathOperator*{\argmin}{arg\,min}

\newcommand{\bigO}{\mathcal{O}}
\newcommand{\dom}{\operatorname{dom}}

\def\shortdisplay{\setlength{\abovedisplayskip}{5pt}%
	\setlength{\belowdisplayskip}{5pt}%
	\setlength{\abovedisplayshortskip}{2pt}%
	\setlength{\belowdisplayshortskip}{2pt}}

\let\oldselectfont\selectfont
\def\selectfont{\oldselectfont\shortdisplay}

\newcommand{\var}{w}
\newcommand{\dimvar}{p}
\newcommand{\rand}{x}
\newcommand{\obj}{h}
\newcommand{\Obj}{F}
\newcommand{\reg}{r}
\newcommand{\set}{C}
\newcommand{\lab}{y}
\newcommand{\Lab}{Y}
\newcommand{\labpred}{\hat y}

\newcommand{\setparam}{R}
\newcommand{\numclass}{k}
\newcommand{\loss}{\mathcal{L}}
\newcommand{\nbsamp}{n}

\newcommand{\chain}{f}

\newcommand{\dyn}{\phi}
\newcommand{\horizon}{\tau}
\newcommand{\dimparam}{\rho}
\newcommand{\dimlatent}{\delta}

\newcommand{\dimlabel}{q}

\newcommand{\latent}{z}

\newcommand{\nonlin}{a}
\newcommand{\linearcste}{\beta^0}
\newcommand{\diminter}{\eta}

\newcommand{\Weight}{W}
\newcommand{\weight}{w}
\newcommand{\offset}{w^0}
\newcommand{\patch}{\Pi}
\newcommand{\Latent}{Z}
\newcommand{\nbfilter}{n^f}
\newcommand{\dimfilter}{s^f}
\newcommand{\nbpatch}{n^p}
\newcommand{\activ}{\alpha}
\newcommand{\activelt}{\bar \alpha}
\newcommand{\batchsize}{m}

\newcommand{\biaffine}{b}
\newcommand{\bilinear}{\beta}

\newcommand{\linearlatent}{\beta^\latent}

\newcommand{\diminput}{\delta}

\newcommand{\intervar}{\omega}

\newcommand{\regbatchnorm}{\epsilon}

\newcommand{\normal}{\nu}

\newcommand{\pool}{\pi}
\newcommand{\batchreg}{\epsilon}

\newcommand{\diminputnext}{\tilde \diminput}
\newcommand{\latentnext}{\tilde \latent}

\newcommand{\costate}{\lambda}

\newcommand{\spars}{s}

\newcommand{\normidx}{2}
\newcommand{\opnormidx}{2, 2}
\newcommand{\tensnormidx}{2, 2, 2}

\newcommand{\classfunc}{\mathcal{C}}

\newcommand{\lip}{\ell}
\newcommand{\smooth}{L}
\newcommand{\tensornorm}{\smooth}
\newcommand{\bound}{m}
\newcommand{\diam}{D}

\newcommand{\stepsize}{\gamma}

\newcommand{\lin}{\ell}
\newcommand{\qua}{q}

\newcommand{\A}{A}
\newcommand{\B}{B}
\newcommand{\PP}{P}
\newcommand{\Q}{Q}
\newcommand{\R}{R}
\newcommand{\p}{p}
\newcommand{\q}{q}
\newcommand{\costogo}{\operatorname{cost}}
\newcommand{\CC}{C}
\newcommand{\cc}{c}
\newcommand{\cste}{\operatorname{cste}}
\newcommand{\K}{K}
\newcommand{\kk}{k}
\newcommand{\dualvar}{\mu}

\newcommand{\radius}{\rho}
\newcommand{\nbcomp}{k}

\newcommand{\Relu}{\operatorname{ReLu}}
\newcommand{\conv}{\operatorname{conv}}
\newcommand{\full}{\operatorname{full}}
\newcommand{\softmax}{\operatorname{softmax}}
\newcommand{\softplus}{\operatorname{softplus}}
\newcommand{\batchnorm}{\operatorname{batch}}
\newcommand{\maxpool}{\operatorname{maxpool}}
\newcommand{\avgpool}{\operatorname{avgpool}}
\newcommand{\avg}{\operatorname{avg}}

\newcommand{\impvar}{\beta}
\newcommand{\dimimpvar}{b}
\newcommand{\dimimpfix}{a}
\newcommand{\impfunc}{\zeta}
\newcommand{\Impfunc}{g}
\newcommand{\impfix}{\alpha}
\newcommand{\impgrad}{\xi}
\newcommand{\smoothess}{H}

\newcommand{\state}{x}
\newcommand{\ctrl}{u}
\newcommand{\auxctrl}{v}
\newcommand{\auxxctrl}{w}
\newcommand{\fixedstate}{\bar x}
\newcommand{\dimctrl}{p}
\newcommand{\dimstate}{d}
\newcommand{\ctrls}{\ctrl}
\newcommand{\auxctrls}{\auxctrl}
\newcommand{\auxxctrls}{\auxxctrl}
\newcommand{\nextctrls}{\ctrls_{\text{new}}}

\newcommand{\auxstate}{y}

\newcommand{\chainaux}{g}
\newcommand{\chainauxx}{h}
\newcommand{\dynaux}{\psi}
\newcommand{\dynauxx}{\chi}
\newcommand{\oraclectrls}{v^*}
\newcommand{\oraclectrl}{v^*}
\newcommand{\currctrls}{u}

\newcommand{\chainoutput}{\psi}

\newcommand{\dimin}{d}
\newcommand{\dimout}{m}
\newcommand{\lipp}{l}
\usepackage{color}
\definecolor{darkgreen}{rgb}{0.01, 0.60, 0.14}
\definecolor{plum}{cmyk}{0.00, 0.51, 0.06, 0.44}

\definecolor{orange}{RGB}{255,90,0}
\definecolor{darkorange}{RGB}{180,50,0}

\definecolor{violet}{RGB}{128,0,128}

\definecolor{pink}{rgb}{1, 0, 0.5}

\definecolor{purple}{RGB}{51,0,111}
\definecolor{darkpurple}{RGB}{75,0,130}

\definecolor{darkblue}{RGB}{20,20,100}
\definecolor{mediumblue}{RGB}{65,105,225}
\definecolor{lightblue}{RGB}{0,160,200}
\definecolor{mediumred}{RGB}{178,34,34}
\definecolor{darkred}{RGB}{90,0,0}

\newcommand{\blue}[1]{\textcolor{mediumblue}{#1}}

\pgfplotsset{compat=1.17}

\usepackage[colorlinks, citecolor=mediumblue, linkcolor=darkred, urlcolor=blue, breaklinks]{hyperref}
\usepackage[numbers, square]{natbib}
\usepackage{fullpage}

\title{An Elementary Approach to Convergence Guarantees\\
	of Optimization Algorithms for Deep Networks}
\author{Vincent Roulet, Zaid Harchaoui \\
Department of Statistics, University of Washington, Seattle, USA
}
\date{}
\begin{document}
	
\maketitle
\begin{abstract}
	We present an approach to obtain convergence guarantees of optimization algorithms for deep networks based on elementary arguments and computations. The convergence analysis revolves around the analytical and computational structures of optimization oracles central to the implementation of deep networks in modern machine learning software. We provide a systematic way to compute estimates of the smoothness constants that govern the convergence behavior of first-order optimization algorithms used to train deep networks. Diverse examples related to modern deep networks are interspersed within the text to illustrate the approach.
\end{abstract}

\section{Introduction}
Deep networks have achieved remarkable performance in several application domains such as computer vision, natural language processing and genomics \citep{krizhevsky2012imagenet, pennington2014glove, duvenaud2015convolutional}. The input-output mapping implemented by a deep neural network is a chain of compositions of modules, where each module is typically a composition of a non-linear mapping, called an activation function, and an affine mapping. The last module in the chain is usually task-specific in that it relates to a performance accuracy for a specific task. This module can be expressed either explicitly in analytical form as in supervised classification or implicitly as a solution of an optimization problem as in dimension reduction or unsupervised clustering.

The optimization problem arising when training a deep network is often framed as a non-convex optimization problem, dismissing the structure of the objective yet central to the software implementation. Indeed optimization algorithms used to train deep networks proceed by making calls to first-order (or second-order) oracles relying on dynamic programming such as gradient back-propagation \citep{werbos1994roots, rumelhart1985learning, lecun1988theoretical, duda2012pattern,martin2002neural,shalev2014understanding,goodfellow2016deep}. Gradient back-propagation is now part of modern machine learning software~\citep{tensorflow2015-whitepaper, paszke2017automatic}. We highlight here the elementary yet important fact that the chain-compositional structure of the objective naturally emerges through the smoothness constants governing the convergence guarantee of a gradient-based optimization algorithm. This provides a reference frame to relate the network architecture and the convergence rate through the smoothness constants. This also brings to light the benefit of specific modules popular among practitioners to improve the convergence.

In Sec.~\ref{sec:pb}, we define the parameterized input-output map implemented by a deep network as a chain-composition of modules and write the corresponding optimization objective consisting in learning the parameters of this map. In Sec.~\ref{sec:oracles}, we detail the implementation of first-order and second-order oracles by dynamic programming; the classical gradient back-propagation algorithm is recovered as a canonical example. Gauss-Newton steps can also be simply stated in terms of calls to an automatic-differentiation oracle implemented in modern machine learning software libraries. In Sec.~\ref{sec:optim}, we present the computation of the smoothness constants of a chain of computations given its components and the resulting convergence guarantees for gradient descent. Finally, in Sec.~\ref{sec:appli}, we present the application of the approach to derive the smoothness constants for the VGG architecture and illustrate how our approach can be used to identify the benefits of batch-normalization~\citep{simonyan2014very, ioffe2015batch}. 
In the Appendix, we estimate the smoothness constants related to the VGG architecture and we investigate batch-normalization in the light of our approach~\citep{simonyan2014very, ioffe2015batch}. All the proofs and the notations are also provided in the Appendix.

\section{Problem formulation}\label{sec:pb}
\subsection{Deep network architecture}\label{sec:deep_pres}
A feed-forward deep  network of depth $\horizon$ can be described as a transformation of an input $\rand$ into an output $\state_\horizon$ through the composition of $\horizon$ blocks, called layers, illustrated in Fig.~\ref{fig:layers}. Each layer is defined by a set of parameters. In general, (see Sec.~\ref{ssec:deep} for a detailed decomposition), these parameters act on the input of the layer through an affine operation followed by a non-linear operation. Formally,  the  $t$\textsuperscript{th} layer can be described as a function of its parameters $\ctrl_t$ and a given input $\state_{t-1}$ that outputs $\state_t$ as
\begin{equation}\label{eq:deep_simp}
\state_t = \dyn_t(\state_{t-1}, \ctrl_t) = \nonlin_t(\biaffine_t(\state_{t-1}, \ctrl_t)),
\end{equation}
where $\biaffine_t$ is generally linear in $\ctrl_t$ and affine in $\state_{t-1}$ and $\nonlin_t$ is non-linear.

Learning a deep network consists in minimizing w.r.t. its parameters an objective involving $\nbsamp$ inputs $\fixedstate^{(1)}, \ldots, \fixedstate^{(\nbsamp)} \in \reals^\diminput$. Formally, the problem is written
\begin{align}
	\min_{(\ctrl_1, \ldots, \ctrl_\horizon)  \in \reals^{\dimctrl_1} \times \ldots \times \reals^{\dimctrl_\horizon}} \qquad & \obj(\state_\horizon^{(1)}, \ldots, \state_\horizon^{(\nbsamp)}) + \reg(\ctrl_1,\ldots, \ctrl_\horizon) \nonumber\\ 
	\mbox{subject to} \qquad & \state_t^{(i)} = \dyn_t(\state_{t-1}, \ctrl_t^{(i)}) \quad \mbox{for} \: t=1, \ldots, \horizon, \  \: i=1, \ldots, \nbsamp, \nonumber \\
	& \state_0^{(i)} = \fixedstate^{(i)} \hspace{42pt}\mbox{for}  \: i=1, \ldots, \nbsamp, \label{eq:deep_obj}
\end{align}
where $\ctrl_t \in \reals^{\dimctrl_t}$ is the set of parameters at layer $t$ whose dimension $\dimctrl_t$ can vary among layers and $\reg$ is a regularization on the parameters of the network.

We are interested in the influence of the architecture on the optimization complexity of the problem. The architecture translates into a structure of the chain of computations involved in the optimization problem. 
\begin{definition}\label{def:chain}
	A chain of $\horizon$ computations  $\dyn_t: \reals^{\dimstate_{t-1}} \times\reals^{\dimctrl_t}  \rightarrow \reals^{\dimstate_t}$ is defined as $\chain: \reals^{\dimstate_0} \times  \reals^{\sum_{t=1}^\horizon \dimctrl_t} \rightarrow \reals^{\sum_{t=1}^ \horizon \dimstate_t}$ such that for $\state_0 \in \reals^{\dimstate_0}$ and  $\ctrls = (\ctrl_1; \ldots;\ctrl_\horizon)\in \reals^{\sum_{t=1}^\horizon \dimctrl_t}$ we have $\chain(\state_0, \ctrls) = (\chain_1(\state_0, \ctrls); \ldots ;\chain_\horizon(\state_0, \ctrls))$ with
	\begin{align}\label{eq:chain_def}
		\chain_t(\state_0, \ctrls)  & =	\dyn_t(\chain_{t-1}(\state_0, \ctrls), \ctrl_t ) \quad \mbox{for} \: t=1, \ldots, \horizon, & 
	\end{align}
	and $\chain_0(\state_0, \ctrls) = \state_0$.  We denote $\chain_{t, \state_0}(\ctrls) = \chain_t(\state_0, \ctrls)$ and $\chain_{t, \ctrls} (\state_0) = \chain_t (\state_0, \ctrls)$.
\end{definition}
Denote then $\chain^0$ the chain of computations associated to the layers of a deep network and consider the concatenation of the transformations of each input as a single transformation, i.e., $\chain_t(\fixedstate, \ctrls) =( \chain_t^0(\fixedstate^{(1)}, \ctrls); \ldots; \chain_t^0(\fixedstate^{(\nbsamp)}, \ctrls))$ for $t\in \{1, \ldots, \horizon\}$, and $\fixedstate = (\fixedstate^{(1)}; \ldots;\fixedstate^{(\nbsamp)})$, the objective in~\eqref{eq:deep_obj} can be written as
\begin{align}
\min_{\ctrls \in \reals^{\sum_{t=1}^\horizon \dimctrl_t}} \quad  & \obj(\chain_\horizon(\fixedstate, \ctrls)) + \reg(\ctrls), \label{eq:pb}
\end{align}
where $\chain_\horizon:  \reals^{\nbsamp \dimstate_0} \times  \reals^{\sum_{t=1}^\horizon \dimctrl_t} \rightarrow \reals^{\nbsamp\dimstate_\horizon}$ is the output of a chain of $\horizon$ computations with $\dimstate_0 = \diminput$,  $\reg: \reals^{\sum_{t=1}^\horizon \dimctrl_t} \rightarrow \reals$  is typically a decomposable differentiable function such as  $\reg(\ctrls) = \lambda \sum_{t=1}^{\horizon}\|\ctrl_t\|_2^2$ for $\lambda \geq 0$,  and we present examples of learning objectives $f:\reals^{\nbsamp \dimstate_\horizon} \rightarrow \reals$ below. Assumptions on differentiability and smoothness of the objective are detailed in Sec.~\ref{sec:optim}.

\begin{figure*}
	\begin{center}
		
		\begin{tikzpicture}
			\node (input)  {$\blue{\state_0}$};
			\node[module] (layer1) [right=2em of input]{$\dyn_1$};
			\node (param1) [above=4em of layer1, align=right, mediumred] {$\ctrl_1$};
			
			\node (dots1) [right=2em of layer1]{$\ldots$};
			\node [module] (biaffine)[right=3em of dots1]{$\linetikz$};
			\node [above=-2pt] at (biaffine.south east) {\footnotesize \hspace{-13pt} $\biaffine_t$};
			\node [module] (nonlin)[right=3em of biaffine]{$\sigmoid$};
			\node [above=-2pt] at (nonlin.south east) {\footnotesize \hspace{-13pt} $\nonlin_t$};
			\node (params) [above=4em of biaffine, align=right, mediumred] {$\ctrl_t$};
			\node (empty) [right=1em of nonlin] {};
			\node [bigmodule, fit=(biaffine) (nonlin) (empty), minimum height = 5em] (layerl){};
			\node [above] at (layerl.south east) {\hspace{-2em} \large $\dyn_t$};
			
			\node (dots2) [right=5em of nonlin] {$\ldots$};
			\node[module] (layerk) [right=2em of dots2]{$\dyn_{\horizon}$};
			\node (paramk) [above=4em of layerk, align=right, mediumred] {$\ctrl_{\horizon}$};
			\node (output) [right=2em of layerk]{$\blue{\state_\horizon} = \chain_\horizon(\state_0, \ctrls)$};
			\draw[->, thick, mediumblue] (input) -- (layer1);
			\draw[->, thick, mediumblue] (layer1) -- (dots1);
			\draw[->, thick, mediumblue] (dots1) -- (biaffine) node[near start, above, mediumblue]{$\state_{t-1}$};
			
			\draw[->, thick, mediumblue] (biaffine) -- (nonlin);
			
			\draw[->, thick, mediumblue] (nonlin) -- (dots2) node[pos=0.85, above, mediumblue]{$\state_t$};
			\draw[->, thick, mediumblue] (dots2) -- (layerk);
			\draw[->, thick, mediumblue] (layerk) -- (output);
			
			\draw[->, thick, darkred] (params) -- (biaffine);
			\draw[->, thick, darkred] (param1) -- (layer1);
			\draw[->, thick, darkred] (paramk) -- (layerk);
			
			\node  [bigmodule, fit=(layer1) (layerl) (layerk), minimum height = 5em] (chain){};
			\node [above] at (chain.south east) {\hspace{-2em} \large $\chain$};
		\end{tikzpicture}
	\end{center}
	\caption{Deep network architecture. \label{fig:layers}}
\end{figure*}
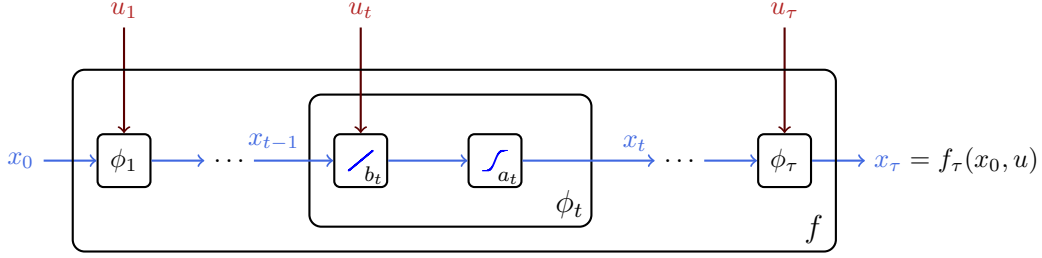	

\subsection{Objectives}
In the following, we consider the output of the chain of computations on $n$ sample to be given as $\labpred = (\labpred^{(1)};\ldots;\labpred^{(\nbsamp)}) = ( \chain_\horizon^0 (\fixedstate^{(1)}, \ctrls); \ldots; \chain_\horizon^0(\fixedstate^{(\nbsamp)}, \ctrls)) = \chain_\horizon(\fixedstate, \ctrls)$  for $\fixedstate = (\fixedstate^{(1)}; \ldots;\fixedstate^{(\nbsamp)})$.
\paragraph{Supervised learning}
For supervised learning, the objective can be decomposed as a finite sum
\begin{equation}\label{eq:supervised}
\obj(\labpred) = \frac{1}{\nbsamp}\sum_{i=1}^\nbsamp \obj^{(i)}(\labpred^{(i)}),
\end{equation}
where $\obj^{(i)}$ are losses on the labels predicted by the chain of computations, i.e., $\obj^{(i)}(\labpred^{(i)}) = \loss(\labpred^{(i)}, \lab^{(i)})$ with $\lab^{(i)}$ the label of $\fixedstate_i$, and $\loss$ is a given loss such as the squared loss or the logistic loss (see Appendix~\ref{ssec:supervised_obj}).

\paragraph{Unsupervised learning}
In unsupervised learning tasks the labels are unknown. The objective itself is defined through a minimization problem rather than through an explicit loss function. For example, a convex clustering objective~\citep{hocking2011clusterpath} is written
\begin{align*}
\obj(\labpred)=  \min_{\lab^{(1)},\ldots, \lab^{(\nbsamp)} \in \reals^\dimlabel} & \sum_{i=1}^\nbsamp\frac{1}{2}\|\lab^{(i)}-\labpred^{(i)} \|_2^2  + \sum_{i<j}\|\lab^{(i)}-\lab^{(j)}\|_2.
\end{align*} 
We consider in
Appendix~\ref{ssec:unsupervised_obj}  different clustering objectives. Note that classical ones, such as the one of $k$-means or spectral clustering, are inherently non-smooth, i.e., non-continuously differentiable.

\subsection{Layer decomposition}\label{ssec:deep}
The $t$\textsuperscript{th} layer of a deep network can be described by the following components:
\begin{enumerate}[label=(\roman*)]
	\item  a bi-affine operation such as a matrix multiplication or a convolution, denoted $\biaffine_t: \reals^{\dimstate_{t-1}} \times \reals^{\dimctrl_t}   \rightarrow \reals^{\diminter_t}$ and decomposed as
	\begin{equation}\label{eq:biaffine}
		\biaffine_t(\state_{t-1}, \ctrl_t) = \bilinear_t(\state_{t-1}, \ctrl_t) + \bilinear_t^\ctrl(\ctrl_t) + \bilinear_t^\state(\state_{t-1}) + \linearcste_t,
	\end{equation}
	where $\bilinear_t$ is bilinear, $\bilinear^\ctrl_t$ and $\bilinear_t^\state$ are linear and $\linearcste_t$ is  a constant vector,
	\item an activation function, such as the element-wise application of a non-linear function,  denoted $\activ_t: \reals^{\diminter_t} \rightarrow \reals^{\diminter_t}$,
	\item  a reduction of dimension, such as a pooling operation, denoted $\pool_t : \reals^{\diminter_t } \rightarrow \reals^{\dimstate_t}$,
	\item a normalization of the output, such as batch-normalization, denoted $\normal_t: \reals^{\dimstate_t} \rightarrow \reals^{\dimstate_t}$.
\end{enumerate}
By concatenating the non-affine operations, i.e., defining $\nonlin_t = \normal_t \circ \pool_t \circ \activ_t$, 
a layer can be written as
\begin{align}\label{eq:deep_layer} 
	\dyn_t(\state_{t-1}, \ctrl_t) =\nonlin_t(\biaffine_t(\state_{t-1}, \ctrl_t)).
\end{align}
Note that some components may not be included, for example some layers do not include normalization. In the following, we consider the non-linear operation $\nonlin_t$ to be an arbitrary composition of functions, i.e., $\nonlin_t = \nonlin_{t, \nbcomp_t} \circ \ldots \circ \nonlin_{t, 1}$.
We present common examples of the components of a deep network.

\subsubsection{Linear operations}
In the following, we drop the dependency w.r.t. the layer $t$ and denote by a tilde $\tilde \cdot$ the quantities characterizing the output. We denote by semi-columns the concatenations of matrices by rows, i.e., for $A\in \reals^{d\times n}, B\in \reals^{q\times n}$, $(A;B) = (A^\top, B^\top)^\top$.
\paragraph{Fully connected layer}
A \emph{fully connected} layer taking an input of dimension $\diminput$ is written
\begin{equation}\label{eq:fully_connected}
	\latentnext = \Weight^\top \latent + \offset,
\end{equation}
where $\latent \in \reals^{\diminput}$ is the input, $\Weight \in \reals^{\diminput \times \diminputnext}$ are the weights of the layer and $\offset \in \reals^{\diminputnext}$ define the intercepts. 
By vectorizing the parameters and the inputs, a fully connected layer can be written as 
\begin{align*}
	& \tilde \state = \bilinear(\state, \ctrl) + \bilinear^\ctrl(\ctrl), \\
	\mbox{where} \qquad & \bilinear(\state, \ctrl) = \Weight^\top \latent \in \reals^{\diminputnext}, \ 
	\bilinear^\ctrl(\ctrl) = \offset,
	\\
	& 
	\state= \latent \in \reals^{\diminput}, \ 
	\ctrl = \Vect(\Weight; \offset) \in \reals^{\diminputnext (\diminput+1)}.
\end{align*}

\paragraph{Convolutional layer}
A \emph{convolutional} layer convolves an input (images or signals) of dimension $\diminput$ denoted $\latent \in \reals^{\diminput}$ with $\nbfilter$ affine filters of size $\dimfilter$ defined by weights $\Weight =(\weight_1, \ldots, \weight_{\nbfilter})\in \reals^{\dimfilter \times \nbfilter}$ and intercepts $\offset =(\offset_{1}, \ldots, \offset_{\nbfilter}) \in \reals^{\nbfilter}$ through $\nbpatch$ patches. The $k$\textsuperscript{th} output of the convolution of the input by the $j$\textsuperscript{th} filter reads 
\begin{equation}\label{eq:conv}
	\Xi_{j, k} = \weight_{j}^\top\patch_k \latent + \offset_{j},
\end{equation}
where $\patch_k\in \reals^{\dimfilter \times \diminput}$ extracts a patch of size $\dimfilter$ at a given position of the input $\latent$.  The output $\latentnext$ is then given by the concatenation $\latentnext_{k + \nbpatch(j-1)} =\Xi_{ j, k}$.
By vectorizing the inputs and the outputs, the convolution operation is defined by a set of matrices $(\patch_k)_{k=1}^{\nbpatch}$ such that
\begin{align*}
	&\tilde \state = \bilinear(\state, \ctrl)  + \bilinear^\ctrl(\ctrl),\\
	\mbox{where} \qquad &
	\bilinear( \state, \ctrl) = (\weight_{j}^\top\patch_k \latent)_{j=1,\ldots, \nbfilter ; k=1,\ldots, \nbpatch} \in \reals^{\nbfilter \nbpatch},
	\ \bilinear^\ctrl(\ctrl) =  \offset \otimes \ones_{\nbpatch},\\
	& \state = \latent \in \reals^{\diminput}, \
	\ctrl = \Vect(\Weight; \offset) \in  \reals^{(\dimfilter +1)\nbfilter}, \ \Weight = (\weight_1, \ldots, \weight_{\nbfilter}).
\end{align*}

\subsubsection{Activation functions}
We consider element-wise activation functions $\activ:\reals^\diminter \rightarrow \reals^\diminter$  such that for a given $\state =(\state_1,\ldots, \state_\diminter) \in \reals^\diminter $, 
\begin{equation}\label{eq:element_wise}
	\activ(\state) = (\activelt(\state_1), \ldots, \activelt(\state_\diminter)),
\end{equation}
for a given scalar function $\activelt$ such as  $\activelt(x) = \max(x, 0)$ for the Rectified Linear Unit (ReLU) or $\activelt(x) = (1+\exp(-x))^{-1}$ for the sigmoid function.

\subsubsection{Pooling functions} 
A pooling layer reduces the dimension of the output. 
For example, an average pooling convolves an input image with a mean filter. Formally, for an input $\latent \in \reals^{\diminput}$, the average pooling with a patch size $\dimfilter$ for inputs with $\nbfilter$ channels and $\nbpatch$ coordinates such that $\diminput= \nbfilter \nbpatch$ convolves the input with a filter $P = \ones_{\dimfilter} \ones_{\nbfilter}^\top/\dimfilter$.
The output dimension for each input is $\tilde \diminput =\nbfilter \tilde \nbpatch $ and the patches, represented by some $(\patch_k)_{k=1}^{\tilde{\nbpatch}}$ acting in Eq.~\eqref{eq:conv}, are chosen such that it induces a reduction of dimension, i.e., $\tilde{\nbpatch} \leq \nbpatch$.

\subsubsection{Normalization functions}
Given a batch of input $\Latent \in \reals^{\diminput \times \batchsize}$ the batch-normalization outputs $\tilde \Latent$ defined by
\begin{align}\label{eq:batchnorm}
	(\tilde \Latent)_{ij} & = \frac{\Latent_{ij} - \mu_i}{\sqrt{\epsilon +\sigma_i^2}}, \\
	\mbox{where} \qquad  \mu_i & = \frac{1}{\batchsize}\sum_{j=1}^\batchsize \Latent_{ij}, \quad \sigma_i^2 = \frac{1}{\batchsize}\sum_{j=1}^\batchsize (\Latent_{ij} - \mu_i)^2, \nonumber
\end{align}
with $\batchreg >0$, such that the vectorized formulation of the batch-normalization reads $\normal(\state) = \Vect(\tilde \Latent)$ for $\state = \Vect(\Latent)$.

\subsection{Specific structures}
\subsubsection{Auto-encoders}
An auto-encoder seeks to learn a compact representation of some data $\fixedstate\in \reals^d$ by passing it through an encoder network with output dimension $\hat d\ll d$ then a decoder network with output dimension $ d$ with the objective that the final output is close to the original input. Each network can be represented by a chain of computations. Given $\nbsamp$ data points $\fixedstate = (\fixedstate^{(1)}: \ldots;\fixedstate^{(\nbsamp)})$, denoting $\chain^e$ the encoder with parameters $\ctrls_e$ such that $\chain_{\ctrls_e}^e: \reals^d \rightarrow \reals^{\hat d}$ and $\chain^d$ the decoder with parameters $\ctrls_d$ such that $\chain_{\ctrls_d}^d: \reals^{\hat d} \rightarrow \reals^{d}$, the objective is 
\[
\min_{\ctrls^e, \ctrls^d} \frac{1}{n}\sum_{i=1}^n \|\fixedstate^{(i)} - \chain^d (\chain^e(\fixedstate^{(i)}, \ctrls_e), \ctrls_d)\|_2^2.
\]
The composition of the encoder and the decoder form a chain of computations such that the overall objective can be written as in~\eqref{eq:pb} as detailed in Appendix~\ref{ssec:deep}. Formally, denoting $\chain^0 :\state_0, \ctrls \rightarrow\chain^d (\chain^e(\state_0, \ctrls_e), \ctrls_d)$  for $\ctrls = (\ctrls_d;\ctrls_e)$ the resulting chain of computations on a single input and $\chain(\fixedstate,\ctrls) = (\chain^0(\fixedstate^{(1)}, \ctrls); \ldots;\chain^0(\fixedstate^{(\nbsamp)}, \ctrls)$ the concatenation of the outputs applied to the set of inputs, the objective of an auto-encoder has the form
$
\obj(\chain(\fixedstate, \ctrls))
$
with $\obj(\hat \state) = \frac{1}{n}\sum_{i=1}^n\|\fixedstate_i - \hat \state_i\|_2^2$.

\subsubsection{Dense, highway or residual networks}
Dense networks use not only the last input but all previous ones.  The output of such networks  can be described as
\begin{align}
	\label{eq:dense_net}
	\chain_\horizon(\state_0, \auxctrls) = \state_\horizon \quad \mbox{with} \qquad \state_t & =	\dyn_t( \state_{0:t-1}, \auxctrl_t), \quad  \state_{0:t-1} = (\state_{0}; \ldots; \state_{t-1}) \quad \mbox{for} \: t=1, \ldots, \horizon,
\end{align}
where  $ \auxctrl_t= (\ctrl_{t,0}; \ldots \ctrl_{t,t-1})$ are the parameters of the layer dispatched with one set of parameters per previous state and $\auxctrls =(\auxctrl_1;\ldots;\auxctrl_\horizon)$. The dynamics can be described as previously as $ \dyn_t(\state_{0:t-1}, \auxctrl_t) = \nonlin_t(\biaffine_t(\state_{0:t-1}, \auxctrl_t))$. The bilinear operation $\biaffine_t$ is still a matrix multiplication or a convolution as previously presented except that it incorporates more variables. The non-linear operation $\nonlin_t$ is also the same, i.e., it incorporates an activation function and, potentially, a pooling operation and a normalization operation.

Dense networks can naturally be translated as a single input-output transformation by defining layers of the form 
\begin{align}\nonumber
	\state_{0:t} =  \dynaux_t(\state_{0:t-1}, \auxctrl_t) = (\state_0; \ldots;  \state_{t-1}; \dyn_t(\state_{0:t-1}, \auxctrl_t) ) =   (\state_{0:t-1}; \dyn_t(\state_{0:t-1}, \auxctrl_t) ) \quad \mbox{for} \: t=1, \ldots, \horizon,
\end{align}
and $\chain_\horizon (\state_0, \auxctrls) = E_\horizon  \state_{0:\horizon} = \state_\horizon$ where $E_\horizon$ is a linear projector that extracts $\state_\horizon$ from $ \state_{0:\horizon}$.

Highway networks are dense networks that consider only the last input and the penultimate one, i.e., they are of the form~\eqref{eq:dense_net} except that they propagate only $ \state_{t-1:t} = (\state_{t-1}, \state_t)$. Namely they are defined by 
\begin{align}\nonumber
	\state_{t-1:t} =  \dynaux_t(\state_{t-2:t-1}, \auxctrl_t) =  (\state_{t-1}; \dyn_t(\state_{t-2:t-1}, \auxctrl_t) ) \quad \mbox{for} \: t=1, \ldots, \horizon,
\end{align}
with  $ \auxctrl_t= (\ctrl_{t,t-2};\ctrl_{t,t-1})$.
Finally, residual networks are highway networks with fixed parameters acting on the penultimate input. In the simple case where the current and penultimate inputs have the same dimension, they read
\begin{align}\label{eq:chain_res_def}
	\state_t & =	\nonlin_t(\biaffine_t(\state_{t-1}, \ctrl_t) + \state_{t-2}) = \nonlin_t(\tilde \biaffine_t(\state_{t-2:t-1}, \ctrl_t)) \quad \mbox{for} \: t=1, \ldots, \horizon, 
\end{align}
with $\state_{-1} = 0$, 
where $\biaffine_t$ and $\nonlin_t$ are of the forms described above.
This amounts to define layers $ \dynaux_t$ on $\state_{t-2:t-1}$ whose bi-affine operation $\tilde \biaffine_t(\state_{t-2:t-1}, \ctrl_t)$ has a non-zero affine term $\tilde \bilinear_t^\state$ on $\state_{t-2:t-1} = (\state_{t-2}; \state_{t-1})$, see Appendix~\ref{ssec:residual}.

\subsubsection{Implicit functions}
We consider implicit functions that take the form 
\[
\Impfunc(\impfix) = \argmin_{\impvar \in \reals^\dimimpvar} \impfunc(\impfix, \impvar)
\]
where $\impfunc$ is twice differentiable and $\impfunc(\alpha, \cdot)$ is strongly convex for any $\impfix$ such that $\Impfunc(\impfix)$ is uniquely defined . These can be used either in the objective as seen before with clustering tasks, in that case $\impfix = \state_\horizon$. These can also be used in the layers such that $\impfix = (\state, \ctrl)$ and $\dyn(\state, \ctrl) = \argmin_{\impvar \in \reals^\dimimpvar} \impfunc(\state, \ctrl, \impvar)$. 

If the minimizer is computed exactly, we can compute the gradient by invoking the implicit function theorem.
Formally, denoting $\impgrad(\impfix, \impvar) = \nabla_{\impvar} \impfunc(\impfix, \impvar)$, the function $g(\alpha)$ is defined by the implicit equation $\impgrad(\impfix, \Impfunc(\impfix)) = 0$ and its gradient is given by
\begin{align*}
	\nabla \Impfunc(\impfix) & = - \nabla_{\impfix} \impgrad(\impfix, \Impfunc(\impfix))\nabla_\impvar \impgrad(\impfix, \Impfunc(\impfix))^{-1} 
	= - \nabla_{\impfix, \impvar}^2\impfunc(\impfix, \Impfunc(\impfix)) \nabla_{\impvar, \impvar}^2 \impfunc(\impfix, \Impfunc(\impfix))^{-1} 
\end{align*}
The smoothness constants of this layer for exact minimizations are provided in Appendix~\ref{ssec:implicit}.

If the minimizer is computed approximately through an algorithm, its derivative can be computed by using automatic differentiation through the chain of computations defining the algorithm (see Subsection~\ref{ssec:autodiff} for a detailed explanation of automatic differentiation). Alternatively, an approximate gradient can be computed by using the above formula. 
The resulting approximation error of the gradient is given by the following lemma. 
\begin{restatable}{lemma}{approximplicitgrad}
	Let $\impfunc: (\impfix, \impvar) \rightarrow \impfunc(\impfix, \impvar) \in \reals$ for $\impfix \in \reals^{\dimimpfix}, \impvar \in \reals^\dimimpvar$ be  s.t. $\impfunc(\impfix, \cdot)$ is $\mu_\impfunc$-strongly convex for any $\impfix$ and denote 	 $\impgrad(\impfix, \impvar) = \nabla_{\impvar} \impfunc(\impfix, \impvar)$. Denote $\Impfunc(\impfix) = \argmin_{\impvar \in \reals^\dimimpvar} \impfunc(\impfix, \impvar)$ and $\hat \Impfunc(\impfix) \approx \argmin_{\impvar \in \reals^\dimimpvar} \impfunc(\impfix, \impvar)$ be an approximate minimizer. Provided that $\impfunc$ has a $\smooth_\impfunc$-Lipschitz gradient and a $\smoothess_\impfunc$-Lipschitz Hessian, the approximation error of using 
	\[
	\widehat \nabla \hat \Impfunc(\impfix) = - \nabla_{\impfix} \impgrad(\impfix, \hat \Impfunc(\impfix))\nabla_\impvar \impgrad(\impfix, \hat \Impfunc(\impfix))^{-1}
	\]
	instead of $\nabla \Impfunc(\impfix)$ is bounded as
	\[
	\|\widehat \nabla \hat \Impfunc(\impfix) - \nabla \Impfunc(\impfix) \|_2 \leq \smoothess_\impfunc \mu_\impfunc^{-1}(1+ \smooth_\impfunc\mu_\impfunc^{-1})\|\hat \Impfunc(\impfix) - \Impfunc(\impfix)\|_2.
	\]
\end{restatable}

\section{Oracle arithmetic complexity}\label{sec:oracles}
For each class of optimization algorithm considered (gradient descent, Gauss-Newton, Newton), we define the appropriate optimization oracle called at each step of the optimization algorithm which can be efficiently computed through a dynamic programming procedure. For a gradient step, we retrieve the gradient back-propagation algorithm. The gradient back-propagation algorithm forms then the basis of automatic-differentiation procedures. 

\subsection{Oracle reformulations}
In the following, we use the notations presented in Sec.~\ref{sec:notations} for gradients, Hessians and tensors. Briefly, $\nabla f(x)$ is used to denote the gradient of a function $f$ at $x$, which, if $f:\reals^p \rightarrow\reals^d$ is multivariate, is the transpose of the Jacobian, i.e., $\nabla f(x) \in \reals^{p\times d}$. For a multivariate function $f:\reals^p\rightarrow\reals^ d$, its second order information at $x$ is represented by a tensor $\nabla^2 f(x) \in \reals^{p\times p\times d}$, and we denote for example $\nabla^2 f(x) [y, y, \cdot] = (y^\top\nabla^2 f^{(1)}(x))y; \ldots; y^\top\nabla^2 f^{(n)}(x)y) \in \reals^d $. 
For a function $f:\reals^p \rightarrow \reals^ d$, we define, provided that $\nabla f(x)$, $\nabla f^2(x)$ are defined, 
\begin{align}
	\lin_f^\state(y) =\nabla f(x)^\top y, \label{eq:lin_approx}\qquad 
	\qua_f^\state(y)  = \nabla f(x)^\top y +\frac{1}{2} \nabla^2f(x)[y, y, \cdot],
\end{align}
such that the linear and quadratic approximations of $f$ around $x$  are  $f(x+ y) \approx f(x) +\lin_f^\state(y)$ and $f(x+ y) \approx f(x) + \qua_f^\state(y)$ respectively. 

We consider optimization oracles as procedures that compute either the next step of an optimization method or a decent direction along which the next step of an optimization method is taken. Formally, the optimization oracles for an objective $f$ are defined by a model $m_f^\ctrls$ that approximates the objective around the current point $\ctrls$ as $f(\ctrls+\auxctrl) \approx f(\ctrls) +m_f^\ctrls(\auxctrl)$. 
The models can be minimized with an additional proximal term that ensures that the minimizer  lies in a region where the model approximates well the objective as
\[
\auxctrl^*_\stepsize = \argmin_{\auxctrl\in \reals^{p}} m_f^\ctrls( \auxctrl) + \frac{1}{2\stepsize}\|\auxctrl\|_2^2 , \qquad \nextctrls   = \currctrls +  \auxctrl^*_\stepsize.
\]
The parameter $\stepsize$ acts as a stepsize that controls how large should be the step (the smaller the $\stepsize$, the smaller the $\auxctrl^*_\stepsize$).
Alternatively the model can be minimized directly providing a descent direction along which the next iterate is taken as
\begin{align*}
	\auxctrl^*  = \argmin_{\auxctrl\in \reals^{p}} m_f^\ctrls(\auxctrl) \qquad
	\nextctrls   = \currctrls + \stepsize \auxctrl^*,
\end{align*}
where $\stepsize$ is found by a line-search using e.g. an Armijo condition~\citep{nocedal2006numerical}.

On a point $\currctrls\in \reals^\dimvar$, given a regularization $\kappa$, for an objective of the form $\obj\circ \chainoutput +\reg:\reals^p \rightarrow \reals$,
	\begin{enumerate}[nosep, leftmargin=*]
	\item[(i)] a \emph{gradient} oracle is defined as 
	\begin{align}
	\oraclectrls = \argmin_{\auxctrls \in \reals^\dimvar} \: & \lin_{\obj \circ\chainoutput}^\currctrls(\auxctrls) +
	 \lin_\reg^\currctrls(\auxctrls) +\frac{\kappa}{2}\|\auxctrls\|_2^2, \label{eq:grad_step}
	\end{align}
	\item[(ii)] a (regularized) \emph{Gauss-Newton} oracle is defined as
	\begin{align}
	\oraclectrls = \argmin_{\auxctrls \in \reals^\dimvar}\: & 
	\qua_{\obj}^{\chainoutput(\currctrls)}(\lin_\chainoutput^\currctrls(\auxctrls)) 
	+ \qua_\reg^\currctrls(\auxctrls) + 
	\frac{\kappa}{2}\|\auxctrls\|_2^2, 
	\label{eq:gn_step}
	\end{align}
	\item[(iii)] a (regularized) \emph{Newton} oracle is defined as
	\end{enumerate}
	\begin{align}
	\oraclectrls = \argmin_{\auxctrls \in \reals^\dimvar} \: & \qua_{\obj\circ \chainoutput}^\currctrls(\auxctrls) + \qua_\reg^\currctrls(\auxctrls) + \frac{\kappa}{2}\|\auxctrls\|_2^2. \label{eq:newton_step}
	\end{align}
\begin{restatable}{proposition}{linquad}\label{prop:lin_quad}
	Let $\chain$ be a chain of $\horizon$ computations $\dyn_t:\reals^{\dimstate_{t-1}} \times \reals^{\dimctrl_t} \rightarrow \reals^{\dimstate_t}$, $\currctrls = (\ctrl_1; \ldots;\ctrl_\horizon)$ and $\state_0 \in \reals^{\dimstate_0}$. Denote $\chainoutput = \chain_{\state_0, \horizon}$ and $\chain(\state_0, \ctrls) =  (\state_1;\ldots; \state_\horizon)$. Assume $\reg$ to be decomposable as $\reg(\currctrls) = \sum_{t=1}^\horizon \reg_t(\ctrl_t)$.
	Gradient~\eqref{eq:grad_step}, Gauss-Newton~\eqref{eq:gn_step} and Newton~\eqref{eq:newton_step} oracles on $\obj\circ \chainoutput + \reg$ are the solutions $\oraclectrls = (\oraclectrl_1;\ldots;\oraclectrl_\horizon)$ of problems of the form
	\begin{align}
		\min_{\substack{\auxctrl_1,\ldots, \auxctrl_\horizon \in \reals^{\dimctrl_1} \times \ldots \times \reals^{\dimctrl_\horizon}\\\auxstate_0,\ldots, \auxstate_\horizon \in \reals^{\dimstate_0} \times  \ldots \times \reals^{\dimstate_\horizon}}} \quad &  
		\sum_{t=1}^{\horizon} \frac{1}{2}\auxstate_t^\top \PP_t \auxstate_t  + \p_t^\top \auxstate_t
		+\auxstate_{t-1}^\top \R_t \auxctrl_t   + \frac{1}{2}\auxctrl_t^\top \Q_t\auxctrl_t + \q_t^\top\auxctrl_t 
		+ \frac{\kappa}{2}\|\auxctrl_t\|_2^2 \label{eq:lin_quad}\\
		\mbox{subject to} \quad & \auxstate_t= \A_t \auxstate_{t-1} + \B_t \auxctrl_t \qquad \mbox{for} \quad  t \in \{1,\ldots,\horizon\}, \nonumber\\
		& \auxstate_0 = 0, \nonumber
	\end{align}
	where 
	\begin{flalign*}
		\A_t = \nabla_{\state_{t-1}} \dyn_t(\state_{t-1}, \ctrl_t)^\top, \quad &
		\B_t = \nabla_{\ctrl_t} \dyn_t(\state_{t-1}, \ctrl_t)^\top,
		\\
		\p_\horizon = \nabla \obj(\chainoutput(\currctrls)),  \quad &  p_t = 0  \qquad \mbox{for $t\neq\horizon$},   \\
		&  \q_t = \nabla \reg_t(\ctrl_t),  
	\end{flalign*} 
	\begin{enumerate}[nosep, leftmargin=*]
		\item for gradient oracles~\eqref{eq:grad_step}, 
		\[
		\PP_t = 0, \quad \R_t = 0,\quad \Q_t = 0,
		\]
		\item for Gauss-Newton oracles~\eqref{eq:gn_step}, 
		\begin{gather*}
		\PP_\horizon = \nabla^2\obj(\chainoutput(\currctrls)), \quad
		\PP_t = 0 \quad \mbox{for} \ t\neq\horizon,  \quad
		\R_t = 0, \quad \Q_t = \nabla^2 \reg_t(\ctrl_t),
		\end{gather*},
		\item for Newton oracles~\eqref{eq:newton_step}, 
		defining 
		\begin{gather*}
				\costate_\horizon  =\nabla\obj(\chainoutput(\currctrls)), \quad  
		\costate_{t-1}  = \nabla_{\state_{t-1}} \dyn_t(\state_{t-1}, \ctrl_t) \costate_t \quad \mbox{for} \ t\in\{1,\ldots, \horizon\},
		\end{gather*}
		we have 
		\begin{gather*}
		 \PP_\horizon =  \nabla^2\obj(\chainoutput(\currctrls)),  \quad
		\PP_{t-1} = \nabla^2_{\state_{t-1}\state_{t-1}}\dyn_t(\state_{t-1}, \ctrl_t)[\cdot, \cdot, \costate_t] \quad  \mbox{for} \ t \in\{1, \ldots, \horizon\}, \\
		\R_t = \nabla^2_{\state_{t-1}\ctrl_t} \dyn_t(\state_{t-1}, \ctrl_t)[\cdot, \cdot, \costate_t], \quad
		\Q_t = \nabla^2 \reg_t(\ctrl_t) 
		+ \nabla^2_{\ctrl_t\ctrl_t} \dyn_t(\state_{t-1}, \ctrl_t)[\cdot, \cdot, \costate_t].
		\end{gather*}
	\end{enumerate}
\end{restatable}
Problems of the form
\begin{align}
	\min_{\substack{\ctrl_1,\ldots, \ctrl_\horizon \in \reals^{\dimctrl_1} \times \ldots \times \reals^{\dimctrl_\horizon}\\\state_0,\ldots, \state_\horizon \in \reals^{\dimstate_0} \times  \ldots \times \reals^{\dimstate_\horizon}}} \quad & \sum_{t=1}^\horizon h_t(\state_t) + \sum_{t=1}^\horizon g_t(\ctrl_t) \label{eq:dyn_pb} \\
	\mbox{subject to} \quad & \state_t = \dyn_t(\state_{t-1}, \ctrl_t) \quad \mbox{for} \ t \in \{1, \ldots, \horizon\}, \nonumber \\
	& \state_0 = \hat \state_0 \nonumber
\end{align}
can be decomposed  into nested  subproblems defined as the cost-to-go from $\hat \state_t$ at time $t$ by
\begin{align*}
	\costogo_t(\hat \state_t) = 
	\min_{\substack{\ctrl_{t+1},\ldots, \ctrl_\horizon \in \reals^{\dimctrl_{t+1}} \times \ldots \times \reals^{\dimctrl_\horizon}\\\state_t,\ldots, \state_\horizon \in \reals^{\dimstate_t} \times  \ldots \times \reals^{\dimstate_\horizon}}} \quad & \sum_{t'=t}^\horizon h_{t'}(\state_{t'}) + \sum_{t'=t+1}^\horizon g_{t'}(\ctrl_{t'})  
	\\
	\mbox{subject to} \quad & \state_{t'} = \dyn_{t'}(\state_{t'-1}, \ctrl_{t'}) \quad \mbox{for} \ t' \in \{t+1, \ldots, \horizon\}, \\
	& \state_t = \hat \state_t,
\end{align*}
such that they follow the recursive relation 
\begin{equation}\label{eq:bellman}
\costogo_{t}(\hat \state_t) = \min_{\ctrl_{t+1} \in \reals^{\dimctrl_{t+1}}} \{h_t(\hat \state_t) + g_{t+1}(\ctrl_{t+1}) + \costogo_{t+1}(\dyn_{t+1}(\state_t, \ctrl_{t+1}))\}.
\end{equation}
This principle cannot be used directly on the original problem, since Eq.~\eqref{eq:bellman} cannot be solved analytically for generic problems of the form~\eqref{eq:dyn_pb}. However, for quadratic problems with linear compositions of the form~\eqref{eq:lin_quad}, this principle can be used to solve problems~\eqref{eq:lin_quad} by dynamic programming  \citep{bertsekas2005dynamic}. Therefore as a corollary of Prop.~\ref{prop:lin_quad},  the complexity of all optimization steps given in~\eqref{eq:grad_step},~\eqref{eq:gn_step},~\eqref{eq:newton_step} is linear w.r.t. to the length $\horizon$ of the chain. Precisely, Prop.~\ref{prop:lin_quad} shows that each optimization step amounts to reducing the complexity of the recursive relation~\eqref{eq:bellman} to an analytic problem.

In particular, while the Hessian of the objective scales as $\sum_{t=1}^\horizon \dimctrl_t$, a Newton step has a linear and not cubic complexity with respect to $\horizon$.
We present in Appendix~\ref{sec:oracle_proofs}
the detailed computation of a Newton step, alternative derivations were first proposed in the control literature \citep{dunn1989efficient}. This involves the inversion of intermediate quadratic costs at each layer. Gauss-Newton steps can also be solved by dynamic programming and can be more efficiently implemented using an automatic-differentiation oracles as we explain below. 

\subsection{Automatic differentiation}\label{ssec:autodiff}
\subsubsection{Algorithm}
As explained in last subsection and shown in Appendix~\ref{sec:oracle_proofs}, a gradient step can naturally be  derived as a dynamic programming procedure applied to the subproblem \eqref{eq:lin_quad}. 
However, the implementation of the gradient step provides itself a different kind of oracle on the chain of computations as defined below.
\begin{definition}
	Given a chain of computations $\chain:  \reals^{\sum_{t=1}^{\horizon} \dimctrl_t} \times\reals^{\dimstate_0}\rightarrow \reals^{\sum_{t=1}^\horizon \dimstate_t}$ as defined in Def.~\ref{def:chain}, $\ctrls \in \reals^{\sum_{t=1}^\horizon \dimctrl_t}$ and $\state_0 \in \reals^{\dimstate_0}$, an \emph{automatic differentiation oracle} is a procedure that gives access to
	\[
	\dualvar \rightarrow \nabla\chain_{\state_0, \horizon}(\ctrls)\dualvar \quad \mbox{for any} \ \dualvar \in \reals^{\dimstate_\horizon}.
	\]
\end{definition}

The subtle difference is that \emph{we have access to $\nabla\chain_{\state_0, \horizon}(\ctrls)$ not as a matrix but as a linear operator}. The matrix $\nabla\chain_{\state_0, \horizon}(\ctrls)$ can also be computed and stored to perform gradient vector products. Yet, this requires a surplus of storage and of computations that are generally not necessary for our purposes. The only quantities that need to be stored are given by the forward pass. Then, these quantities can be used to compute any gradient vector product directly. 

The definition of an automatic differentiation oracle is composed of two steps:
\begin{enumerate}
	\item a \emph{forward} pass that computes $\chain_{\state_0, \horizon}( \ctrls)$ and stores the information necessary to compute gradient-vector products.
	\item the compilation of a \emph{backward} pass that  computes $\dualvar \rightarrow \nabla\chain_{\state_0, \horizon}( \ctrls)\dualvar$ for any $\dualvar \in \reals^{\dimstate_\horizon}$ given the information collected in the forward pass.
\end{enumerate}
Note that the two aforementioned passes are decorrelated in the sense that the forward pass does not require the knowledge of the slope $\dualvar$ for which $\nabla \chain_{\state_0, \horizon}(\ctrls)\dualvar$ is computed.

We present in Algo.~\ref{algo:forward_lin} and Algo.~\ref{algo:backward_lin} the classical forward-backward passes used in modern automatic-differentiation libraries. The implementation of the automatic differentiation oracle as a procedure that computes both the value of the chain $\chain_{\state_0, \horizon}(\ctrls)$ and the linear operator $\dualvar\rightarrow \chain_{\state_0, \horizon}(\ctrls)\dualvar$ is then presented in Algo.~\ref{algo:auto_diff} and illustrated in Fig.~\ref{fig:auto_diff}.

Computing the gradient $g= \nabla (\obj\circ \chain_{\state_0, \horizon})(\ctrls)$ on $\ctrls \in \reals^p$ amounts then to
\begin{enumerate}
	\item computing with Algo.~\ref{algo:auto_diff}, $\chain_{\state_0, \horizon}(\ctrls), \dualvar \rightarrow \nabla \chain_{\state_0, \horizon}(\ctrls) \dualvar = \operatorname{Autodiff}(\chain, \ctrls)$,
	\item computing $\mu = \nabla \obj(\chain_{\state_0, \horizon}(\ctrls))$ then $g = \nabla \chain_{\state_0, \horizon}(\ctrls) \mu$.
\end{enumerate}

\begin{figure}

\begin{algorithm}[H]\caption{Forward pass \label{algo:forward_lin}}
	\begin{algorithmic}[1]
		\State{{\bf Inputs:} Chain of computations $\chain$ defined by $(\dyn_t)_{t=1,\ldots, \horizon}$, input $\rand$ as in Def.~\ref{def:chain}, variable $\ctrls = (\ctrl_1; \ldots; \ctrl_\horizon)$ }
		\State{Initialize $\state_0 = \rand$}
		\For{$t = 1,\ldots, \horizon$}
		\State{Compute $\state_t = \dyn_t(\state_{t-1}, \ctrl_t)$}
		\State{Store $\nabla \dyn_t(\state_{t-1}, \ctrl_t)$}
		\EndFor
		\State{{\bf Output:} $\state_\horizon, \nabla \dyn_t(\state_{t-1}, \ctrl_t)$ for $t \in \{1,\ldots, \horizon\}$.
		}
	\end{algorithmic}
\end{algorithm}

\begin{algorithm}[H]\caption{Backward pass \label{algo:backward_lin}}
	\begin{algorithmic}[1]
		\State{{\bf Inputs:} Slope $\dualvar$, intermediate gradients $\nabla \dyn_t(\state_{t-1}, \ctrl_t)$ for $t \in \{1,\ldots, \horizon\}$}
		\State{Initialize $\costate_\horizon = \dualvar$}
		\For{$t =\horizon,\ldots, 1$}
		\State{Compute $\costate_{t-1} = \nabla_{\state_{t-1}}\dyn_t(\state_{t-1}, \ctrl_t)\costate_t $}
		\State{Store $g_t= \nabla_{\ctrl_t} \dyn_t(\state_{t-1}, \ctrl_t) \costate_t$}
		\EndFor
		\State{{\bf Output:} $(g_1,\ldots, g_\horizon)= \nabla \chain_{\state_0, \horizon}(\ctrls)\dualvar$ }
	\end{algorithmic}
\end{algorithm}

\begin{algorithm}[H]\caption{Chain of computations with automatic-differentiation oracle ($\operatorname{Autodiff}$) \label{algo:auto_diff}}
	\begin{algorithmic}[1]
		\State{{\bf Inputs:} Chain of computations $\chain$ defined by $(\dyn_t)_{t=1,\ldots, \horizon}$, input $\rand$ as in Def.~\ref{def:chain}, variable $\ctrls = (\ctrl_1; \ldots; \ctrl_\horizon)$ }
		\State{Compute using Algo.~\ref{algo:forward_lin} $\left(\state_\horizon, (\nabla \dyn_t(\state_{t-1}, \ctrl_t)_{t=1}^\horizon \right)= \operatorname{Forward}(\chain, \ctrls)$ which gives $\chain_{\state_0, \horizon}(\ctrls)= \state_\horizon $} 
		\State{Define $\dualvar \rightarrow \nabla \chain_{\state_0, \horizon}(\ctrls)\dualvar$ as $\dualvar \rightarrow \operatorname{Backward}(\dualvar, (\nabla \dyn_t(\state_{t-1}, \ctrl_t))_{t=1}^\horizon)$} according to Algo.~\ref{algo:backward_lin}.
		\State{{\bf Output:} $\chain_{\state_0, \horizon}(\ctrls), \dualvar \rightarrow \nabla \chain_{\state_0, \horizon}(\ctrls)\dualvar$}
	\end{algorithmic}
\end{algorithm}

\begin{figure}[H]
	\begin{center}
		\resizebox{\linewidth}{!}{%
			\begin{tikzpicture}[round/.style={circle, draw=black!60, thick, minimum size=30pt}, square/.style={regular polygon,regular polygon sides=4, draw=black!60, thick, minimum size=45pt, inner sep=2pt}, node distance=1em]
				\node[round] (x) {$\rand$};
				\node (forward) [above=2em of x, align=center, mediumblue]{Forward pass \\ $\ctrls \rightarrow \chain(\state_0, \ctrls)$ \\ $\ctrls {=}(\ctrl_1; \ldots;\ctrl_\horizon$)};
				\node[round] (z0) [right=2em of x, mediumblue] {$\state_0$};
				\node[square](phi1) [right=2em of z0] {$\dyn_1$};
				\node[round] (v1) [above=2em of phi1, teal] {$\ctrl_1$};
				\node[round] (z1) [right=2em of phi1, mediumblue] {$\state_1$};
				\node[square](phi2) [right=2em of z1] {$\dyn_2$};
				\node[round] (v2) [above=2em of phi2, teal] {$\ctrl_2$};
				\node[round] (z2) [right=2em of phi2, mediumblue] {$\state_2$};
				\node (dots) [right=2em of z2] {$\ldots$};
				\node[square](phitau) [right=2em of dots] {$\dyn_\horizon$};
				\node[round] (vtau) [above=2em of phitau, teal] {$\ctrl_\horizon$};
				\node[round] (ztau) [right=2em of phitau, mediumblue] {$\state_\horizon$};
				\draw[->, thick, mediumblue] (x) -- (z0);
				\draw[->, thick, mediumblue] (z0) -- (phi1);
				\draw[->, thick, teal] (v1) -- (phi1);
				\draw[->, thick, mediumblue] (phi1) -- (z1);
				\draw[->, thick, mediumblue] (z1) -- (phi2);
				\draw[->, thick, teal] (v2) -- (phi2);
				\draw[->, thick, mediumblue] (phi2) -- (z2);
				\draw[->, thick, mediumblue] (z2) -- (dots);
				\draw[->, thick, mediumblue] (dots) -- (phitau);
				\draw[->, thick, teal] (vtau) -- (phitau);
				\draw[->, thick, mediumblue] (phitau) -- (ztau);
				
				\node[square] (nablaphi1) [below=4em of phi1] {$\nabla \dyn_1$};
				\node[square] (nablaphi2) [below=4em of phi2] {$\nabla \dyn_2$};	
				\node[square] (nablaphitau) [below=4em of phitau] {$\nabla \dyn_\horizon$};
				\draw[->, thick, purple] (phi1) -- (nablaphi1);
				\draw[->, thick, purple] (phi2) -- (nablaphi2);
				\draw[->, thick, purple] (phitau) -- (nablaphitau);
				
				\node[round] (lambdatau) [right=2em of nablaphitau, mediumred] {$\costate_\horizon$};
				\node[round] (mu) [right=2em of lambdatau] {$\dualvar$};
				\node[round](gtau) [below=2em of nablaphitau, darkorange] {$g_\horizon$};
				\node (dotsback) [left=2em of nablaphitau] {$\ldots$};
				\node[round] (lambda2) [right=2em of nablaphi2, mediumred] {$\costate_2$};
				\node[round] (lambda1) [right=2em of nablaphi1, mediumred] {$\costate_1$};
				\node[round] (lambda0) [left=2em of nablaphi1, mediumred] {$\costate_0$};
				\node[round](g2) [below=2em of nablaphi2, darkorange] {$g_2$};
				\node[round](g1) [below=2em of nablaphi1, darkorange] {$g_1$};
				\draw[->, thick, mediumred] (mu) -- (lambdatau);
				\draw[->, thick, mediumred] (lambdatau) -- (nablaphitau);
				\draw[->, thick, darkorange] (nablaphitau) -- (gtau);
				\draw[->, thick, mediumred] (nablaphitau) -- (dotsback);
				\draw[->, thick, mediumred] (dotsback) -- (lambda2);
				\draw[->, thick, mediumred] (lambda2) -- (nablaphi2);
				\draw[->, thick, darkorange] (nablaphi2) -- (g2);
				\draw[->, thick, mediumred] (nablaphi2) -- (lambda1);
				\draw[->, thick, mediumred] (lambda1) -- (nablaphi1);
				\draw[->, thick, darkorange] (nablaphi1) -- (g1);
				\draw[->, thick, mediumred] (nablaphi1) -- (lambda0);
				
				\node (backward) [below=8em of x, align=center, mediumred]{Backward pass \\ $\dualvar{ \rightarrow }\nabla \chain_{\state_0, \horizon}(\ctrls) \dualvar {=}g$ \\ $g {=}(g_1; \ldots ;g_\horizon$)};
				\node (store) [below=3ex of x, align=center, purple]{Store gradients \\ $\nabla \dyn_t(\state_{t-1}, \ctrl_t)$};
			\end{tikzpicture}
		}
		\caption{Automatic differentiation of a chain of computations. \label{fig:auto_diff}}
	\end{center}
\end{figure}

\end{figure}

\subsubsection{Complexity}
Without additional information on the structure of the layers, the space and time complexities of the forward-backward algorithm is of the order of
\begin{align*}
\mathcal{S}_{\operatorname{FB}} & \leq \sum_{t=1}^{\horizon} (\dimctrl_t + \dimstate_{t-1})\dimstate_t,   \\ 
\mathcal{T}_{\operatorname{FB}} & \leq \sum_{t=1}^{\horizon} \mathcal{T}(\dyn_t, \nabla \dyn_t) + 2\sum_{t=1}^{\horizon}\left(\dimstate_{t-1}\dimstate_{t}+ \dimctrl_t\dimstate_{t}\right),
\end{align*}
respectively, where $\mathcal{T}(\dyn_t, \nabla \dyn_t)$ is the time complexity of computing $\dyn_t, \nabla \dyn_t$ during the backward pass. The units chosen are for the space complexity the cost of storing one digit and for the time complexity the cost of performing an addition or a multiplication. 

Provided that for all $t \in \{1,\ldots \horizon\}$,
\begin{equation}\label{eq:asm_bauer_strass}
\mathcal{T}(\dyn_t, \nabla \dyn_t) + 2(\dimstate_{t-1}\dimstate_{t}+ \dimctrl_t\dimstate_{t}) \leq Q \mathcal{T}(\dyn_t),
\end{equation}
where $\mathcal{T}(\dyn_t)$ is the time complexity of computing $\dyn_t$ and $Q\geq 0$ is  a constant,
we get that
\[
\mathcal{T}_{\operatorname{FB}} \leq  Q \mathcal{T}(\chain),
\]
where $\mathcal{T}(\chain)$ is the complexity of computing the chain of computations~\citep{kim1984estimate}. We
retrieve Baur-Strassen's theorem which states that the complexity of computing the derivative of a function formulated as a chain of computations is of the order of the complexity of computing the function itself~\citep{baur1983complexity, griewank2012}.

For chain of computations of the form~\eqref{eq:deep_layer}, this cost can be refined as shown in
Appendix~\ref{sec:oracle_proofs}. 
Specifically, for a chain of fully-connected layers with element-wise activation function, no normalization or pooling, the cost of the backward pass is then of the order of
$
\bigO\left(\sum_{t=1}^{\horizon}2\batchsize \diminput_{t}(\diminput_{t-1} +1)\right)
$
elementary operations.
For a chain of convolutional layers with element-wise activation function, no normalization or pooling, the cost of the backward pass is of the order of $
\bigO\left(\sum_{t=1}^{\horizon}(2 \nbpatch_t \nbfilter_t \dimfilter_t + \nbpatch_t\nbfilter_t + \diminput_t)\batchsize \right)
$ elementary operations.

\subsection{Gauss-Newton by automatic differentiation}
The Gauss-Newton step can also be solved by making calls to an automatic differentiation oracle. 
\begin{restatable}{proposition}{gaussnewtonautodiff}\label{prop:dual_gn_step}
	Consider the Gauss-Newton oracle~\eqref{eq:gn_step} on $\currctrls  = (\ctrl_1;\ldots;\ctrl_\horizon)$ for a convex objective $\obj$, a convex decomposable regularization $\reg(\ctrls) = \sum_{t=1}^\horizon \reg_t(\ctrl_t)$ and a differentiable chain of computations $\chain$ with output $\chainoutput = \chain_{\state_0, \horizon}$ on some input $\state_0$.  
	We have that
	\begin{enumerate}[nosep]
		\item the Gauss-Newton oracle amounts to solving 
		\begin{align}\label{eq:dual_gn_step}
			\min_{\dualvar \in \reals^{\dimstate_{\horizon}}}  \left(\qua_\obj^{\chainoutput(\ctrls)}\right)^ \star(\dualvar) +  \left(\qua_\reg^\ctrls + \kappa\|\cdot\|_2^2/2\right)^\star(- \nabla \chainoutput(\currctrls) \dualvar ),
		\end{align}
		where for a function $f$ we denote by $f^\star$ its convex conjugate,
		\item the Gauss-Newton oracle is  $\oraclectrls = \nabla \left(\qua_\reg^\ctrls + \kappa\|\cdot\|_2^2/2\right)^\star(-\nabla \chainoutput(\currctrls) \dualvar^*  )$ where $\dualvar^*$ is the solution of ~\eqref{eq:dual_gn_step},
		\item the dual problem~\eqref{eq:dual_gn_step} can be solved by $2\dimstate_{\horizon}+1$ calls to an automatic differentiation procedure.
	\end{enumerate}
\end{restatable}
Proposition~\ref{prop:dual_gn_step} shows that a Gauss-Newton step is only $2\dimstate_{\horizon}+1$ times more expansive than a gradient-step. Precisely, for a deep network with a supervised objective, we have $\dimstate_{\horizon} =\nbsamp \numclass$ where $\nbsamp$ is the number of samples and $\numclass$ is the number of classes. A gradient step makes then one call to an automatic differentiation procedure to get the gradient of the batch and the Gauss-Newton method will then make $2\nbsamp \numclass+1$ more calls. If mini-batch Gauss-Newton steps are considered then the cost reduces to $2\batchsize \numclass +1$ calls to an automatic differentiation oracle, where $\batchsize$ is the size of the mini-batch.

\section{Optimization complexity}\label{sec:optim}
We present smoothness properties with respect to the Euclidean norm $\|\cdot\|_2$,
whose operator norm is denoted $\|\cdot\|_{\opnormidx}$.
In the following, for a function $f: \reals^d \rightarrow \reals^n$ and a set $\set \subset \dom f \subset \reals^d$, we denote by
\begin{gather*}
	\bound_f^\set = \sup_{\substack{x\in \set}} \|f(x)\|_2,\quad 
	\lip_f^\set   = \sup_{\substack{x, y \in \set \\ x\neq y}} \frac{\|f(x)-f(y)\|_2}{\|x-y\|_2}, \quad   
	\smooth_f^\set   = \sup_{\substack{x, y \in \set \\ x\neq y}} \frac{\|\nabla f(x)-\nabla f(y)\|_{\opnormidx}}{\|x-y\|_2},
\end{gather*}
a bound of $\obj$ on $\set$,  the Lipschitz-continuity parameter of $\obj$ on $\set$, and the smoothness parameter of $\obj$ on $\set$ (i.e., the Lipschitz-continuity parameter of its gradient if it exists), all with respect to $\| \cdot\|_2$. Note that if $x= \Vect(X)$ for a given matrix $X$, $\|x\|_2= \|X\|_F$. We denote by $\bound_f, \lip_f, L_f$ the same quantities defined  on the domain of $f$, e.g., $\bound_f =\bound_f^{\dom f}$. 
We denote by $\classfunc_{\bound, \lip, \smooth}$ the class of functions $f$ such that $\bound_f =\bound, \lip_f = \lip, \smooth_f = \smooth$. In the following, we allow the quantities $\bound_f, \lip_f, \smooth_f$ to be infinite if for example the function is unbounded or the smoothness constant is not defined. The procedures presented below output infinite estimates if the combinations of the smoothness properties do not allow for finite estimates. On the other hand, they provide finite estimates automatically if they are available. 
In the following we denote $\bigotimes_{t=1}^ \horizon B_{\setparam_t}(\reals^{\dimctrl_t})  = \{\ctrls =(\ctrl_1;\ldots;\ctrl_\horizon)\in \reals^ {\sum_{t=1}^\horizon \dimctrl_t}:  \ctrl_t \in \reals^{\dimctrl_t}, \|\ctrl_t\|_2 \leq \setparam_t\} $. 

\subsection{Convergence rate to a stationary point}
We recall the convergence rate to a stationary point of a gradient descent and a stochastic gradient descent on constrained problems. 
\begin{theorem}[{\citealp[Theorems 1 and 2]{ghadimi2016mini}}]\label{thm:conv_grad}
	Consider problems of the form
	\begin{align*}
		\mbox{(i)}\qquad \min_{\ctrls \in \reals^\dimvar} \quad  & \left\{\Obj(\ctrls) := \obj(\chainoutput(\ctrls)) + \reg(\ctrls)\right\}, & \mbox{or} \quad  (ii) \qquad \min_{\ctrls \in \reals^\dimvar} \quad  & \left\{\Obj(\ctrls) := \frac{1}{n}\sum_{i=1}^n \obj_i(\chainoutput_i(\ctrls)) + \reg(\ctrls)\right\}, \\
		\mbox{subject to} \quad & \ctrls \in \set,  & 
		\mbox{subject to} \quad & \ctrls \in \set,
	\end{align*}
	where $\set$ is a closed convex set and  $\Obj$ is $\smooth_\Obj^\set$ smooth on $\set$.  For problem (ii), consider that we have access to an unbiased estimate $\widehat \nabla \Obj(\ctrls)$ of $\nabla \Obj(\ctrls)$ with a variance bounded as $\Expect(\|\widehat \nabla \Obj(\ctrls)-\nabla \Obj(\ctrls)\|_2^2)\leq \sigma^2$.
	
	A projected gradient descent applied on problem (i) with step-size $\stepsize = (\smooth^\set_{\Obj})^{-1}$
	converges to an $\varepsilon$-stationary point in at most
	\[
	\bigO\left(\frac{\smooth_\Obj^\set (\Obj(\ctrls_0)- \Obj^*)}{\epsilon^2}\right)
	\]
	iterations, where $\ctrls_0$ is the initial point and $\Obj^*= \min_{\ctrls \in \set}\Obj(\ctrls)$. 
	
	A stochastic projected gradient descent applied on  problem $(ii)$ with step-size $\stepsize = (2\smooth^\set_{\Obj})^{-1}$ converges in expectation to an $(\varepsilon + \sigma)$-stationary point in at most
	\[
	\bigO\left(\frac{\smooth_\Obj^\set (\Obj(\ctrls_0)- \Obj^*)}{\epsilon^2}\right)
	\]
	iterations, where $\ctrls_0$ is the initial point and $\Obj^*= \min_{\ctrls \in \set}\Obj(\ctrls)$. 
\end{theorem}

\paragraph{Remarks}
\begin{enumerate}
	\item Since a gradient descent is monotonically decreasing, a gradient descent applied to the unconstrained problem converges to an $\varepsilon$-stationary point in at most \[
	\bigO\left(\frac{\smooth_\Obj^{S_0} (\Obj(\ctrls_0)- \Obj^*)}{\epsilon^2}\right)
	\] 
	iterations, where $S_0=\{\ctrls\in \reals^p: \Obj(\ctrls) \leq \Obj(\ctrls_0)\}$ is the initial sub-level set.
	\item Tighter rates of convergence can be obtained for the finite-sum problem (ii) by using variance reduction methods and by varying mini-batch sizes \citep{li2018simple}. They would then depend on the smoothness constants of the objective or the maximal smoothness of the components on $\set$, i.e., $\max_{i=1, \ldots, n}\smooth_{\obj_i\circ \chainoutput_i +\reg}^\set$.
\end{enumerate}

The smoothness of the objectives $\Obj$ defined in Theorem~\ref{thm:conv_grad} can be derived from the smoothness properties of their components. 
\begin{restatable}{proposition}{smoothobj}
	Consider a closed convex set $\set \subset\reals^p$, $\chainoutput \in \classfunc^\set_{\bound_\chainoutput^\set, \lip_\chainoutput^\set, \smooth_\chainoutput^\set}$, $\reg \in \classfunc_{\smooth_\reg}$ and $\obj \in \classfunc_{\lip_\obj, \smooth_\obj}$ with $\ell_\obj = +\infty$ if $\obj$ is not Lipschitz-continuous. The smoothness of $\Obj = \obj\circ\chainoutput +\reg$ on $\set$ is bounded as
	\[
	\smooth_\Obj^\set \leq \smooth_\chainoutput^\set \tilde \ell_\obj^\set + \left(\ell_\chainoutput^\set\right)^2\smooth_\obj + \smooth_\reg,
	\]
	where $\tilde \ell_\obj^\set = \min \{\ell_\obj, \min_{z \in \chainoutput(\set)}\|\nabla \obj(z)\|_2 + \smooth_\obj \lip_\chainoutput^\set\diam^\set\}$, where $\diam^\set = \sup_{x,y \in \set} \|x-y\|_2$.
\end{restatable}
What remain to characterize are the smoothness properties of a chain of computations.

\subsection{Smoothness estimates}

We present the smoothness computations for a deep network. Generic estimations of the smoothness properties of a chain of computation are presented in  Appendix~\ref{sec:smooth_proofs}.
The propositions below give \emph{upper bounds} on the smoothness constants of the function achieved through chain-composition. For a trivial composition such as $f\circ f^{-1}$, the upper bound is clearly loose. The upper bounds we present here are informative for non-trivial architectures.

The estimation is done by a forward pass on the network, as illustrated in Fig.~\ref{fig:smoothness}. The reasoning is based on the following lemma. 
\begin{restatable}{lemma}{smoothgen}\label{lem:smooth}
	Consider a chain $\chain$ of $\horizon$  computations $\dyn_t \in \classfunc_{\lip_{\dyn_t}, \smooth_ {\dyn_t}}$ initialized at some $\state_0 \in \reals^{\dimstate_0}$.
	\begin{enumerate}[nosep, leftmargin=*]
		\item[(i)] We have $\lip_{\chain_{\horizon, \state_0}} \leq \lip_\horizon$, where
		\begin{flalign*}
			\lip_{0}  = 0, \hspace{1ex} 	\lip_{t} = \lip_{\dyn_t} + \lip_{t-1} \lip_{\dyn_t},  \hspace{1ex}   \mbox{for $t \in \{1, \ldots, \horizon\}$.}
		\end{flalign*}
		\item[(ii)] We have $\smooth_{\chain_{\horizon, \state_0}} \leq \smooth_\horizon$, where
		\begin{align*}
			\smooth_{0}  = 0,  \hspace{1ex}  \smooth_t = \smooth_{t-1}\lip_{\dyn_t} + \smooth_{\dyn_t}(1+\lip_{t-1})^2,  \hspace{1ex}   \mbox{for $t \in \{1, \ldots, \horizon\}$.}
		\end{align*}
	\end{enumerate}
\end{restatable}

In the case of deep networks, the computations are not Lipschitz continuous due to the presence of bi-affine functions. Yet, provided that inputs of computations are bounded and that we have access to the smoothness of the computations, we can have an estimate of the Lipschitz-continuity of the computations restricted to these bounded sets.

\begin{corollary}
	Consider a chain $\chain$ of $\horizon$ of computations $\dyn_t \in \classfunc_{\bound_{\dyn_t}, \lip_{\dyn_t}, \smooth_ {\dyn_t}}$ initialized at some $\state_0 \in \reals^{\dimstate_0}$ and consider  $\set = \bigotimes_{t=1}^ \horizon B_{\setparam_t}(\reals^{\dimctrl_t}) $. Then the smoothness of the output of the chain $\chain_{\horizon, \state_0}$ on $\set$, can be estimated as in Lemma~\ref{lem:smooth} by replacing $\lip_{\dyn_t}$ with $\tilde \lip_{\dyn_t}$ defined by
	\begin{align*}
		\tilde \lip_{\dyn_t} & = \min\{\lip_{\dyn_t}, \smooth_{\dyn_t} (\bound_{t-1} + \setparam_t) + \|\nabla \dyn_t(0,0)\|_{2, 2}\},\\
		\bound_t & =  \min\{ \bound_{\dyn_t}, \tilde \lip_{\dyn_t} (\bound_{t-1} + \setparam_t) + \|\dyn_t(0,0)\|_2 \},
	\end{align*}
	for $t\in \{1, \ldots, \horizon\}$, with $\bound_0 = \|\state_0\|_2$.
\end{corollary}

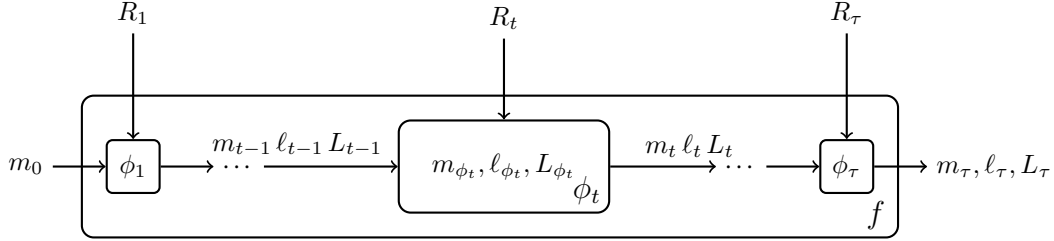
\begin{figure}

\begin{figure}[H]
	\begin{center}
		\begin{tikzpicture}
			\node (bound0) {$\bound_0$};
			\node[module] (dyn1) [right = 2em of bound0] {$\dyn_1$};
			\node (setparam1) [above=4em of dyn1] {$\setparam_1$};
			
			\node (dots1) [right=2em of dyn1]{\dots};
			
			\node (smoothl) [right=6em of dots1]{$\bound_{\dyn_t}, \lip_{\dyn_t}, \smooth_{\dyn_t}$};
			\node [bigmodule] (layerl) [fit=(smoothl)] {};
			\node [above] at (layerl.south east) {\hspace{-2em} \large $\dyn_t$};
			
			\node (setparaml) [above=4em of smoothl] {$\setparam_t$};
			
			\node (dots2) [right=4em of layerl] {\dots};
			
			\node[module] (dynt) [right=2em of dots2] {$\dyn_\horizon$};
			\node (setparamt) [above=4em of dynt] {$\setparam_\horizon$};
			
			\node (output) [right=2em of dynt] {$\bound_\horizon, \lip_\horizon, \smooth_\horizon$};
			
			\draw [->, thick] (bound0) -- (dyn1);
			\draw [->, thick] (dyn1) -- (dots1); 
			\draw [->, thick] (dots1) -- (layerl) node[near start, above]{$\bound_{t-1}\, \lip_{t-1}\, \smooth_{t-1}$};
			\draw [->, thick] (layerl) -- (dots2) node[near end, above]{$\bound_t\,  \lip_t\, \smooth_t$};
			\draw [->, thick] (dots2) -- (dynt);
			\draw [->, thick] (dynt) -- (output);

			\draw [->, thick] (setparam1) -- (dyn1);
			\draw [->, thick] (setparaml) -- (layerl);
			\draw [->, thick] (setparamt) -- (dynt);
			
			\node  [bigmodule, fit=(dyn1) (layerl) (dynt), minimum height = 5em] (chain){};
			\node [above] at (chain.south east) {\hspace{-2em} \large $\chain$};
		\end{tikzpicture}
	\end{center}
	\caption{Smoothness estimates computations. \label{fig:smoothness}}
\end{figure}

\vfill

\begin{algorithm}[H]\caption{Automatic smoothness computations \label{algo:smooth_comput}}
	\begin{algorithmic}[1]
		\State{{\bf Inputs:}
			\begin{enumerate}
				\item Chain of computations $\chain$ defined by $\dyn_t = \nonlin_t\circ \biaffine_t$ for $t \in\{1, \ldots, \horizon\}$  with  $\nonlin_t=  \nonlin_{t, \nbcomp_t} \circ \ldots \circ \nonlin_{t, 1}$ 
				\item Smoothness properties $\smooth_{\biaffine_t}, \lipp_{\biaffine_t}^\ctrl, \lipp_{\biaffine_t}^\state$ of the biaffine function $\biaffine_t \in \mathcal{B}_{\smooth_{\biaffine_t}, \lipp_{\biaffine_t}^\ctrl, \lipp_{\biaffine_t}^\state}$ 
				\item Smoothness properties $ \bound_{\nonlin_{t, i}}, \lip_{\nonlin_{t, i}}, \smooth_{\nonlin_{t, i}}$ of the nonlinear functions $\nonlin_{t, i} \in \classfunc_{\bound_{\nonlin_{t, i}}, \lip_{\nonlin_{t, i}}, \smooth_{\nonlin_{t, i}}}$
				\item Initial point $\state_0$
				\item Bounds $\setparam_t$ on the parameters
		\end{enumerate}}
		
		\State{Initialize $\bound_0 =\|\state_0\|_2$, $\lip_0 = 0$, $\smooth_0=0$}
		\For{$t=1, \ldots, \horizon$}
		\State{$\lip_{t, 0}^ \state = \smooth_{\biaffine_t} \setparam_t + \lipp^\state_{\biaffine_t}, \hspace{1em} \lip_{t, 0}^ \ctrl = \smooth_{\biaffine_{t}} \bound_{t-1}+ \lipp^\ctrl_{\biaffine_t}, \hspace{1em} \lip_{t, 0}^0 = 1$}
		\State{$\bound_{t, 0} = \lip_{t, 0}^ \state \bound_{t-1} +  \lip_{t, 0}^ \ctrl  \setparam_t + \|\biaffine_t(0,0)\|_2$}
		\State{$\smooth_{t, 0}=0$}
		\For{$j=1, \ldots, \nbcomp_t$}
		\State{$\tilde \lip_{\nonlin_{t, j}} = \min \{\lip_{\nonlin_{t, j}}, \ \|\nabla \nonlin_{t, j} (0) \|_2 + \smooth_{\nonlin_{t, j}} \bound_{t, j-1} \}$ \label{line:lip_bound}}
		\State{$\bound_{t, j} = \min\{\bound_{\nonlin_{t, j}},\  \| \nonlin_{t, j}(0)\|_2 + \tilde \lip_{\nonlin_{t, j}}\bound_{t, j-1}\}$\label{line:bound}}
		\State{${ \lip_{t, j}^0 =\tilde \lip_{\nonlin_{t, j}}\lip_{t, j-1}^0}$}
		\State{$\smooth_{t, j} =  \smooth_{t, j-1}\lip_{\nonlin_t, j} +  \smooth_{\nonlin_{t, j}}(\lip_{t, j-1}^ 0) ^ 2$}
		\EndFor
		\State{$	\bound_t =\bound_{t, \nbcomp_t}$}
		\State{$	\lip_t =  \lip_{t, 0}^ \state \lip_{t, \nbcomp_t}^0 \lip_{t-1} {+}  \lip_{t, 0}^\ctrl \lip_{t, \nbcomp_t}^0$}
		\State{$ 	\smooth_t  = 
			\smooth_{t-1} \lip_{t, 0}^ \state \lip_{t, \nbcomp_t}^0
			{+} (\smooth_{\biaffine_t}\setparam_t 
			{+} \lipp_{\biaffine_t})^ 2 \smooth_{t, \nbcomp_t} \lip_{t-1}^ 2 \newline
			\phantom{\hspace{35pt}}		 {+} 2\left( (\smooth_{\biaffine_t} \bound_{t-1} 
			{+} \lipp_{\biaffine_t}^ \ctrl)(\smooth_{\biaffine_t}\setparam_{t} 
			{+} \lipp_{\biaffine_t}^ \state ) \smooth_{t, \nbcomp_t}
			{+} \smooth_{\biaffine_t} \lip_{t, \nbcomp_t}^0\right) \lip_{t-1}\newline
			 	\phantom{\hspace{35pt}}	{+} (\smooth_{\biaffine_t}\bound_{t-1} {+} \lipp_{\biaffine_t}^\ctrl)^ 2 \smooth_{t, \nbcomp_t}$\label{line:smooth}}
		\EndFor
		\State{ 
			{\bf Output:}
			$\bound_\horizon$, $\lip_\horizon$, $\smooth_\horizon$ 
		}
	\end{algorithmic}
\end{algorithm}
\end{figure}
We specialize the result to deep networks.
We denote by $\mathcal{B}_{\smooth, \lipp^{\ctrl}, \lipp^{ \state}}$ the set of $\smooth$-smooth bi-affine functions $b$ such that $\|\nabla_\ctrl \biaffine(0, 0)\|_{2, 2} = \lipp^\ctrl$, $\|\nabla_\state \biaffine (0, 0)\|_{2, 2} = \lipp^\state$ , i.e., functions of the form
\[
\biaffine(\state, \ctrl) = \bilinear(\state, \ctrl) + \bilinear^\ctrl(\ctrl) + \bilinear^\state(\state) +\linearcste,
\]
with $\bilinear$  bilinear and $\smooth$-smooth, $\bilinear^\ctrl$, $\bilinear^\state$ linear and $\lipp^\ctrl$, $\lipp^\state$ Lipschitz continuous respectively and  $\linearcste$  a constant vector.

\begin{restatable}{proposition}{smoothdeep}\label{prop:smoothdeep}
	Consider a chain $\chain$ of $\horizon$ computations whose layers $\dyn_{t}$ are defined by
	\[
	\dyn_{t}( \state_{t-1}, \ctrl_t) = \nonlin_t(\biaffine_t( \state_{t-1}, \ctrl_t)),
	\]
	for $t\in\{1,\ldots, \horizon\}$, where $\biaffine_t \in \mathcal{B}_{\smooth_{\biaffine_t}, \lipp_{\biaffine_t}^\ctrl, \lipp_{\biaffine_t}^\state}$, 
	and $\nonlin_t$ is decomposed as
	\[
	\nonlin_t =  \nonlin_{t, \nbcomp_t} \circ \ldots \circ \nonlin_{t, 1},
	\]
	with $\nonlin_{t, i} \in \classfunc_{\bound_{\nonlin_{t, i}}, \lip_{\nonlin_{t, i}}, \smooth_{\nonlin_{t, i}}}$.
	Consider  $\set = \bigotimes_{t=1}^ \horizon B_{\setparam_t}(\reals^{\dimctrl_t}) $.
	The outputs $\bound_\horizon$, $\lip_\horizon$ and $\smooth_\horizon$ of Algo.~\ref{algo:smooth_comput} satisfy $\bound_{\chain_{\horizon, \state_0}}^\set \leq \bound_\horizon$, $\lip_{\chain_{\horizon, \state_0}}^\set \leq \lip_\horizon$, $\smooth_{\chain_{\horizon, \state_0}}^\set \leq \smooth_\horizon$.
\end{restatable}
The proof of the above lemma is the consequence of simple technical lemmas provided in Appendix~\ref{sec:smooth_proofs}.
Note that the smoothness of the chain with respect to its input given a fixed set of parameters can also easily be estimated by a similar method; see Corollary~\ref{cor:smooth_state} in  Appendix~\ref{sec:smooth_proofs}.

The smoothness properties of a chain of composition around a given point follows then directly as stated in the following corollary. 
\begin{restatable}{corollary}{smoothball}
	Consider a chain $\chain$ of $\horizon$ computations as defined in Prop.~\ref{prop:smoothdeep} and $\ctrls^* = (\ctrl_1^*; \ldots, \ctrl_\horizon^*) \in \reals^\dimvar$. The smoothness properties of $\chain$ on $\set' = \{\ctrls= (\ctrl_1; \ldots; \ctrl_\horizon) \in \reals^\dimvar: \ \forall t \in\{1,\ldots, \horizon\}, \ \|\ctrl_t - \ctrl_t^*\|\leq \setparam_t'\}$ are given as in Prop.~\ref{prop:smoothdeep} by considering \
	\begin{align*}
		R_t' \quad & \mbox{in place of} \quad R_t, \\
		\lipp_{\bilinear^\state_t} + \smooth_{\bilinear_t}\|\ctrl_t^*\|_{\normidx} \quad &  \mbox{in place of} \quad \lipp_{\bilinear^\state_t}, \\
		\|\linearcste_t\|_{\normidx} + \lipp_{\bilinear^\ctrl_t}\|\ctrl_t^*\|_{\normidx} \quad & \mbox{in place of} \quad \|\linearcste_t\|_{\normidx}.
	\end{align*}
\end{restatable}

\section{Application}\label{sec:appli}
We apply our framework to assess the smoothness properties of the Visual Geometry Group (VGG) deep network used for image classification~\citep{simonyan2014very}. 

\subsection{VGG network}
The VGG Network is a benchmark network for image classification with deep  networks. The objective is to classify images among $1000$ classes. Its architecture is composed of 16 layers described in Appendix~\ref{sec:net_examples}.  
We consider in the following smoothness properties for mini-batches with size $\batchsize=128$, i.e., by concatenating $\batchsize$ chains of computations $\chain^{(i)}$ each defined by a different input. This highlights the impact of the size of the mini-batch for batch-normalization.

\paragraph{Smoothness computations}
To compute the Lipschitz-continuity and smoothness parameters, we recall the list of Lipschitz continuity and smoothness constants of each layer of interest. For the bilinear and linear operations we denote by $\tensornorm$ the smoothness of the bilinear operation $\bilinear$ and  by $\lip$ the Lipschitz-continuity of the linear operation $\bilinear^\ctrl$. The smoothness constants of interest are
\begin{enumerate}[topsep=1ex,itemsep=-1ex,partopsep=1ex,parsep=1ex, leftmargin=*]
	\item $\lip_{\conv} = \sqrt{\batchsize}\left \lceil \frac{k}{s} \right\rceil$, $\tensornorm_{\conv} = \left \lceil \frac{k}{s} \right\rceil$, where the patch is of size $k\times k$ and the stride is $s$,
	\item $\lip_{\full} = \sqrt{\batchsize}$, $\tensornorm_{\full} = 1$,
	\item $\lip_{\Relu} = 1$, $\smooth_{\Relu}$ not defined,
	\item $\lip_{\softmax} = 2$, $\smooth_{\softmax} =4$,
	\item $\lip_{\maxpool} = 1$, $\smooth_{\maxpool}$ not defined,
	\item $\lip_{\log} = 2$, $\smooth_{\log} = 2$.
\end{enumerate}
A Lipschitz-continuity estimate of this architecture can then be computed using Prop.~\ref{prop:smoothdeep}  on a Cartesian product of balls $\set = \{\var = (\ctrl_1; \ldots; \ctrl_{16}): \|\ctrl_t\|_2 \leq \setparam \}$ for $\setparam =1$ for example.

\subsection{Variations of VGG}
\paragraph{Smooth VGG}
First, the VGG architecture can be made continuously differentiable by considering the soft-plus activation instead of the ReLU activation and average pooling instead of the max-pooling operation. As shown in
Appendix~\ref{sec:smooth_list}, we have
\begin{enumerate}[topsep=1ex,itemsep=-1ex,partopsep=1ex,parsep=1ex, leftmargin=*]
	\item $\lip_{\avgpool} = 1$, $\smooth_{\avgpool} = 0$,
	\item $\lip_{\softplus} = 1$, $\smooth_{\softplus} = 1/4$.
\end{enumerate}
Denoting $\lip_{\operatorname{VGG}}$ and $\lip_{\operatorname{VGG-smooth}}$ the Lipschitz-continuity estimates of the original VGG network and the modified original network on a Cartesian product of balls $\set = \{\ctrls = (\ctrl_1; \ldots; \ctrl_{16}): \|\ctrl_t\|_2 \leq 1 \}$ with $\|x\|_2 = 1$, we get using Prop.~\ref{prop:smoothdeep},
$
\frac{|\lip_{\operatorname{VGG}} - \lip_{\operatorname{VGG-smooth}}|}{\lip_{\operatorname{VGG}}} \leq 10^{-4} .
$

\paragraph{Batch-normalization effect}
We can also compare the smoothness properties of the smoothed network with the same network modified by adding the batch-normalization layer for $\batchsize$ inputs and $\regbatchnorm$ normalization parameter at each convolutional layer. As shown in
Appendix~\ref{sec:smooth_list}, the  batch-normalization satisfies
\begin{enumerate}[topsep=1ex,itemsep=-1ex,partopsep=1ex,parsep=1ex, leftmargin=*]
	\item $\bound_{\batchnorm}{= }\diminput \batchsize$, $\lip_{\batchnorm}{ = }2\regbatchnorm^{-1/2}$,  $\smooth_{\batchnorm} {=} 2\batchsize^{-1/2} \regbatchnorm^{-1}$.
\end{enumerate}
Denoting $\lip_{\operatorname{VGG-smooth}}$, $\smooth_{\operatorname{VGG-smooth}}$ and $\lip_{\operatorname{VGG-batch}}$, $\smooth_{\operatorname{VGG-batch}}$ the Lipschitz-continuity and smoothness estimates of the smoothed VGG network with and without batch-normalization respectively on a Cartesian product of balls $\set = \{\ctrls = (\ctrl_1; \ldots; \ctrl_{16}): \|\ctrl_t\|_2 \leq 1 \}$ with $\|x\|_2 = 1$, we get using Prop.~\ref{prop:smoothdeep},
\begin{align*}
\mbox{for} \quad \regbatchnorm=10^{-2}, \qquad  & \begin{array}{cc}
\lip_{\operatorname{VGG-smooth}} & \leq \lip_{\operatorname{VGG-batch}} \\
\smooth_{\operatorname{VGG-smooth}}  &\leq \smooth_{\operatorname{VGG-batch}}
\end{array}\\
\mbox{for} \quad \regbatchnorm=10^2, \qquad  & \begin{array}{cc}
\lip_{\operatorname{VGG-smooth}} & \geq \lip_{\operatorname{VGG-batch}} \\
\smooth_{\operatorname{VGG-smooth}}  &\geq \smooth_{\operatorname{VGG-batch}}
\end{array}
\end{align*}
Intuitively, the batch-norm bounds the output of each layer, mitigating the increase of $\bound_t$ in the computations of the estimates of the smoothness in lines~\ref{line:lip_bound} and~\ref{line:bound} of Algo.~\ref{algo:smooth_comput}. Yet, for a small $\regbatchnorm$, this effect is balanced by the non-smoothness of the batch-norm layer (which for $\regbatchnorm\rightarrow 0$ tends to have an infinite slope around 0).

\paragraph{Acknowledgments}
This work was supported by NSF CCF-1740551, NSF DMS-1839371, the program ``Learning in Machines and Brains'', and faculty research awards.

{\bibliographystyle{abbrv}
\bibliography{autodiff.bib}}

\begin{thebibliography}{10}

\bibitem{tensorflow2015-whitepaper}
M.~Abadi, A.~Agarwal, P.~Barham, E.~Brevdo, Z.~Chen, C.~Citro, G.~S. Corrado,
  A.~Davis, J.~Dean, M.~Devin, S.~Ghemawat, I.~Goodfellow, A.~Harp, G.~Irving,
  M.~Isard, Y.~Jia, R.~Jozefowicz, L.~Kaiser, M.~Kudlur, J.~Levenberg,
  D.~Man\'{e}, R.~Monga, S.~Moore, D.~Murray, C.~Olah, M.~Schuster, J.~Shlens,
  B.~Steiner, I.~Sutskever, K.~Talwar, P.~Tucker, V.~Vanhoucke, V.~Vasudevan,
  F.~Vi\'{e}gas, O.~Vinyals, P.~Warden, M.~Wattenberg, M.~Wicke, Y.~Yu, and
  X.~Zheng.
\newblock {TensorFlow}: Large-scale machine learning on heterogeneous systems,
  2015.

\bibitem{martin2002neural}
M.~Anthony and P.~Bartlett.
\newblock {\em Neural Network Learning: Theoretical Foundations}.
\newblock Cambridge University Press, 2009.

\bibitem{baur1983complexity}
W.~Baur and V.~Strassen.
\newblock The complexity of partial derivatives.
\newblock {\em Theoretical computer science}, 22(3):317--330, 1983.

\bibitem{bertsekas2005dynamic}
D.~P. Bertsekas.
\newblock {\em Dynamic programming and optimal control}.
\newblock Athena Scientific, 3rd edition, 2005.

\bibitem{duda2012pattern}
R.~Duda, P.~Hart, and D.~Stork.
\newblock {\em Pattern classification}.
\newblock John Wiley \& Sons, 2nd edition, 2012.

\bibitem{dunn1989efficient}
J.~C. Dunn and D.~P. Bertsekas.
\newblock Efficient dynamic programming implementations of {N}ewton's method
  for unconstrained optimal control problems.
\newblock {\em Journal of Optimization Theory and Applications}, 63(1):23--38,
  1989.

\bibitem{duvenaud2015convolutional}
D.~K. Duvenaud, D.~Maclaurin, J.~Iparraguirre, R.~Bombarell, T.~Hirzel,
  A.~Aspuru-Guzik, and R.~P. Adams.
\newblock Convolutional networks on graphs for learning molecular fingerprints.
\newblock In {\em Advances in Neural Information Processing Systems 28}, 2015.

\bibitem{ghadimi2016mini}
S.~Ghadimi, G.~Lan, and H.~Zhang.
\newblock Mini-batch stochastic approximation methods for nonconvex stochastic
  composite optimization.
\newblock {\em Mathematical Programming}, 155(1-2):267--305, 2016.

\bibitem{goodfellow2016deep}
I.~Goodfellow, Y.~Bengio, and A.~Courville.
\newblock {\em Deep Learning}.
\newblock The MIT Press, 2016.

\bibitem{griewank2012}
A.~Griewank.
\newblock Who invented the reverse mode of differentiation?
\newblock {\em Documenta Mathematica}, Optimization stories:389--400, 2012.

\bibitem{hocking2011clusterpath}
T.~D. Hocking, A.~Joulin, F.~Bach, and J.-P. Vert.
\newblock Clusterpath: an algorithm for clustering using convex fusion
  penalties.
\newblock In {\em Proceedings of the 28th International Conference on Machine
  Learning}, 2011.

\bibitem{horn2012matrix}
R.~A. Horn and C.~R. Johnson.
\newblock {\em Matrix analysis}.
\newblock Cambridge university press, 2nd edition, 2012.

\bibitem{ioffe2015batch}
S.~Ioffe and C.~Szegedy.
\newblock Batch normalization: Accelerating deep network training by reducing
  internal covariate shift.
\newblock In {\em Proceedings of the 32nd International Conference on Machine
  Learning}, volume~37, pages 448--456, 2015.

\bibitem{kim1984estimate}
K.~V. Kim, Y.~E. Nesterov, and B.~Cherkasskii.
\newblock An estimate of the effort in computing the gradient.
\newblock In {\em Doklady Akademii Nauk}, volume 275, pages 1306--1309. Russian
  Academy of Sciences, 1984.

\bibitem{krizhevsky2012imagenet}
A.~Krizhevsky, I.~Sutskever, and G.~E. Hinton.
\newblock {ImageNet} classification with deep convolutional neural networks.
\newblock In {\em Advances in Neural Information Processing Systems 25}, 2012.

\bibitem{lecun1988theoretical}
Y.~Lecun.
\newblock A theoretical framework for back-propagation.
\newblock In {\em 1988 Connectionist Models Summer School, CMU, Pittsburg, PA},
  1988.

\bibitem{li2018simple}
Z.~Li and J.~Li.
\newblock A simple proximal stochastic gradient method for nonsmooth nonconvex
  optimization.
\newblock In {\em Advances in Neural Information Processing Systems 31}, pages
  5564--5574, 2018.

\bibitem{nocedal2006numerical}
J.~Nocedal and S.~J. Wright.
\newblock {\em Numerical Optimization}.
\newblock Springer, 2nd edition, 2006.

\bibitem{paszke2017automatic}
A.~Paszke, S.~Gross, F.~Massa, A.~Lerer, J.~Bradbury, G.~Chanan, T.~Killeen,
  Z.~Lin, N.~Gimelshein, L.~Antiga, A.~Desmaison, A.~Kopf, E.~Yang, Z.~DeVito,
  M.~Raison, A.~Tejani, S.~Chilamkurthy, B.~Steiner, L.~Fang, J.~Bai, and
  S.~Chintala.
\newblock Pytorch: An imperative style, high-performance deep learning library.
\newblock In {\em Advances in Neural Information Processing Systems 32}, 2019.

\bibitem{pennington2014glove}
J.~Pennington, R.~Socher, and C.~D. Manning.
\newblock {GloVe}: Global vectors for word representation.
\newblock In {\em Empirical Methods in Natural Language Processing (EMNLP)},
  pages 1532--1543, 2014.

\bibitem{roulet2019iterativeArxiv}
V.~Roulet, S.~Srinivasa, D.~Drusvyatskiy, and Z.~Harchaoui.
\newblock Iterative linearized control: stable algorithms and complexity
  guarantees.
\newblock In {\em Proceedings of the 36th International Conference on Machine
  Learning}, 2019.
\newblock Long version.

\bibitem{rumelhart1985learning}
D.~E. Rumelhart, G.~E. Hinton, and R.~J. Williams.
\newblock Learning representations by back-propagating errors.
\newblock {\em Nature}, 323(6088):533--536, 1986.

\bibitem{shalev2014understanding}
S.~Shalev-Shwartz and S.~Ben-David.
\newblock {\em Understanding Machine Learning: From Theory to Algorithms}.
\newblock Cambridge University Press, 2014.

\bibitem{simonyan2014very}
K.~Simonyan and A.~Zisserman.
\newblock Very deep convolutional networks for large-scale image recognition.
\newblock In {\em International Conference on Learning Representations}, 2015.

\bibitem{werbos1994roots}
P.~Werbos.
\newblock {\em The Roots of Backpropagation: From Ordered Derivatives to Neural
  Networks and Political Forecasting}.
\newblock Wiley-Interscience, 1994.

\end{thebibliography}

\clearpage
\appendix
\section{Notations}\label{sec:notations}
\subsection{Matrices}
For a matrix $M \in \reals^{d \times n}$, we denote by $\Vect(M)$ the concatenation of the columns of $M$.
We denote
$\|M\|_{\opnormidx} = \sup_{x \neq 0, y \neq 0} \frac{x^\top M y}{\|x\|_2\|y\|_2}$ its norm induced by the Euclidean norm and $\|M\|_F = \sqrt{\sum_{ij}M_{ij}^2}$ its Frobenius norm. 

\subsection{Tensors}
A tensor $\mathcal{A} = (a_{ijk})_{i\in\{1,\ldots, d\}, j\in \{1,\ldots,n\}, k\in \{1,\ldots, p\}} \in \reals^{d \times n \times p}$ is represented as a list of matrices $\mathcal{A} = (A_1,\ldots, A_p)$ where $A_k = (a_{ijk})_{i\in\{1,\ldots, d\}, j\in \{1,\ldots,n\}} \in \reals^{d\times n}$ for $ k\in \{1,\ldots p\}$. 

\paragraph{Tensor-matrix product}Given matrices $P \in \reals^{d \times d'}, Q \in \reals^{n \times n'}, R \in \reals^{p \times p'}$, we denote
\[
\mathcal{A}[P, Q, R] = \left(\sum_{k=1}^{p} R_{k,1}P^\top A_k Q,  \ldots,  \sum_{k=1}^{p} R_{k,p'}P^\top A_k Q \right) \in \reals^{d'\times n'\times p'}
\]
If $P, Q$ or $ R$ are identity matrices, we use the symbol ``$\: \cdot\: $" in place of the identity matrix. For example, we denote $\mathcal{A}[P, Q, \idm_p] = \mathcal{A}[P,Q, \cdot] = \left(P^\top A_1 Q,  \ldots,  P^\top A_p Q \right)$.
\begin{fact}\label{fact:tensormat_rule}
	Let $\mathcal{A} \in \reals^{d\times p\times n}$, $P \in \reals^{d \times d'}, Q \in \reals^{p \times p'}, R \in \reals^{n \times n'}$ and $S \in \reals^{d'\times d''}, T \in \reals^{p' \times p''}, U \in \reals^{ n'\times n''}$. Denote $\mathcal{A}'= \mathcal{A}[P, Q, R] \in \reals^{d'\times p'\times n'}$. Then
	\[
	\mathcal{A}'[S, T, U] = \mathcal{A}[PS, QT, RU] \in \reals^{d''\times p''\times n''}.
	\]
\end{fact}
\paragraph{Flat tensors}
If $P, Q$ or $R$ are vectors we consider the flatten object. In particular, for $x\in \reals^d, y\in \reals^n$, we denote
\[
\mathcal{A}[x, y, \cdot] =  \left(\begin{matrix}
	x^\top A_1 y\\ \vdots \\ x^\top A_py
\end{matrix}
\right) \in \reals^p,
\]
rather than having $\mathcal{A}[x, y, \cdot] \in \reals^{1 \times 1\times p}$.
Similarly, for $z\in \reals^p$, we have
\[
\mathcal{A}[\cdot, \cdot, z] = \sum_{k=1}^p z_kA_k \in \reals^{d\times n}.
\]
\paragraph{Transpose}
For a tensor $\mathcal{A} =(A_1, \ldots, A_p)\in \reals^{d, n, p}$ we denote $\mathcal{A}^t = (A_1^\top, \ldots, A_p^\top)\in \reals^{n, d, p}$.
\paragraph{Outer product} We denote the outer product of three vectors $x\in \reals^d, y\in \reals^n, z \in \reals^p$ as $x \boxtimes y \boxtimes z \in \reals^{d\times n\times p}$ such that 
\[
(x \boxtimes y \boxtimes z)_{ijk} =x_iy_jz_k.
\]
\paragraph{Tensor norm}
We define the norm of a tensor $\mathcal{A}$ induced by the Euclidean norm as follows.
\begin{definition}\label{def:norm_tensor}
	The norm of a tensor $\mathcal{A}$ induced by the Euclidean norm is defined as
	\begin{equation}\label{eq:norm_tensor}
		\|\mathcal{A}\|_{\tensnormidx} = \sup_{x \neq 0, y \neq 0, z \neq 0} \frac{\mathcal{A}[x,y,z]}{\|x\|_2 \|y\|_2 \|z\|_2}.
	\end{equation}
\end{definition}
\begin{fact}\label{fact:tensor_norm}
	The tensor norm satisfies the following properties, for a given tensor $\mathcal{A} \in \reals^{d\times n\times p}$,
	\begin{enumerate}
		\item $\|\mathcal{A}\|_{\tensnormidx} = \|\mathcal{A}^t\|_{\tensnormidx}$,
		\item $\|\mathcal{A}[P,Q,R]\|_{\tensnormidx} \leq  \|\mathcal{A}\|_{\tensnormidx} \|P\|_{\opnormidx}\|Q\|_{\opnormidx}\|R\|_{\opnormidx}$ for $P,Q,R$ with appropriate sizes,
		\item $\|\mathcal{A}\|_{\tensnormidx} = \sup_{z\neq 0} \frac{\|\sum_{k=1}^{p} z_k A_k \|_{\opnormidx}}{\|z\|_2}$.
	\end{enumerate} 
\end{fact}

\subsection{Gradients}
For a multivariate function $f :\reals^d \mapsto \reals^n$, composed of  $n$ real functions $f^{(j)}$ for $j\in \{1, \ldots n\}$, we denote $\nabla f( x) = (\nabla f^{(1)}( x), \ldots, \nabla f^{(n)}( x))\in \reals^{d \times n}$, that is  the transpose of its Jacobian on $ x$, $\nabla f( x) = (\frac{\partial f^{(j)} }{\partial x_i}( x))_{\substack{i\in \{1, \ldots, d\}, j\in \{1, \ldots, n\}}} \in \reals^{d \times n}$. 
We represent its 2nd order information by a tensor $\nabla^2 f( x) = (\nabla^2 f^{(1)}( x), \ldots, \nabla^2 f^{(n)}( x))\in \reals^{d \times d \times n}$. 

\begin{fact}\label{fact:lip_smooth}
	We have for $f: \reals^d \rightarrow \reals^n$, twice differentiable, and $\set \subset \dom f$ convex,
	\begin{align*}
		\lip_f^\set  = \sup_{\substack{x, y \in \set \\ x\neq y}} \frac{\|f(x)-f(y)\|_2}{\|x-y\|_2} = \sup_{x\in  \set}\|\nabla f(x)\|_{\opnormidx}, \quad  
		\smooth_f^\set  = \sup_{\substack{x, y \in \set \\ x\neq y}} \frac{\|\nabla f(x)-\nabla f(y)\|_{\opnormidx}}{\|x-y\|_2}  = \sup_{x \in  \set} \|\nabla^2 f(x)\|_{\tensnormidx}.
	\end{align*}
	where $\|\nabla f(x)\|_{\opnormidx}$ denotes the operator norm of $\nabla f(x)$ and $\|\nabla^2 f(x)\|_{\tensnormidx}$ denotes the tensor norm of $\nabla^2 f(x)$ both with respect to the Euclidean norm.
\end{fact}
\begin{proof}
	We have for $x,y \in \set$,
	\begin{align*}
		\|f(x) - f(y)\|_2 = \|\int_{0}^{1}\nabla f(x + t(y-x))^\top (y-x) dt\|_2 \leq \sup_{x \in \set}\|\nabla f(x)\|_{\opnormidx} \|x-y\|_2, \\
		\|\nabla f(x) - \nabla f(y)\|_{\opnormidx} = \|\int_{0}^{1}\nabla^2f(x+t(y-x))[y-x, \cdot, \cdot] dt\|_{\opnormidx} \leq \sup_{x \in \set} \|\nabla^2 f(x)\|_{\tensnormidx}\|x-y\|_2,
	\end{align*}
	which gives $\lip_f^\set \leq \sup_{x\in  \set}\|\nabla f(x)\|_{\opnormidx}$ and $\smooth_f^\set \leq \sup_{x \in  \set} \|\nabla^2 f(x)\|_{\tensnormidx}$. The equalities come from the definitions of the gradient and the Hessian. 
\end{proof}

For a real function, $f: \reals^d \times \reals^p \mapsto \reals$, whose value is denoted $f(x,y)$, we decompose its gradient $\nabla f( x,  y) \in \reals^{d +p}$ on $( x,  y) \in \reals^d \times \reals^p$ as 
\[
\nabla f( x,  y) =\left(
\begin{matrix}
	\nabla_x f( x,  y) \\
	\nabla_y f( x,  y)
\end{matrix}\right) \qquad \mbox{with} \qquad \nabla_x f(x,y) \in \reals^d, \quad \nabla_y f(x,y) \in \reals^p.
\]
Similarly we decompose its Hessian $\nabla f( x,  y) \in \reals^{(d +p)\times(d +p) }$ on blocks  that correspond to the variables $(x,y)$ as follows
\begin{gather*}
	\nabla^2 f( x,  y) =\left(
	\begin{matrix}
		\nabla_{xx} f( x,  y) & \nabla_{xy} f(x,y)\\
		\nabla_{yx} f( x,  y) & \nabla_{yy} f(x,y) 
	\end{matrix}\right) 
	\\
	\qquad \mbox{with} \qquad \nabla_{xx} f(x,y) \in \reals^{d\times d}, \quad \nabla_{yy} f(x,y) \in \reals^{p\times p}, \quad  \nabla_{xy} f(x,y) = \nabla_{yx} f( x,  y)^\top \in \reals^{d\times p}.
\end{gather*}
Given a function $f: \reals^{d + p} \mapsto \reals^{n}$ and $( x,  y) \in \reals^d \times \reals^p$, we denote $\nabla_x f( x, y ) = (\nabla_x f^{(1)}( x,  y), \ldots, \nabla_x f^{(n)}( x,  y))\in \reals^{d \times n}$ and  we define similarly $\nabla_y f( x,  y)\in \reals^{p \times n}$. 

For its second order information we define  $\nabla_{xx} f( x,  y) = (\nabla_{xx} f^{(1)}( x,  y), \ldots, \nabla_{xx} f^{(n)}( x,  y))$, similarly for $\nabla_{xx} f( x,  y)$.
Dimension of these definitions are
\begin{gather*}
	\nabla_x f( x,  y) \in \reals^{d \times n}, \quad \nabla_y f( x,  y) \in \reals^{p \times n}, \\
	\nabla_{xx} f( x,  y) \in \reals^{d \times d \times n}, \quad \nabla_{yy} f( x,  y) \in \reals^{p \times p \times n},\\
	\nabla_{xy} f( x,  y) \in \reals^{d \times p \times n}, \quad \nabla_{yx} f( x,  y) \in \reals^{p \times d \times n}.
\end{gather*}

\subsection{Matrix functions}
For a differentiable matrix-valued multivariate function $g: \reals^d \rightarrow \reals^{p \times n}$ such that $g(x) = (g_{j, k}(x)_{1\leq j \leq p, 1\leq k \leq n})$, we denote its first order information as a tensor
\[
\nabla g(x) = \left(\frac{\partial g_{j, k}(x)}{\partial x_i}\right)_{1\leq i \leq d, 1\leq j \leq p, 1\leq k \leq n} \in \reals^{d \times p \times n}.
\]
This notation is consistent with previous ones, i.e., for $f: \reals^{d + p} \mapsto \reals^{n}$ and $g(y) = \nabla_y f( x,  y) \in \reals^{p \times n}$, then $\nabla g(y) =  \nabla_{xy}^2 f( x,  y) \in \reals^{d \times p \times n}$.
From previous definitions, we have the following fact.
\begin{fact}\label{fact:matgrad}
	For a differentiable matrix-valued multivariate function $g: \reals^d \rightarrow \reals^{p \times n}$, $A \in \reals^{p' \times p}, B \in \reals^{n \times n'}$, denoting $h(x) = Ag(x)B \in \reals^{p'\times n'}$, we have 
	\[\nabla h(x) = \nabla g(x)[\cdot, A^\top, B] \in \reals^{d \times p' \times n'}. \]
\end{fact}

\subsection{Bilinear functions}\label{ssec:bilinear}
\begin{definition}
	A function $\bilinear: \reals^d \times \reals^n \rightarrow \reals^p$ is bilinear if it is represented by a tensor $\mathcal{B} \in \reals^{d \times n \times p}$ such that for any $x \in \reals^d, y\in \reals^n$, 
	\[
	\bilinear(x,y) = \mathcal{B}[x,y, \cdot].
	\]
\end{definition}
The gradient of a bilinear function $\bilinear$ represented by a tensor $\mathcal{B} \in \reals^{d \times n \times p}$ at a point $x, y$ is given by 
\begin{equation}\label{eq:bilinear_grad}
	\nabla_x \bilinear(x,y) = \mathcal{B}[\cdot, y, \cdot] \in \reals^{d \times p}, \qquad \nabla_y \bilinear(x,y)=\mathcal{B}[x, \cdot, \cdot] \in \reals^{n\times p}.
\end{equation}
The Hessian of the bilinear function is given 
\begin{equation}\label{eq:bilinear_hess}
	\nabla^2_{xx} \bilinear(x, y) = 0, \quad \nabla^2_{yy} \bilinear(x, y) = 0, \quad \nabla^2_{xy} \bilinear(x, y) = \mathcal{B}, \quad \nabla^2_{yx} \bilinear(x, y) = \mathcal{B}^t.
\end{equation}
A bilinear function is not Lipschitz continuous as can be seen from Eq.~\eqref{eq:bilinear_grad}. It is  smooth w.r.t. the Euclidean norm with a smoothness constant given by the tensor norm of $\mathcal{B}$ as shown in the following proposition.

\begin{lemma}\label{lem:tensnorm}
	The smoothness of a bilinear function $\bilinear$ defined by a tensor $\mathcal{B}$ is upper bounded as $L_\bilinear \leq \|\mathcal{B}\|_{\tensnormidx}$.
\end{lemma}
\begin{proof}
	We have 
	\begin{align*}
		\|\nabla^2\bilinear(x,y)\|_{\tensnormidx} & = \sup_{z\neq 0} \frac{\|\sum_{k=1}^{p} z_k \tilde B_k \|_{\opnormidx}}{\|z\|_2},
	\end{align*}
	where $\nabla^2\bilinear(x,y) = (\tilde B_1, \ldots, \tilde B_p)$. We have by Eq.~\eqref{eq:bilinear_hess} that
	$ \sum_{k=1}^{p} z_k \tilde B_k$ is of the form
	\[
	\sum_{k=1}^{p} z_k \tilde B_k  = \left(\begin{matrix}
		0 & \sum_{k=1}^{p} z_k  B_k \\
		\sum_{k=1}^{p} z_k  B_k^\top & 0
	\end{matrix}\right)
	\]
	where $\mathcal{B} = (B_1, \ldots, B_p)$. Therefore  we get $\|\sum_{k=1}^{p} z_k \tilde B_k\|_{\opnormidx} = \|\sum_{k=1}^{p} z_k B_k\|_{\opnormidx}$, see \citep[Theorem 7.3.3.]{horn2012matrix}. Therefore
	\begin{align*}
		\|\nabla^2\bilinear(x,y)\|_2 & = \sup_{z\neq 0} \frac{\|\sum_{k=1}^{p} z_k \tilde B_k \|_{\opnormidx}}{\|z\|_2} =\sup_{z\neq 0} \frac{\|\sum_{k=1}^{p} z_k B_k \|_{\opnormidx}}{\|z\|_2}  = \|\mathcal{B}\|_{\tensnormidx}.
	\end{align*}
\end{proof}

\section{Oracle arithmetic complexity proofs}\label{sec:oracle_proofs}
\subsection{Feasibility of the optimization oracle steps}
The gradient step~\eqref{eq:grad_step} is always feasible for any $\kappa>0$, the Gauss-Newton step~\eqref{eq:gn_step} is feasible for any $\kappa>0$ if $\obj, \reg$ are convex and the Newton step is feasible for any $\kappa>0$ if $\obj\circ \chain$ and $\reg$ are convex. A sufficient condition for the Newton step to be feasible if $\obj\circ \chain$ is not convex but $\obj\circ \chain$ and $\reg$ are smooth is to choose $\kappa < (\smooth_{\obj\circ \chain} + \smooth_\reg)^{-1}$. 
In other words, the step-size must be chosen small enough such that the Newton step is a convex problem. 

\subsection{Optimization oracles as linear quadratic  problems}

\begin{lemma}\label{lem:derivatives_chain}
	Let $\chain$ be a chain of $\horizon$ computations $\dyn_t$. Let $\ctrls =(\ctrl_1; \ldots; \ctrl_\horizon)\in \reals^{\sum_{t=1}^\horizon \dimctrl_t}$, $\state_0 \in \reals^{\dimstate_0}$, denote $\state_t = \chain_t(\state_0, \ctrls)$ and $E_t^\top  = (0_{\dimctrl_t \dimstate_0}, 0_{\dimctrl_t \dimctrl_1}, \ldots, \idm_{\dimctrl_t\dimctrl_t}, \ldots, 0_{\dimctrl_t \dimctrl_\horizon}) \in \reals^{ \dimctrl_t \times (\dimstate_0 + \sum_{t=1}^\horizon \dimctrl_t)}$  such that $E_t^\top (\state_0; \ctrls)  = \ctrl_t$ for $t\in \{1, \ldots, \horizon\}$.
	\begin{enumerate}
		\item If $\dyn_t$ are differentiable, then
		\begin{equation}\label{eq:grad_decomp}
		\nabla \chain_{t}(\state_0, \ctrls) = \nabla \chain_{t-1}(\state_0, \ctrls) \nabla_{\state_{t-1}} \dyn_t(\state_{t-1}, \ctrl_t) + E_t\nabla_{\ctrl_t}  \dyn_t(\state_{t-1}, \ctrl_t) 
		\end{equation}
		with $\nabla  \chain_{0}(\state_0, \ctrls) = E_0$,  with $E_0^\top = (\idm_{\dimstate_0 \dimstate_0}, 0_{\dimstate_0 \dimctrl_1}, \ldots, \ldots, 0_{\dimstate_0 \dimctrl_\horizon}) \in \reals^{ \dimstate_0 \times (\dimstate_0 + \sum_{t=1}^\horizon \dimctrl_t)}$ such that $E_0^\top (\state_0;\ctrls) = \state_0$.
		\item If $\dyn_t$ are twice differentiable,
		\begin{align}
			\nabla^2 \chain_t(\state_0, \ctrls) = \: & \nabla^2 \chain_{t-1}(\state_0, \ctrls)[\cdot, \cdot, \nabla_{\state_{t-1}}\dyn_t(\state_{t-1}, \ctrl_t)] \label{eq:hess_decomp}\\
			& + \nabla^2_{\state_{t-1}\state_{t-1}}\dyn_t(\state_{t-1}, \ctrl_t)[\nabla \chain_{t-1}(\state_0, \ctrls)^\top, \nabla \chain_{t-1}(\state_0, \ctrls)^\top, \cdot ] \nonumber\\
			& + \nabla^2_{\state_{t-1}\ctrl_t}\dyn_t(\state_{t-1}, \ctrl_t)[\nabla \chain_{t-1}(\state_0, \ctrls)^\top, E_t^\top, \cdot ] \nonumber\\
			& + \nabla^2_{\ctrl_t\state_{t-1}}\dyn_t(\state_{t-1}, \ctrl_t)[E_t^\top, \nabla \chain_{t-1}(\state_0, \ctrls)^\top, \cdot ] \nonumber\\
			& + \nabla^2_{\ctrl_t\ctrl_t}\dyn_t(\state_{t-1}, \ctrl_t)[E_t^\top, E_t^\top, \cdot ]\nonumber
		\end{align}
		with $\nabla^2 \chain_0(\state_0, \ctrls) = 0$.
	\end{enumerate}
\end{lemma}
\begin{proof}
	It follows from the definition of the chain of computations and the notations used for tensors. Precisely we have that
	\[
	\chain_t(\state_0, \ctrls) = \dyn_t(\chain_{t-1}(\state_0, \ctrls), E_t^\top (\state_0;\ctrls)),
	\]
	hence the first result that can be written
	\[
	\nabla \chain_{t}(\state_0, \ctrls) = \nabla \chain_{t-1}(\state_0, \ctrls) \nabla_{\state_{t-1}} \dyn_t(\chain_{t-1}(\state_0, \ctrls), E_t^\top (\state_0;\ctrls)) + E_t\nabla_{\ctrl_t}  \dyn_t(\chain_{t-1}(\state_0, \ctrls), E_t^\top (\state_0;\ctrls)),
	\]
	hence the second result using the tensor notations.
\end{proof}

\linquad*
\begin{proof}
	To reformulate the optimization oracle problems as quadratic problems with linear dynamics we reformulate $\nabla \chainoutput(\currctrls)^\top \auxctrls $ as a linear chain of compositions and $\nabla^2 \chainoutput(\currctrls)[\auxctrls, \auxctrls, \nabla\obj(\chainoutput(\currctrls))]$ as a quadratic on the linear trajectory defined by the gradient and the parameters using Lemma~\ref{lem:derivatives_chain}.
	Precisely, for $\ctrls = (\ctrl_1,\;\ldots;\ctrl_\horizon)$ and  $\auxctrls = (\auxctrl_1;\ldots;\auxctrl_\horizon)$, denoting  $\chain_{\state_0}(\ctrls) = (\state_1; \ldots;\state_\horizon)$ and $\nabla \chain_{\state_0}(\ctrls)^\top \auxctrls= (\auxstate_1; \ldots; \auxstate_\horizon)$ with $\nabla \chain_{\state_0, t}(\ctrls)^\top \auxctrls = \auxstate_t$, we have from~\eqref{eq:grad_decomp}
	\begin{flalign} 
		 \auxstate_t &  = \nabla_{\state_{t-1}} \dyn_t(\state_{t-1}, \ctrl_t)^\top \auxstate_{t-1} + \nabla_{\ctrl_t} \dyn_t(\state_{t-1}, \ctrl_t)^\top \auxctrl_t, 
		\quad \mbox{for} \quad t\in \{1, \ldots, \horizon\} \label{eq:linear_chain_grad}\\
		\auxstate_0 & = 0,  \nonumber
	\end{flalign}
	and		$\nabla \chainoutput(\ctrls)^\top\auxctrls = \nabla \chain_{\state_0, \horizon}(\ctrls)^\top \auxctrls  = \auxstate_\horizon $.
	For the second order derivatives, from~\eqref{eq:hess_decomp}, we have  for $\auxctrls = (\auxctrl_1;\ldots;\auxctrl_\horizon)$ and $\lambda_t \in \reals^{\dimstate_t}$,
	\begin{align*}
		\frac{1}{2}\nabla^2 \chain_{\state_0, t}(\ctrls) [\auxctrls, \auxctrls, \lambda_t] = & 
		 \nabla^2 \chain_{\state_0, t-1}(\ctrls)[\auxctrls, \auxctrls, \nabla_{\state_{t-1}}\dyn_t(\state_{t-1}, \ctrl_t)\lambda_t] \\
		 & +
		  \frac{1}{2}\nabla^2_{\state_{t-1}\state_{t-1}}\dyn_t(\state_{t-1}, \ctrl_t)[\auxstate_{t-1}, \auxstate_{t-1}, \costate_t] \\
		&  + \nabla^2_{\state_{t-1}\ctrl_t}\dyn_t(\state_{t-1}, \ctrl_t)[\auxstate_{t-1},\auxctrl_t, \costate_t ] \nonumber\\
		&  + \frac{1}{2}\nabla^2_{\ctrl_t\ctrl_t}\dyn_t(\state_{t-1}, \ctrl_t)[\auxctrl_t, \auxctrl_t, \costate_t ] \nonumber
	\end{align*}
	Hence we get
	\begin{align}
		\frac{1}{2}\nabla^2 \chainoutput(\ctrls) [\auxctrls, \auxctrls, \nabla\obj(\chainoutput(\ctrls))] = \sum_{t=1}^\horizon &  \frac{1}{2}\nabla^2_{\state_{t-1}\state_{t-1}}\dyn_t(\state_{t-1}, \ctrl_t)[\auxstate_{t-1}, \auxstate_{t-1}, \costate_t] \label{eq:newton_decomp}\\
		&  + \nabla^2_{\state_{t-1}\ctrl_t}\dyn_t(\state_{t-1}, \ctrl_t)[\auxstate_{t-1},\auxctrl_t, \costate_t ] \nonumber\\
		&  + \frac{1}{2}\nabla^2_{\ctrl_t\ctrl_t}\dyn_t(\state_{t-1}, \ctrl_t)[\auxctrl_t, \auxctrl_t, \costate_t ] \nonumber
	\end{align}
	where $\auxstate_t$ are given in~\eqref{eq:linear_chain_grad} and $\costate_t$ are defined by
	\begin{align*}
		\costate_\horizon =\nabla\obj(\chainoutput(\ctrls)), \quad
		\costate_{t-1} = \nabla_{\state_{t-1}} \dyn_{t}(\state_{t-1}, \ctrl_t) \costate_{t} \qquad \mbox{for} \quad t\in\{1,\ldots, \horizon\}.
	\end{align*}
	The results follow by using the decomposability of $\reg$ and inserting~\eqref{eq:linear_chain_grad} and~\eqref{eq:newton_decomp}.
\end{proof}
We present the resolution of the Newton step by dynamic programming in Algo.~\ref{algo:Newton} whose implementation is justified in Proposition~\ref{prop:newton_step}. Note that the gradient is computed during the first backward pass which can reduce the computations by factorizing those computations. For the Gauss-Newton steps the same dynamic programming approach can be applied, however it is less computationally expansive to use automatic differentiation procedures as presented in Sec.~\ref{sec:oracles}. 
\begin{proposition}\label{prop:newton_step}
	Consider problem~\eqref{eq:lin_quad} and assume it is bounded below, then the cost-to-go functions defined for $t \in \{0,\ldots, \horizon\}$ and $\state_{t} \in \reals^{\dimlatent_t}$ as
	\begin{align}\label{eq:costogo}
	\costogo_t(\state_t) =  \min_{\substack{\auxctrl_{t+1},\ldots, \auxctrl_\horizon \\ \auxstate_t, \ldots, \auxstate_\horizon}} 
	\quad & \sum_{t'=t}^{\horizon} \frac{1}{2}\auxstate_{t'}^\top \PP_{t'} \auxstate_{t'}  + \p_{t'}^\top \auxstate_{t'}
	+\sum_{t'=t+1}^{\horizon}
	\auxstate_{t'-1}^\top \R_{t'} \auxctrl_{t'} + \frac{1}{2}\auxctrl_{t'}^\top \Q_{t'}\auxctrl_{t'} + \q_{t'}^\top\auxctrl_{t'} 
	+\frac{\kappa}{2}\|\auxctrl_{t'}\|_2^2 \\
	\mbox{subject to} \quad & \auxstate_{t'}= \A_{t'} \auxstate_{t'-1} + \B_{t'} \auxctrl_{t'} \qquad \mbox{for} \quad  t' \in \{t+1,\ldots,\horizon\}, \nonumber\\
	& \auxstate_t = \state_t, \nonumber
	\end{align}
	where $\PP_0 =0$, $\p_0 = 0$, are quadratics of the form 
	\begin{align}\label{eq:quad_costogo}
		\costogo_t(\state_t) = \frac{1}{2}\state_t^\top \CC_t \state_t + \cc_t^\top\state_t + \cste,
	\end{align}
	where $\CC_t, \cc_t$ are defined recursively in line~\ref{line:costogo} Algo.~\ref{algo:Newton} with $\CC_t = \CC_t^\top$ and $\cste$ is a constant.
	The solution of~\eqref{eq:lin_quad} is given by, starting from $\auxstate_0 = 0$, 
	\begin{align*}
		\auxctrl_t^*  = \K_t\auxstate_{t-1} + \kk_t\qquad
		\auxstate_t = \A_t \auxstate_{t-1} + \B_t \auxctrl_t^*,
	\end{align*}
	where $\K_t$ and $\kk_t$ are defined in line~\ref{line:gains} of Algo.~\ref{algo:Newton}.
\end{proposition}
\begin{proof}
	The cost-to-go functions satisfy the recursive equation~\eqref{eq:bellman} for $t \in\{1, \ldots, \horizon\}$
	\begin{align*}
	\costogo_{t-1}(\state_{t-1}) = \frac{1}{2}\state_{t-1}^\top \PP_{t-1} \state_{t-1} + \p_{t-1}^\top \state_{t-1} +
	\min_{\auxctrl_{t}\in \reals^{\dimparam_t}} \bigg\{ & \state_{t-1}^\top \R_{t} \auxctrl_{t} 
	+ \frac{1}{2}\auxctrl_{t}^\top \Q_{t} \auxctrl_{t} + \q_{t}^\top \auxctrl_{t} 
	+\frac{\kappa}{2} \|\auxctrl_{t}\|_2^2 \\
	& + \costogo_{t}(\A_{t}\state_{t-1} + \B_{t}\auxctrl_{t} )\bigg\},
	\end{align*}
	starting from $\costogo_\horizon(\state_\horizon) = \frac{1}{2}\state_\horizon^\top \PP_\horizon \state_\horizon + \p_\horizon^\top \state_\horizon$ so we get $\CC_\horizon = P_\horizon$ and $\cc_\horizon = p_\horizon$.
	Assume that the cost-to-go function $\costogo_{t}$ has the form~\eqref{eq:quad_costogo}  for $t\in \{1,\ldots, \horizon\}$ then the recursive equation~\eqref{eq:bellman} reads
	\begin{align*}
	\costogo_{t-1}(\state_{t-1}) & = \frac{1}{2}\state_{t-1}^\top (\PP_{t-1} + \A_{t}^\top \CC_t \A_t)\state_{t-1} + (\A_t^\top \cc_t+  \p_{t-1})^\top \state_{t-1} \\
	& \hspace{1em} +  \min_{\auxctrl_t \in \reals^{\dimparam_t}} \Big\{\auxctrl_t^\top(\R_t^\top \state_{t-1} + \q_t + \B_t^\top(\CC_t \A_t \state_{t-1} +  \cc_t)) 
	+ \frac{1}{2}\auxctrl_t^\top (\kappa\id  + \Q_t +  \B_t^\top \CC_t\B_t)\auxctrl_t\Big\}.
	\end{align*}
	If $\kappa \id + \Q_t + \B_t^\top C_t \B_t \not\succeq 0$, then the minimization problem is unbounded below and so is the original objective. If $\kappa \id + \Q_t + \B_t^\top C_t \B_t \succ 0$, then the minimization gives us $\costogo_{t-1}$ as a quadratic and the corresponding minimizer $\auxctrl_{t}^*(\state_{t-1})$ for a given $\state_{t-1}$, i.e.
	\begin{align*}
		\CC_{t-1} & = \PP_{t-1} +  \A_t^\top \CC_t \A_t - (\R_t + \A_t^\top \CC_t \B_t)(\kappa\id+ \Q_t  + \B_t^\top C_t \B_t)^{-1}( \R_t^\top + \B_t^\top\CC_t \A_t), \\
		\cc_{t-1} &  =  \A_t^\top \cc_t+  \p_{t-1} - (\R_t + \A_t^\top \CC_t \B_t)(\kappa\id+ \Q_t  + \B_t^\top C_t \B_t)^{-1}(\q_t + \B_t^\top \cc_t), \\
		\auxctrl_{t}^*(\state_{t-1}) &  = -(\kappa\id + \Q_t + \B_t^\top C_t \B_t)^{-1} ((\R_t^\top + \B_t^\top \CC_t \A_t )\state_{t-1} + \q_t + \B_t^\top \cc_t).
	\end{align*}
	The solution of~\eqref{eq:lin_quad} is given by computing $\costogo_0(0)$ which amounts to compute, starting from $\auxstate_0 = 0$, 
	\begin{align*}
	\auxctrl_t^* &  = \argmin_{\auxctrl\in \reals^{\dimparam_t}} \left\{\frac{1}{2}\auxctrl^\top \Q_t \auxctrl + \q_t^\top \auxctrl 
	+ \auxstate_{t-1}^\top \R_t \auxctrl
	+ \costogo_{t +1}(\A_t \auxstate_{t-1} + \B_t \auxctrl) \right\}= \auxctrl_{t}^*(\state_{t-1}),\\
	\auxstate_t & = \A_t \auxstate_{t-1} + \B_t \auxctrl_t^* .
	\end{align*}
\end{proof}

\begin{algorithm}
	\caption{Newton oracle by dynamic programming \label{algo:Newton}}
	\begin{algorithmic}[1]
		\State{\textbf{Inputs:} Chain of computations $\chain$ defined by $\dyn_t$, objective $\obj$, regularization $\reg$, regularization for the step $\kappa$, current weights $\currctrls = (\ctrl_1;\ldots;\ctrl_\horizon)$}
		\State{\textbf{Forward pass:}}
		\For{$t=1, \ldots, \horizon$}
		\State{Compute $\state_t = \dyn_t(\state_{t-1}, \ctrl_t)$}
		\State{Store $\A_t = \nabla_{\state_{t-1}} \dyn_t(\state_{t-1}, \ctrl_t)^\top$, $\B_t = \nabla_{\ctrl_t} \dyn_t(\state_{t-1}, \ctrl_t)^\top$ \newline 
		and  $\nabla_{\ctrl_t \ctrl_t}^2 \dyn_t(\state_{t-1}, \ctrl_t)$, $\nabla_{\ctrl_t \state_{t-1}}^2 \dyn_t(\state_{t-1}, \ctrl_t)$,  $\nabla_{\state_{t-1} \state_{t-1}}^2 \dyn_t(\state_{t-1}, \ctrl_t)$}
		\EndFor
		\State{\textbf{1st Backward pass:}}
		\State{Initialize $\costate_\horizon = \nabla \obj(\state_\horizon) $, $\PP_\horizon = \nabla^2 \obj(\state_\horizon)$, $\p_\horizon = \nabla \obj(\state_\horizon)$ }
		\For{$t=\horizon,\ldots, 1$}
			\State{Compute 
				\begin{align*}
					\PP_{t-1} &  = \nabla^2_{\state_{t-1} \state_{t-1}} \dyn_t(\state_{t-1}, \ctrl_t)[\cdot, \cdot, \costate_t]
					\qquad & \p_{t-1} & = 0 \\
					\Q_t & = \nabla^2_{\ctrl_t \ctrl_t} \dyn_t(\state_{t-1}, \ctrl_t)[\cdot, \cdot, \costate_t] + \nabla^2 \reg_t(\ctrl_t)
					\qquad & \q_t &  = \nabla \reg_t (\ctrl_t) \\
					\R_t &  = 
					\nabla^2_{\state_{t-1} \ctrl_t} \dyn_t(\state_{t-1}, \ctrl_t)[\cdot, \cdot, \costate_t] &
				\end{align*}
 			}
			\State{Compute $\costate_{t-1} = \nabla_{\state_{t-1}}\dyn_t(\state_{t-1}, \ctrl_t) \costate_t$}
		\EndFor
		\State{\textbf{2nd Backward pass:}}
		\State{Initialize $\CC_\horizon = \PP_\horizon$, $\cc_\horizon = \p_\horizon$, $\texttt{feasible}=\texttt{True}$}
		\For{$t = \horizon,\ldots, 1$}
			\If{$\kappa\id+ \Q_t  + \B_t^\top C_t \B_t \not \succ 0$}
			\State{$\texttt{feasible} = \texttt{False}$}
			\State{\textbf{break}}
			\EndIf
			\State{Compute \label{line:costogo}
			\begin{align*} 
				\CC_{t-1} & = \PP_{t-1} +  \A_t^\top \CC_t \A_t - (\R_t + \A_t^\top \CC_t \B_t)(\kappa\id+ \Q_t  + \B_t^\top C_t \B_t)^{-1}( \R_t^\top + \B_t^\top\CC_t \A_t) \\
				\cc_{t-1} &  =  \A_t^\top \cc_t+  \p_{t-1} - (\R_t + \A_t^\top \CC_t \B_t)(\kappa\id+ \Q_t  + \B_t^\top C_t \B_t)^{-1}(\q_t + \B_t^\top \cc_t)
			\end{align*}}\
			\State{Store  \label{line:gains}
				\begin{align*}
					\K_t  = -(\kappa\id + \Q_t + \B_t^\top C_t \B_t)^{-1} (\R_t^\top + \B_t^\top \CC_t \A_t ) \qquad
					\kk_t  =-(\kappa\id + \Q_t + \B_t^\top C_t \B_t)^{-1} (\q_t + \B_t^\top \cc_t)
				\end{align*}}
		\EndFor
		\If{$\texttt{feasible} = \texttt{False}$}
		\State{Re-do 2nd backward pass with  $\kappa := 2\cdot\kappa$}
		\EndIf
		\State{\textbf{Rollout:}}
		\State{Initialize $\auxstate_0 =0$}
		\For{$t=1,\ldots, \horizon$}
		\State{\begin{align*}
				\auxctrl_t^*  = \K_t \auxstate_{t-1} + \kk_t, \qquad 
				\auxstate_t  = \A_t \auxstate_{t-1} + \B_t \auxctrl_t
			\end{align*}}
		\EndFor
		\State{\textbf{Output:} $ (\auxctrl_1^*; \ldots; \auxctrl_\horizon^*)$}
	\end{algorithmic}
\end{algorithm}
Finally we present the derivation of a gradient step, i.e., gradient back-propagation, as a dynamic programming procedure, which gives the forward-backward algorithm presented in Sec.~\ref{sec:oracles} by taking $\reg=0$, $\kappa=-1$.
\begin{proposition}\label{prop:grad_step}
	Consider the gradient step~\eqref{eq:grad_step} as formulated in~\eqref{eq:lin_quad} with $\kappa=1/\stepsize$. The cost-to-go functions defined as in~\eqref{eq:costogo} are linear of the form 
	\begin{equation}\label{eq:costogo_grad}
	\costogo_t(\state_t) = \costate_t^\top \state_t + \cste,
	\end{equation}
	where 
	\begin{align*}
	\costate_\horizon & = \nabla \obj(\chainoutput(\ctrls)) \\
	\costate_{t-1} &  = \nabla_{\state_{t-1}}  \dyn_t(\state_{t-1}, \ctrl_t) \costate_t
	\qquad \mbox{for}\: t \in \{1, \ldots, \horizon\}
	\end{align*} and the solution of the step is given by
	\[
	\auxctrl_t^* = -\stepsize (\nabla \reg(\ctrl_t) + \nabla_{\ctrl_t}\dyn_t(\state_{t-1}, \ctrl_t) \costate_t).
	\]
\end{proposition}
\begin{proof}
	The cost-to-go function defined in~\eqref{eq:costogo} for a gradient step reads for $t=\horizon$, $\costogo_\horizon(\state_\horizon) = \p_\horizon^\top \state_\horizon$, so we get Eq.~\eqref{eq:costogo_grad} for $t=\horizon$ with $\costate_\horizon = \nabla \obj(\chainoutput(\ctrls))$. Assume that the cost-to-go function has the form~\eqref{eq:costogo_grad} for $t\in \{1,\ldots, \horizon\}$, then the recursive equation~\eqref{eq:bellman} reads
	\begin{align*}
		\costogo_{t-1}(\state_{t-1}) =  \min_{\auxctrl_t \in \reals^{\dimparam_t}} \left\{ \auxctrl_t^\top \q_t +  \costate_t^\top (\A_t \state_{t-1} + \B_t \auxctrl_{t}) + \frac{1}{2\stepsize} \|\auxctrl_t \|_2^2\right\}
	\end{align*}
	So we get that $\costogo_{t-1}$ is a linear function defined by $
	\costate_{t-1} = \A_t^\top \costate_t 
	$
	and that the optimal corresponding parameter is independent of $\state_{t-1}$ and reads 
	\[
	\auxctrl_t^* = -\stepsize (\q_t + \B_t^\top \costate_t).
	\] 
	Plugging the values of $\A_t, \B_t, \q_t$ into the solutions give the results.
\end{proof}

\subsection{Detailed complexities of forward and backward passes}
\label{def:sparsity}

\begin{definition}[Sparsity of the operations]
	We define the \emph{sparsity $\spars_\bilinear$ of a bilinear operation $\bilinear$} as the number of non-zero elements in its corresponding tensor.
	
	We define the \emph{sparsity $\spars_\activ$ of a function $\activ$} as the sparsity of its gradient, i.e., the maximal number of its non-zero elements for any inputs.
\end{definition}

The sparsity of a bilinear operation amounts to the number of multiplications needed to compute $\mathcal{B}[x,y,z]$,  $\mathcal{B}[\cdot, y, z]$, $\mathcal{B}[x, \cdot, z]$ or $ \mathcal{B}[x,y,\cdot]$, which gives us the sparsity of the two bilinear operations studied in this paper.
\begin{fact}
	For a matrix-product as in~\eqref{eq:fully_connected}, we have $\spars_{\bilinear} = \batchsize \tilde \diminput \diminput$. For a convolution as in~\eqref{eq:conv}, we have $\spars_{\bilinear} = \batchsize \nbpatch \nbfilter \dimfilter$.
\end{fact}
We considered $\Pi_k\Latent_{t-1}$ as the extraction of coordinates and not a matrix-vector product.
Note that the sparsity of the bilinear operation defines also the number of multiplications needed to compute gradient vector products like $\nabla_{\state_{t-1}} \bilinear_t(\state_{t-1}, \ctrl_t) \costate_{t+1}$ or $\nabla_{\ctrl_{t}} \bilinear_t(\state_{t-1}, \ctrl_t)\costate_{t+1}$  for $\costate_{t+1} \in \reals^{\dimlatent_{t}}$.

The sparsity of a function $f \in \reals^d \rightarrow \reals^n$ naturally gives the number of multiplications needed to compute gradient-vector products $\nabla f(x) \lambda$ for any $x\in \reals^d$, $\lambda \in \reals^n$. For element-wise activation functions as in~\eqref{eq:element_wise}, we have $\spars_{\activ} = \batchsize \diminput$, where we consider the input of the activation function to be $z = \Vect(Z)$ for $Z \in \reals^{\batchsize \times \diminput}$. Note that the sparsity of an activation function as defined here does not directly give the cost of computing it, neither its gradient. 

\paragraph{Forward-backward detailed complexity}
We present in the next proposition the cost of computing only the backward pass to compute the whole gradient.
The cost of computing the function and the gradients  of the layers in the forward pass can be further detailed using the sparsity of the bilinear operation and the cost of computing the activation function and its derivatives. The detailed complexities given in Sec.~\ref{sec:oracles} follow.
\begin{restatable}{proposition}{oraclecplx}
	Consider a chain $\chain$ of $\horizon$ layers as defined in Def.~\ref{def:chain} whose layers $\dyn_{t}$ are defined by $\nonlin_t, \biaffine_t$ as in~\eqref{eq:deep_layer}. 
	Then the cost of the backward pass defined in Algo.~\ref{algo:backward_lin} is of the order of
	\[
	\bigO\left(\sum_{t=1}^{\horizon}\spars_{\nonlin_t} + 2\spars_{\bilinear_{t}} + \spars_{\bilinear^\ctrl_t} + \spars_{\bilinear^\state_t}\right)
	\]
	elementary operations.
\end{restatable}
\begin{proof}
	If the chain of layers has the form~\eqref{eq:deep_layer}, the gradient vector products during the backward pass read
	\begin{align*}
	\nabla_{\state_{t-1}} \dyn_{t}(\state_{t-1}, \ctrl_t)\costate_{t+1} & = 
	\nabla_{\state_{t-1}} \biaffine_t(\state_{t-1}, \ctrl_t)
	\nabla \nonlin_t(\intervar_t)\costate_{t+1} = 
	(\mathcal{B}_t[ \cdot, \ctrl_t, \cdot] + \nabla \linearlatent_t(\state_{t-1})) \nabla \nonlin_t(\intervar_t) \costate_{t+1}, \\
	\nabla_{\ctrl_{t}} \dyn_{t}(\state_{t-1}, \ctrl_t)\costate_{t+1} & = 
	\nabla_{\ctrl_t} \biaffine_t(\state_{t-1}, \ctrl_t)\nabla \nonlin_t(\intervar_t) \costate_{t+1} =   (\mathcal{B}_t[\state_{t-1}, \cdot,  \cdot]  + \nabla \bilinear^\ctrl_t(\ctrl_t))\nabla \nonlin_t(\intervar_t) \costate_{t+1},
	\end{align*}
	where $\intervar_t =\biaffine_t(\state_{t-1}, \ctrl_t)$.
	The definitions of the sparsity of bilinear or general operations give the result by looking at each operation starting from the right.
\end{proof}

\subsection{Gauss-Newton by axutomatic differentiation}
Derivatives of the gradient vector product can then be computed themselves by back-propagation as recalled in the following lemma.
\begin{lemma}[{\citep[Lemma 3.4]{roulet2019iterativeArxiv}}]\label{lem:auto_diff_grad_prod}
	Consider  a differentiable chain of composition $\chain$ and an input $\state_0 \in \reals^{\dimstate_0}$ such that $\chainoutput = \chain_{\state_0, \horizon}: \reals^{\sum_{t=1}^\horizon \dimctrl_t} \rightarrow \reals^{\dimstate_{\horizon}}$. Given a variable $\ctrls \in \reals^{\sum_{t=1}^\horizon \dimctrl_t}$ and
	a decomposable differentiable function $g:\reals^{\sum_{t=1}^\horizon \dimctrl_t} \rightarrow\reals$ such that $g(\ctrls) = \sum_{t=1}^{\horizon} g_t(\ctrl_t)$ for $\ctrls = (\ctrl_1;\ldots;\ctrl_\horizon)$, computing the derivative of $\dualvar \rightarrow g(\nabla \chainoutput(\ctrls)\dualvar)$ requires two calls to an automatic-differentiation procedure.
\end{lemma}

\gaussnewtonautodiff*
\begin{proof}
	The first and second claims follow from standard duality computations applied to~\eqref{eq:gn_step}, they require convexity of $\obj$ and $\reg$. The third claim comes from the fact that~\eqref{eq:dual_gn_step} is a quadratic convex problem that can be solved in at most $\dimstate_{\horizon}$ iterations of a conjugate gradient descent. Each iteration requires to compute the gradient of $\dualvar\rightarrow  (\qua_\reg^\ctrls + \kappa \|\cdot\|_2^ 2/2)^\star(- \nabla \chainoutput(\ctrls) \dualvar )$ which requires two calls to an automatic differentiation procedure by Lemma~\ref{lem:auto_diff_grad_prod} and using that $\reg^*$ is also decomposable. A last call to an automatic differentiation procedure is needed to compute $\nabla \chainoutput(\ctrls) \dualvar^*$. 
	The costs of computing $\nabla (\qua_\obj^{\chainoutput(\ctrls)})^ \star(\dualvar)$ for $\dualvar\in \reals^{\dimstate_{\horizon}}$ and $\nabla  (\qua_\reg^\ctrls + \kappa \|\cdot\|_2^ 2/2)^\star(\ctrls)$ for $\ctrls \in \reals^{\sum_{t=1}^\horizon \dimctrl_t}$ are ignored since they do not involve a chain of compositions and are assumed to be easily accessible.
\end{proof}

\section{Smoothness computations}\label{sec:smooth_proofs}
\subsection{Elementary operations}
\paragraph{Univariate functions}
\begin{lemma}\label{lem:smooth_elt_wise}
	Let $\alpha_i \in \mathcal{C}_{\lip_i, \smooth_i}$ for $i=1, \ldots, n$. Denote $\lip = (\lip_i)_{i=1}^n$, $\smooth = (\smooth_i)_{i=1}^ n$.
	\begin{enumerate}
		\item Assume $\alpha_i: \reals^{\dimin_i} \rightarrow \reals^{\dimout_i}$,  then 
		\[
		a: \begin{cases}
			\reals^{\sum_{i=1}^ n \dimin_i} & \rightarrow \reals^{\sum_{i=1}^{n} \dimout_i}\\
			x=(x_1;\ldots;x_n) &  \rightarrow  (\alpha_1(x_1); \ldots; \alpha_n(x_n))
		\end{cases} \qquad 	
		\]
		is $\|\lip\|_2$-Lipschitz continuous and $\|\smooth\|_{\infty}$-smooth.
\item Assume $\alpha_i : \reals^{\dimin_i} \rightarrow \reals^{\dimout}$, then 
\[
	a:
\begin{cases}
	\reals^{\sum_{i=1}^ n \dimin_i} & \rightarrow \reals^{\dimout}\\
	x=(x_1;\ldots;x_n) &  \rightarrow  \sum_{i=1}^n \alpha_i(x_i)
\end{cases}
\]
		is $\|\lip\|_2$-Lipschitz continuous and $\|\smooth\|_{\infty}$-smooth.
		\item Assume $a_i: \reals^{\dimin} \rightarrow \reals^{\dimout_i}$, then 
		\[
		a: \begin{cases}
			\reals^{\dimin} & \rightarrow \reals^{\sum_{i=1}^{n} \dimout_i}\\
			x &  \rightarrow  (\alpha_1(x); \ldots; \alpha_n(x))
		\end{cases} \qquad
		\]
		is  $\|\lip\|_2$-Lipschitz continuous and $\|\smooth\|_2$-smooth. 
		\item Assume  $\alpha_i : \reals^{\dimin} \rightarrow \reals^{\dimout}$, then
		\[
			a:
		\begin{cases}
			\reals^{\dimin} & \rightarrow \reals^{\dimout}\\
			x&  \rightarrow  \sum_{i=1}^n \alpha_i(x)
		\end{cases}
		\]
		is  $\|\lip\|_1$-Lipschitz continuous and $\|\smooth\|_1$-smooth.
	\end{enumerate}
\end{lemma}
\begin{proof}
	
	\begin{enumerate}
		\item
	We have for $x = (x_1;\ldots;x_n) \in \reals^{\sum_{i=1}^ n \dimin_i}$ and
	$z =(z_1; \ldots; z_n) \in \reals^{\sum_{i=1}^{n} \dimout_i}$, 
	\begin{align*}
		\|\nabla a(x) z \|_2 = \|\sum_{i=1}^{n} \nabla \alpha_i(x_i)z_i\|_2 \leq \sum_{i=1}^{n} \|z_i\|_2 \|\nabla \alpha_i(x_i)\|_{2, 2} \leq \|z\|_2 \sqrt{\sum_{i=1}^n \lip_i^2},
	\end{align*}
	which gives an upper bound on the Lipschitz-continuity of $a$.
	For  $x= (x_1;\ldots;x_n), y = (y_1;\ldots;y_n) \in \reals^{\sum_{i=1}^ n \dimin_i}$, we have
	\begin{align*}
		\|(\nabla a(x) {-} \nabla a(y))z\|_2 \leq \sum_{i=1}^{n} \|z_i\|_2 \|\nabla \alpha_i(x_i) {-} \nabla \alpha_i(y_i)\|_{2, 2} \leq \sum_{i=1}^{n} \|z_i\|_2\|x_i{-} y_i\|_2 \smooth_i \leq  \|z\|_2 \|x{-}y\|_2\max_{i\in\{1, \ldots, n\}} \smooth_i .
	\end{align*}
	Hence $\|\nabla a(x) - \nabla a(y)\|_{2, 2} \leq  \|x-y\|_2\max_{i\in\{1, \ldots, n\}} \smooth_i $ which gives an upper bound on the smoothness of $a$.
		
\item 	We have for $x = (x_1;\ldots;x_n) \in \reals^{\sum_{i=1}^ n \dimin_i}$ and $z \in \reals^\dimout$,
	\[
	\|\nabla a(x)z\|_2^ 2 = \sum_{i=1}^{n}\| \nabla \alpha_i(x_i) z\|_2^ 2 \leq \sum_{i=1}^ n \lip_i^ 2 \|z\|_2^ 2,
	\]
	which gives the Lipschitz-continuity parameter. Similarly we have for  $x= (x_1;\ldots;x_n), y = (y_1;\ldots;y_n) \in \reals^{\sum_{i=1}^ n \dimin_i}$,
	\[
	\|(\nabla a(x)  -\nabla a(y)) z\|_2^ 2 =  \sum_{i=1}^{n}\| (\nabla \alpha_i(x_i)- \nabla \alpha_i(y_i)) z\|_2^ 2 \leq \sum_{i=1}^{n} \smooth_i^ 2\|x_i-y_i\|_2^ 2 \|z\|_2^ 2\leq  \max_{i\in \{1, \ldots, n\}} \smooth_i^ 2\|x-y\|_2^ 2 \|z\|_2^ 2,
	\]
	which gives  the smoothness constant of $a$.

\item 
	The bound on the Lipschitz-continuity parameter follows from the same argument as in 1. For the smoothness parameter, we have for
	$x, y \in \reals^ \dimin$ and $z =(z_1; \ldots; z_n) \in \reals^{\sum_{i=1}^{n} \dimout_i}$, 
	\[
	\|(\nabla a(x) - \nabla \nonlin(y))z\|_2 \leq \sum_{i=1}^{n} \|z_i\|_2 \|\nabla \alpha_i(x) - \nabla \alpha_i(y)\|_{2, 2} \leq \sum_{i=1}^{n} \|z_i\|_2\|x- y\|_2 \smooth_i \leq  \|z\|_2 \|x-y\|_2 \sqrt{\sum_{i=1}^n \smooth_i^2}.
	\]
	Hence the result as in 1. 
	\item Clear by linearity of the gradient and triangular inequality.
\end{enumerate}
\end{proof}

\paragraph{Bilinear functions}
\begin{lemma}\label{lem:bilin_smooth_elt_wise}
	Consider $s\times t$ bilinear functions $\bilinear_{i + (j-1)s}:\reals^{d_i} \times \reals^{p_j} \rightarrow \reals^{m_{i + (j-1)s}}$ for $i \in \{1, \ldots s\}, j\in \{1, \ldots t\}$ then
	\[
	\bilinear : \begin{cases}
		\reals^{\sum_{i=1}^s d_i}  \times \reals^{\sum_{j=1}^t p_j} & \rightarrow \reals^{\sum_{k=1}^{st} m_k} \\
		(\state, \ctrl) &\rightarrow (\bilinear_1(\state_1, \ctrl_1); \ldots ;\bilinear_s(\state_s, \ctrl_1); 
		\bilinear_{s+1}(\state_1, \ctrl_2); \ldots ;\bilinear_{st}(\state_s, \ctrl_t) )
	\end{cases}
	\]
	is $\smooth_\bilinear = \max_{k \in \{1, \ldots, st\}} \smooth_{\bilinear_{k}}$ smooth.
\end{lemma}
\begin{proof}
	By Lemma~\ref{lem:tensnorm},  we have that  $\smooth_\bilinear = \sup_{\state, \ctrl} \|\bilinear(\state, \ctrl)\|_2/\|\state\|_2\|\ctrl\|_2$. Now 
	\begin{align*}
			\|\bilinear(\state, \ctrl)\|_2^ 2 & = \sum_{i=1}^{s} \sum_{j=1}^ t \|\bilinear_{i + s(j-1)}(\state_i, \ctrl_j)\|_2^ 2\\
			& \leq  \sum_{i=1}^{s} \sum_{j=1}^ t  \smooth_{\bilinear_{i + s(j-1)}} \|\state_i\|_2^ 2 \|\ctrl_j\|_2^2 \leq \max_{k \in \{1, \ldots, st\}} \smooth_{\bilinear_{k}} \|\state\|_2^ 2 \|\ctrl\|_2^ 2.
	\end{align*}

\end{proof}

\subsection{Compositions}
The smoothness properties of the functions can be derived by bounding appropriately their first and second order information. Even if e.g. the functions are not twice differentiable, the same results would apply by decomposing carefully the terms, we directly use the second order information as it directly gives what we are interested in.

For a smooth function, an upper bound on the Lipschitz continuity of the function on a bounded set can be estimated even if the function is not Lipschitz continuous. Similarly a bound on the function a bounded set can be refined as defined below.

\begin{fact}\label{lem:refined_smooth}
	For a function $f \in \classfunc_{\bound_f, \lip_f, \smooth_f}$ and $R>0$. Denoting $B_R = \{x \in \dom f: \|x\|_2\leq R\}$, we have that
	\begin{align*}
		\lip_f^{B_R} &\leq \lip_f(R) := \min\{\lip_f, \|\nabla f(0)\|_{2, 2} + R\smooth_f\}, \\
		\bound_f^{B_R} &\leq \bound_f(R) := \min\{\bound_f, \|f(0)\|_2 + R\lip_f(R)\}
	\end{align*}
\end{fact}

For a sequence of compositions we have the following result. 
\begin{lemma}\label{lem:nonlin_comp_bounds}
	Consider
	\[
	\nonlin =  \nonlin_\nbcomp \circ \ldots \circ \nonlin_1
	\]
	with $\nonlin_j \in \classfunc_{\bound_{\nonlin_j}, \lip_{\nonlin_j}, \smooth_{\nonlin_j}}$ for $j \in \{1, \ldots, \nbcomp\}$ and $\nonlin: \reals^{d}\rightarrow \reals^{n}$. Denote $B_R = \{x\in \reals^d: \|x\|_{\normidx}\leq R\}$, and for $j \in \{1,\ldots, \nbcomp\}$, 
	\begin{align*}
		\bound_j & = \bound_{a_j}(\bound_{t-1}), \\
		\lip_j & = \lip_{j-1} \lip_{\nonlin_j}(\bound_{j-1}),\\
		\smooth_j & = \smooth_{\nonlin_j}\lip_{j-1}^ 2 +  \smooth_{j-1}\lip_{\nonlin_j}(\bound_{j-1}),
	\end{align*}
	with $\bound_0 = R, \lip_0= 1, \smooth_0 = 0$.
	We have
	\begin{flalign*}
		\bound_\nonlin^{B_R} \leq \bound_\horizon,   \qquad
		\lip_{\nonlin}^{B_R} \leq \lip_\horizon = \prod_{j=1}^{\nbcomp} \lip_{\nonlin_j}(\bound_{j-1}), \qquad 
		\smooth_{\nonlin}^{B_R} \leq  \smooth_\horizon= \sum_{j=1}^{\nbcomp}\smooth_{\nonlin_j}
		\left(\prod_{i=1}^{j-1} \lip_{\nonlin_i}(\bound_{i-1}) \right)^2
		\left(\prod_{i=j+1}^{\nbcomp}\lip_{\nonlin_i}(\bound_{i-1})\right). 
	\end{flalign*}
\end{lemma}
\begin{proof}
	The bound on the output is a direct iterative application of Fact~\ref{lem:refined_smooth}.
	We have for $x\in \reals^d$, 
	\begin{align*}
		\nabla \nonlin(x)  = \prod_{j=1}^{\nbcomp} g_j(x), \qquad \mbox{where} \quad 
		g_j(x)  = \nabla \nonlin_j(\nonlin_{j-1}\circ \ldots\circ \nonlin_{1} (x))  \quad \mbox{for} \ j \in \{1, \ldots, \nbcomp\}.
	\end{align*}
	We have 
	\[
	\sup_{x\in \reals^d: \|x\|_{\normidx}\leq R} \|g_j(x)\|_{\opnormidx} \leq \min\{\lip_{\nonlin_j}, \|\nabla \nonlin_j(0)\|_{\opnormidx} + \smooth_{\nonlin_j}\bound_{\nonlin_{j-1} \circ \ldots \circ \nonlin_1}^{B_R} \}.
	\]
	Therefore
	\[
	\lip_{\nonlin}^{B_R} \leq \prod_{j=1}^{\nbcomp}\lip_{\nonlin_j}(\bound_{j-1}).
	\]
	We have for $x\in \reals^d$, 
	\begin{align*}
		\nabla^2 \nonlin(x) = \sum_{j=1}^{\nbcomp} \nabla^2 \nonlin_j(x)\left[
		\left(\prod_{i=1}^{j-1}g_{i}(x)\right)^\top,
		\left(\prod_{i=1}^{j-1}g_{i}(x)\right)^\top, \prod_{i=j+1}^{\nbcomp}g_{i}(x)\right].
	\end{align*}
	Therefore
	\[
	\smooth_\nonlin^{B_R} \leq  \sum_{j=1}^{\nbcomp}\smooth_{\nonlin_j}
	\left(\prod_{i=1}^{j-1} \lip_{\nonlin_i}(\bound_{i-1}) \right)^2
	\left(\prod_{i=j+1}^{\nbcomp}\lip_{\nonlin_i}(\bound_{i-1})\right). 
	\]
\end{proof}

Lemma~\ref{lem:nonlin_comp_bounds} can be used to estimate the smoothness of a chain of computations with respect to its input for fixed parameters.  
\begin{corollary}\label{cor:smooth_state}
	Consider a chain $\chain$ of $\horizon$ computations $\dyn_t \in \classfunc_{\bound_{\dyn_t}, \lip_{\dyn_t}, \smooth_ {\dyn_t}}$ with given parameters $\ctrls = (\ctrl_1;\ldots;\ctrl_\horizon)$. Denote $\dyn_{t ,\ctrl_t} = \dyn_t(\cdot, \ctrl_t)$. 
	Denote $B_R = \{x\in \reals^d: \|x\|_{\normidx}\leq R\}$, and for $j \in \{1,\ldots, \nbcomp\}$, 
	\begin{align*}
		\bound_j & =\bound_{\dyn_t(\cdot, \ctrl_t)}(\bound_{t-1}), \\
		\lip_j & = \lip_{j-1} \lip_{\dyn_j(\cdot, \ctrl_j)}(\bound_{j-1}),\\
		\smooth_j & = \smooth_{\dyn_j(\cdot, \ctrl_j)}\lip_{j-1}^ 2 +  \smooth_{j-1}\lip_{\dyn_j(\cdot, \ctrl_j)}(\bound_{j-1}),
	\end{align*}
	with $\bound_0 = R, \lip_0= 1, \smooth_0 = 0$.
	We have
	\begin{gather*}
		\bound_{\chain_{\horizon, \ctrls}}^{B_R} \leq \bound_\horizon,   \quad
		\lip_{\chain_{\horizon, \ctrls}}^{B_R} \leq \lip_\horizon = \prod_{j=1}^{\nbcomp}\lip_{\dyn_j(\cdot, \ctrl_j)}(\bound_{j-1}), \\ 
		\smooth_{\chain_{\horizon, \ctrls}}^{B_R} \leq \smooth_\horizon=  \sum_{j=1}^{\nbcomp}\smooth_{\dyn_j(\cdot, \ctrl_j)}
		\left(\prod_{i=1}^{j-1}\lip_{\dyn_i(\cdot, \ctrl_i)}(\bound_{i-1}) \right)^2
		\left(\prod_{i=j+1}^{\nbcomp}\lip_{\dyn_i(\cdot, \ctrl_i)}(\bound_{i-1})\right). 
	\end{gather*}
\end{corollary}

\subsection{Chains of computations}

We have  the following result for smooth and Lipschitz continuous chains of computations.
\smoothgen*
\begin{proof}
	The first claim follows directly from Lemma~\ref{lem:derivatives_chain}. For the second claim we have that~\eqref{eq:hess_decomp} gives 
	\[
	\smooth_{\chain_t} \leq \smooth_{\chain_{t-1} } \lip_{\dyn_t} + \smooth_{\dyn_t} \lip_{\chain_{t-1}}^ 2 + 2 \smooth_{\dyn_t} \lip_{\chain_{t-1}} + \smooth_{\dyn_t},
	\]
	which simplifies to give the result.
\end{proof}
For a bivariate function $\dyn(\state, \ctrl): \reals^\dimstate \times \reals^\dimctrl \rightarrow \reals^\diminter$, we define 
\begin{align*}
	\lip_\dyn^\ctrl = \sup_{\ctrl \in \reals^ \dimctrl, \state\in \reals^ \dimstate} \lip_{\dyn(\state, \ctrl +\cdot)} ,
	\qquad
	\lip_\dyn^\state  = \sup_{\ctrl \in \reals^ \dimctrl, \state\in \reals^ \dimstate} \lip_{\dyn(\state+\cdot, \ctrl)}. 
\end{align*}
Moreover if the function is continuously differentiable, we define
\begin{align*}
	\smooth_\dyn^{\ctrl\ctrl} &  = \sup_{\ctrl \in \reals^ \dimctrl, \state\in \reals^ \dimstate} \lip_{\nabla_\ctrl \dyn(\state, \ctrl +\cdot)}, \qquad
	\smooth_\dyn^{\state\ctrl}   = \sup_{\ctrl \in \reals^ \dimctrl, \state\in \reals^ \dimstate} \lip_{\nabla_\ctrl \dyn( \state+\cdot, \ctrl)}, \\
	\smooth_\dyn^{\ctrl\state} &  = \sup_{\ctrl \in \reals^ \dimctrl, \state\in \reals^ \dimstate} \lip_{\nabla_\state \dyn(\state, \ctrl +\cdot)}, \qquad
	\smooth_\dyn^{\state\state}  = \sup_{\ctrl \in \reals^ \dimctrl, \state\in \reals^ \dimstate} \lip_{\nabla_\state \dyn(\state+\cdot, \ctrl)}.
\end{align*}

For a bivariate continuosuly differentiable  function $\dyn(\state, \ctrl): \reals^\dimctrl \times \reals^\dimstate \rightarrow \reals^\diminter$, we have that
\begin{align*}
	\lip_\dyn^\ctrl = \sup_{\ctrl \in \reals^ \dimctrl, \state\in \reals^ \dimstate} \|\nabla_\ctrl \dyn(\state, \ctrl)\|_{2, 2},
	\qquad 
	\lip_\dyn^\state = \sup_{\ctrl \in \reals^ \dimctrl, \state\in \reals^ \dimstate} \|\nabla_\state \dyn(\state, \ctrl)\|_{2, 2}.
\end{align*}
If the function $\dyn$ is twice continuously differentiable, we have that
\begin{align*}
	\smooth_\dyn^{\ctrl\ctrl} & = \sup_{\ctrl \in \reals^ \dimctrl, \state\in \reals^ \dimstate} \|\nabla^2_{\ctrl\ctrl} \dyn(\state, \ctrl)\|_{2, 2, 2}, \qquad
	\smooth_\dyn^{\state\ctrl} = \sup_{\ctrl \in \reals^ \dimctrl, \state\in \reals^ \dimstate} \|\nabla^2_{\state\ctrl} \dyn(\state, \ctrl)\|_{2, 2, 2},\\
	\smooth_\dyn^{\ctrl\state} &  = \sup_{\ctrl \in \reals^ \dimctrl, \state\in \reals^ \dimstate} \|\nabla^2_{\ctrl\state} \dyn(\state, \ctrl)\|_{2, 2, 2},\qquad
	\smooth_\dyn^{\state\state} = \sup_{\ctrl \in \reals^ \dimctrl, \state\in \reals^ \dimstate} \|\nabla^2_{\ctrl\ctrl} \dyn(\state, \ctrl)\|_{2, 2, 2}.
\end{align*}
Finally for $ R_\state \geq 0, R_\ctrl \geq0$, we have 
\begin{align} \label{eq:smooth_refined_dyn}
	\sup_{(\state, \ctrl )\in B_{R_\state} \times B_{R_\ctrl}} \|\nabla_\ctrl \dyn(\state, \ctrl)\|_{2, 2} & \leq \lip_{\dyn}^\ctrl(R_\state, R_\ctrl) := \min\{\lip_{\dyn}^ \ctrl, 
	\|\nabla_\ctrl \dyn (0, 0)\|_{\opnormidx} 	 
	+ \smooth_{\dyn}^ {\ctrl\ctrl}R_\ctrl 
	+ \smooth_{\dyn}^ {\state\ctrl} R_\state\} \\
	\sup_{(\state, \ctrl )\in B_{R_\state} \times B_{R_\ctrl}} \|\nabla_\state \dyn(\state, \ctrl)\|_{2, 2} &\leq  \lip_{\dyn}^\state(R_\state, R_\ctrl)  := \min\{\lip_{\dyn}^ \state, 
	\|\nabla_\state \dyn (0, 0)\|_{\opnormidx} 	 
	+ \smooth_{\dyn}^ {\state\state}R_\ctrl 
	+ \smooth_{\dyn}^ {\ctrl\state} R_\state\}. \nonumber
\end{align}

We then have the following.
\begin{restatable}{lemma}{smoothgenb}\label{lem:smooth_gen}
	Let $\chain$ be a chain of $\horizon$ computations  $\dyn_t \in \mathcal{C}_{\bound_{\dyn_t}, \lip_{\dyn_t}^\ctrl, \lip_{\dyn_t}^\state,  \smooth_{\dyn_t}^{\ctrl\ctrl}, \smooth_{\dyn_t}^{\state\ctrl},\smooth_{\dyn_t}^{\state\state}}$, initialized at some $\state_0$ such that $\|\state_0\|_2 \leq R_0$. Let $\set = \bigotimes_{t=1}^ \horizon B_{\setparam_t}(\reals^{\dimctrl_t})  = \{\ctrls =(\ctrl_1;\ldots;\ctrl_\horizon)\in \reals^ {\sum_{t=1}^\horizon \dimctrl_t}:  \ctrl_t \in \reals^{\dimctrl_t}, \|\ctrl_t\|_2 \leq \setparam_t\} $. Define for $t \in \{1, \ldots, \horizon\}$, 
	\begin{align*}
		\bound_t & = \min\{\bound_{\dyn_t},  \|\dyn_t(0, 0)\|_2 + \lip_{\dyn_t}^\ctrl(\bound_{t-1}, R_t) R_t + \lip_{\dyn_t}^\state(\bound_{t-1}, R_t) \bound_{t-1}\}, \\
		\lip_t & =  \lip_{\dyn_t}^\ctrl(\bound_{t-1}, R_t) +  \lip_{t-1} \lip_{\dyn_t}^\state(\bound_{t-1}, R_t), \\
		\smooth_t & = \smooth_{t-1} \lip_{\dyn_t}^\state(\bound_{t-1}, R_t)  + \smooth_{\dyn_t}^{\state\state}\lip_{t-1}^2+  (\smooth_{\dyn_t}^{\state\ctrl} + \smooth_{\dyn_t}^{\ctrl\state} )\lip_{t-1}+ \smooth_{\dyn_t}^{\ctrl\ctrl}.
	\end{align*}
	with $\bound_0 = R_0$, $\lip_0 = 0$, $\smooth_0 = 0$.
	We have that 
	\[
	\bound_{\chain_\horizon, \state_0}^\set \leq \bound_\horizon, \quad \lip_{\chain_\horizon, \state_0}^\set \leq \lip_\horizon, \quad  \smooth_{\chain_\horizon, \state_0}^\set \leq \smooth_\horizon.
	\]
\end{restatable}
\begin{proof}
	The result directly follows from  Lemma~\ref{lem:derivatives_chain}, with the Lipschitz-continuity constants derived in~\eqref{eq:smooth_refined_dyn}. 
\end{proof}

\begin{proof}
	The proof relies on Lemma~\ref{lem:smooth_gen}, where the smoothness of the  inner compositions are computed according to Lemma~\ref{lem:nonlin_comp_bounds}. Namely, we have
	\begin{align*}
		\lip_{\dyn_t}^\ctrl(R_\state, R_\ctrl)   \leq  \lip_{\nonlin_t}(\bound_{\biaffine_t}(R_\state, R_\ctrl))\lip_{\biaffine_t}^\ctrl(\R_\state, R_\ctrl), \qquad
		\lip_{\dyn_t}^\state(R_\state, R_\ctrl)  \leq  \lip_{\nonlin_t}(\bound_{\biaffine_t}(R_\state, R_\ctrl))\lip_{\biaffine_t}^\state(\R_\state, R_\ctrl),
	\end{align*}
	with 
	\begin{gather*}
		\lip_{\biaffine_t}^\ctrl(\R_\state, R_\ctrl)   = \smooth_{\biaffine_t} R_\state + \lipp^\ctrl_{\biaffine_t}, \qquad
		\lip_{\biaffine_t}^\state(\R_\state, R_\ctrl)  = \smooth_{\biaffine_t} R_\ctrl + \lipp^\state_{\biaffine_t}, \\
		\bound_{\biaffine_t}(R_\state, R_\ctrl)  = \lip_{\biaffine_t}^\ctrl(\R_\state, R_\ctrl) R_\ctrl + \lip_{\biaffine_t}^\state(\R_\state, R_\ctrl) R_\state+ \|\biaffine_t(0,0)\|_2,
	\end{gather*}
	and $\lip_{\nonlin_t}$ can be computed as in Lemma.~\ref{lem:nonlin_comp_bounds}. 
	On the other hand, denoting $\smooth^{\state\state}_{\dyn_t}(R_\state, R_\ctrl) = \sup_{(\state,\ctrl) \in B_{R_\state} \times B_{R_\ctrl}} \|\nabla^2_{\state\state} \dyn_t(\state, \ctrl)\|_{2, 2, 2}$ (and similarly for $\smooth^{\ctrl\ctrl}_{\dyn_t}, \smooth^{\ctrl\state}_{\dyn_t}, \smooth^{\state\ctrl}_{\dyn_t}$), we have
	\begin{align*}
		\smooth^{\state\state}_{\dyn_t}(R_\state, R_\ctrl) &\leq  \smooth_{\nonlin_t}(	\bound_{\biaffine_t}(R_\state, R_\ctrl))\lip_{\biaffine_t}^\state(\R_\state, R_\ctrl)^ 2 \\
		\smooth^{\ctrl\ctrl}_{\dyn_t}(R_\state, R_\ctrl) &\leq  \smooth_{\nonlin_t}( 	\bound_{\biaffine_t}(R_\state, R_\ctrl))\lip_{\biaffine_t}^\ctrl(\R_\state, R_\ctrl)^ 2 \\
		\smooth^{\state\ctrl}(R_\state, R_\ctrl) = \smooth^{\ctrl\state}(R_\state, R_\ctrl) & = \smooth_{\biaffine_t} \lip_{\nonlin_t}(\bound_{\biaffine_t}(R_\state, R_\ctrl))  + \smooth_{\nonlin_t}(	\bound_{\biaffine_t}(R_\state, R_\ctrl) )	\lip_{\biaffine_t}^\ctrl(\R_\state, R_\ctrl)\lip_{\biaffine_t}^\state(\R_\state, R_\ctrl),
	\end{align*}
	where $\smooth_{\nonlin_t}(	\bound_{\biaffine_t}(R_\state, R_\ctrl) )$ is computed by Lemma~\ref{lem:nonlin_comp_bounds}.
\end{proof}

\section{Smoothness of objectives and layers}\label{sec:smooth_list}
\subsection{Supervised objectives}\label{ssec:supervised_obj}
For supervised objectives $\obj:\reals^{\nbsamp\dimstate_{\horizon}} \rightarrow \reals$ that reads for $ \labpred =(\labpred_1; \ldots;\labpred_\nbsamp)$ with $\labpred_i \in \reals^{\dimstate_\horizon}$, 
\[
\obj(\labpred) = \frac{1}{n}\sum_{i=1}^{\nbsamp}\obj_i(\labpred_i),
\]
we only need to compute the smoothness of $\obj_i(\labpred_i)$ (see Lemma~\ref{lem:smooth_elt_wise}) which is usually defined by a loss $\obj_i(\labpred_i) = \loss(\labpred_i, \lab_i)$.
We are interested in this section in the smoothness $\smooth_\obj(\set)$ and Lipschitz-continuity $\lip_\obj(\set)$ of the objective $\obj$ on a set $\set$. We omit the dependency on the set $\set$ if Lipschitz-continuity or smoothness properties of the functions are defined on its whole domain. 

\paragraph{Square loss}
Assume that the labels belong to a compact set $\mathcal{Y}$. The square loss is defined by $\obj(\labpred) = \loss_{\operatorname{sq}}(\labpred, \lab) = (\labpred-\lab)^2/2$. We have then
\begin{align*}
	\ell_{\operatorname{sq}}(\set)  = \radius_\set + \radius_{\mathcal{Y}}, \qquad 
	\smooth_{\operatorname{sq}}  = 1.	
\end{align*}
where $\radius_\set = \max_{x\in C} \|x\|_2$ and $\radius_{\mathcal{Y}} = \max_{\lab\in \mathcal{Y}} \|\lab\|_2$.

\paragraph{Logistic loss}
Consider $\lab \in \{0,1\}^{\dimlabel}$, the logistic loss is defined as
$
\obj(\labpred) = \loss_{\log}(\labpred, \lab)  = -\lab^\top \labpred + \log\left(\sum_{j=1}^{\dimlabel} \exp(\labpred_j)\right).
$
We have then, denoting $\exp(y)= (\exp(y_i))_{i=1,\ldots \dimlabel}$,
\begin{align*}
	\nabla \obj(\labpred) = -\lab +\frac{\exp(\labpred)}{\exp(\labpred)^\top \ones_\dimlabel}, \qquad 
	\nabla^2 \obj(\labpred)= \frac{\diag(\exp(\labpred))}{\exp(\labpred)^\top \ones_\dimlabel} - \frac{\exp(\labpred)\exp(\labpred)^\top}{(\exp(\labpred)^\top \ones_\dimlabel)^2}.
\end{align*}
Therefore using that $\lab \in \{0,1\}^\dimlabel$ and that $\|\exp(\labpred)\|_2 \leq \|\exp(\labpred)\|_1$,
\begin{align*}
	\lip_{\log} \leq 2, \qquad
	\smooth_{\log} \leq 2.
\end{align*}

\subsection{Unsupervised objectives}\label{ssec:unsupervised_obj}
For the k-means and spectral clustering objectives, we consider the outputs of the chains of the computations to form a matrix $F(\fixedstate, \ctrls) = (\chain(\fixedstate^{(1)}, \ctrls), \ldots, \chain(\fixedstate^{(n)}, \ctrls)) \in \reals^{\dimlabel \times \nbsamp}$ where $\dimlabel = \dimstate_\horizon$ and $\nbsamp$ to be the number of samples. The objectives are then $\obj : \reals^{\dimlabel \times \nbsamp} \rightarrow \reals$ and we denote by $Z \in \reals^{\dimlabel \times n}$ their variables.  The overall objective is $\obj(F(\fixedstate, \ctrls))$ for $\fixedstate = (\fixedstate^{(1)}; \ldots;\fixedstate^{(n)})$. 
We denote $\numclass$ the number of classes that the unsupervised objective aims to cluster and
\begin{align*}
\mathcal{Y} & = \{\Lab =(\lab_1,\ldots, \lab_n)^\top\in \{0,1\}^{\nbsamp\times\numclass} \quad \mbox{s.t.} \quad \Lab \ones_\numclass =\ones_\nbsamp\}.
\end{align*}
\paragraph{K-means clustering}
The K-means clustering objective reads
\begin{align*}
\obj(Z) & = \min_{\substack{\Lab \in \mathcal{Y} \\ C\in \reals^{\dimlabel\times \numclass} }} \sum_{i=1}^{\nbsamp}\|C\lab_i -z_i\|_2^2.
\end{align*}
for $Z = (z_1, \ldots, z_n) \in \reals^{\dimlabel\times \nbsamp}$.
Minimization in $C$ can be performed analytically such that the problem can be rewritten
\begin{align*}
	\obj(Z) & =\min_{N\in \mathcal{N}} \Tr((\id_\nbsamp -N)Z^\top Z),
\end{align*}
where $\mathcal{N} = \{N = \Lab(\Lab^\top \Lab)^{-1}\Lab^\top \in \reals^{\nbsamp \times \nbsamp} \quad \mbox{for} \quad \Lab \in \mathcal{Y}, \quad \Lab^\top \Lab \succ 0 \}$ is the set of normalized equivalence matrices.

\paragraph{Spectral clustering}
A natural relaxation of K-means is spectral clustering, that considers 
\[
\mathcal{P} = \{P \in \reals^{\nbsamp\times \nbsamp} \quad \mbox{s.t.}  \: P\succeq 0,\: P^2 =P,\: \Rank(P)=\numclass\} \supset \mathcal{N}
\]
instead of the set of normalized equivalence matrices. The solution of 
\begin{align*}
\obj(Z) & =\min_{P\in \mathcal{P}} \Tr((\id_\nbsamp -P)Z^\top Z)
\end{align*}
is then given by finding the $\numclass$ largest eigenvectors of the Gram matrix $Z^\top Z$. Formally the objective is written 
\begin{flalign*}
	\obj(Z) & = \sum_{i=\nbsamp-\numclass+1}^{\nbsamp}\sigma_i^2(Z),
\end{flalign*}
where for a matrix $A$, $\sigma_1(A) \geq \ldots, \geq \sigma_\nbsamp(A)$ are the singular values of $A$ in decreasing order.
The objective $\obj$ is then a spectral function of the matrix $Z$.

\paragraph{Convex clustering}
The convex clustering objective reads for  $ \labpred = (\labpred_1;\ldots;\labpred_\nbsamp) \in \reals^{\dimlabel\nbsamp}$
\begin{align}\label{eq:cvx_clustering}
	\obj(\labpred) & =  \min_{\lab^{(1)},\ldots, \lab^{(\nbsamp)} \in \reals^\dimlabel}  \sum_{i=1}^\nbsamp\frac{1}{2}\|\lab^{(i)}-\labpred^{(i)} \|_2^2  + \sum_{i<j}\|\lab^{(i)}-\lab^{(j)}\|_2,\\
	&  = \min_{ y \in \reals^{\dimlabel\nbsamp}} \frac{1}{2}\| y - \labpred\|_2^2 + \|D y\|_G \nonumber
\end{align}
where  $ \lab = (\lab_1;\ldots;\lab_\nbsamp) \in \reals^{\dimlabel\nbsamp}$ and $D \in \reals^{\dimlabel\nbsamp(\nbsamp-1)/2 \times \dimlabel\nbsamp}$ maps $ y$ to the concatenation of all possible $y_i-y_j$ for $i<j$ and $\|\cdot\|_G$ is a group norm, i.e., $\|x\|_G = \sum_{g\in \mathcal{G}} \|x_g\|_2$ where $\mathcal{G}$ is a partition of $\{1, \ldots, N\}$ for $x\in \reals^N$ and $x_g \in \reals^{s_g}$ is the vector corresponding to the group $g$ of size $s_g$. Here the groups are defined by all possible differences for $i<j$ in Eq.~\eqref{eq:cvx_clustering}.

\begin{proposition}
	The convex-clustering objective 
	\[
	\obj(\labpred) = \min_{\lab \in \reals^{\dimlabel\nbsamp}}\frac{1}{2}\| y -  \labpred\|_2^2 + \|D y\|_G
	\]
	is convex, Lipschitz-continuous and smooth with parameters
	\begin{align*}
		\lip_{\operatorname{cvx-cluster}}  \leq \frac{n(n-1)}{2}, \qquad \smooth_{\operatorname{cvx-cluster}} \leq 1.
	\end{align*}
\end{proposition}
\begin{proof}
The convex clustering objective $\obj$ is the Moreau envelope of the function $\Omega: y \rightarrow \|D y\|_G$. It is therefore convex and $1$-smooth, i.e., 
$
\smooth_{\obj} =1.
$
Moreover, the Moreau envelope can be rewritten
\[
\obj(\labpred) = \sup_{z\in \dom(\Omega^*)} \labpred^\top z - \Omega^*(z) -\frac{1}{2}\|z\|_2^2, 
\]
where $\Omega^*$ is the convex conjugate of $\Omega$. Therefore $\nabla \obj(\labpred) \in \dom(\Omega^*)$. We have that 
\[
\Omega^*(z) = \sup_{y \in \reals^\dimlabel} z^\top y - \|Dy\|_G \geq \sup_{y \in \reals^\dimlabel} z^\top y - \frac{\nbsamp(\nbsamp-1)}{2}\|y\|_2,
\]
such that the supremum is finite only if $\|z\|_2 > \frac{\nbsamp(\nbsamp-1)}{2}$. Therefore
\[
\nabla \obj(\labpred) \in \dom(\Omega^*) \subset \mathcal{B}_{2}\left(0,\frac{\nbsamp(\nbsamp-1)}{2}\right),
\]
where $\mathcal{B}_2(0,\frac{\nbsamp(\nbsamp-1)}{2})$ is the Euclidean ball centered at $0$ with radius $\frac{\nbsamp(\nbsamp-1)}{2}$.
\end{proof}

\subsection{Bilinear and linear layers}
\paragraph{Vectorized matrix-products as a bilinear operation}
Given two matrices $A \in \reals^{n\times d}$ and $B\in \reals^{d \times p}$, the matrix product $AB$ is defined by a tensor $\mathcal{M} = ((\idm_d \otimes e_{(q \operatorname{mod}n)+1}) (f_{\lceil q/n\rceil}^\top \otimes \idm_d))_{q ={1,\ldots, np}} \in \reals^{nd \times dp \times np}$ where $e_i$ is the i\textsuperscript{th} canonical vector in $\reals^n$ and $f_j$ is the j\textsuperscript{th} canonical vector in $\reals^p$  such that
\begin{equation}\label{eq:tensor_mat_prod}
\Vect(AB) = \mathcal{M}[\Vect(A), \Vect(B), \cdot].
\end{equation}
This can be checked as for $q = i + n(j-1) \in \{1, \ldots, np\}$, with $i \in \{1, \ldots n\}, j\in \{1,\ldots,p\}$,
\begin{align*}
\Vect(AB)_q = (AB)_{ij} & =\Vect(e_i^\top A)^\top \Vect(B f_j) \\
& = \Vect(A)^\top(\idm_d \otimes e_i) (f_j^\top \otimes \idm_d) \Vect(B) \\
& = (\mathcal{M}[\Vect(A), \Vect(B), \cdot])_q.
\end{align*}

\paragraph{Convolutional layer detailed}
For completeness we detail the convolution for an image.
For a convolutional layer, the input is an image $\mathcal{I} \in \reals^{C\times H \times B}$ with $C$ channels each composed of a matrix of height $H$ and breadth $B$ the weights are given by $\tilde C$ filters $\mathcal{F}_1, \ldots, \mathcal{F}_{\tilde C} \in \reals^{C \times K \times K}$ of patch size $K$ and the biases are given by $b\in\reals^{\tilde C}$. The convolution of the image by a filter $\mathcal{F}_{\tilde c}$, with $\tilde c \in \{1, \ldots, \tilde C\}$ with additional bias $b_{\tilde c}$, is given at point $i,j$ as 
\[
\mathcal{C}_{\tilde c,i,j} = \sum_{c=1}^C \langle \mathcal{F}_{\tilde c}[c, \cdot, \cdot], E_{row,i}^\top \mathcal{I}[c, \cdot, \cdot] E_{col,j} \rangle + b_{\tilde c},
\]
where $\mathcal{F}_{\tilde c}[c, \cdot, \cdot]$ is the filter of size $K \times K$ in channel $c$ of filter $\mathcal{F}_{\tilde c}$ and 
$I[c, \cdot, \cdot]$ is the image in channel $c$. 

The matrices $E_{row,i} \in \reals^{H \times K}$ and $E_{col,j} \in \reals^{B \times K}$ extract rows and columns of $I[c, \cdot, \cdot]$. They are bands with a diagonal of $K$ ones centered at positions $i$ or $j$.
If the pattern of the patch is given as $P = \ones_K\ones_K^\top$, the extraction matrices read 
$E_{row, i} = e_i \otimes \ones_K^\top \in \reals^{H\times K}$, $e_i \in\reals^H$ for $i \in \{1, \ldots H\}$, 
similarly $E_{col, j} = e_j \otimes \ones_K^\top \in \reals^{W \times K}$. 
They satisfy $E_{row,i}^\top E_{row,i} = \id_{K^2}$ and $E_{row,i} E_{row,i}^\top \in \reals^{H\times H}$ is a projector. Similarly facts apply for $E_{col,j}$ except that one replaces $H$ by $B$.
The output of the convolution with all filters is then a tensor $\mathcal{C} \in \reals^{\tilde H \times \tilde B \times \tilde C}$ where $\tilde H$ and $\tilde B$ depend on the choices of the stride chosen in the convolution.

\paragraph{Smoothness of fully-connected layer}
For a fully connected layer, the bilinear function $\bilinear(\state, \ctrl)  \rightarrow U^\top X$ for $\ctrl = \Vect(U)$, $\state= \Vect(X)$ is clearly 1-smooth (because $\| U^\top X\|_F \leq \|U\|_F\|X\|_F$). The linear part $\bilinear^\ctrl$ is clearly $1$-Lipschitz continuous.
So we get 
\[
\smooth_{full} = 1,  \quad  \lipp^\ctrl_{full} = 1.
\]

\paragraph{Smoothness of convolutional layer}
For a convolution, by Lemma~\ref{lem:bilin_smooth_elt_wise}, we only need to compute the smoothness of the convolution of an image with one filter. This is done by the following Lemma.

\begin{lemma}
	Consider $p$ subsets $S_k$ of $\{1, \ldots, n\}$ of size $|S_k| =d$. Denote $\Pi_k \in \{0,1\}^{d\times n}$ the linear form that extracts the $S_k$ coordinates of  a vector of size $n$, i.e., $\Pi_k z = z_{S_k}$ for $z \in \reals^ n$. The convolution of $z \in \reals^n$ by $w \in \reals^ d$ through the $p$ subsets $S_k$ defined as
	\[
	\bilinear(z, w) = (w^\top \Pi_1 z; \ldots; w^\top \Pi_p z)
	\]
	is $\smooth_\bilinear = \sqrt{\max_{i=1, \ldots, n} |V_i|} $-smooth where $V_i = \{S_j : i \in S_j\}$. 
\end{lemma}
\begin{proof}
	We have \begin{align*}
\|\bilinear(z, w)\|_2^ 2  = \sum_{j=1}^ p (w^\top \Pi_j z)^ 2 & \leq \sum_{j=1}^ p \|w\|_2^ 2 \|z_{S_j}\|_2^ 2\\
 &= \|w\|_2^ 2 \sum_{i=1}^d \sum_{S_j \in V_i} z_i^ 2 \leq \|w\|_2^ 2 \max_{i=1, \ldots, n} |V_i| \|z\|_2^ 2.
	\end{align*}
\end{proof}
Concretely, for a convolution such that at most $p$ patches contain a coordinate $i$ the convolution is $\sqrt{p}$-smooth. If the patches do not overlap then the convolution is $1$-smooth. If the convolution has a stride of 1 and the operation is normalized by the size of the filters then the convolution has again a smoothness constant of 1. Generally for a 2d convolution with a kernel of size $k \times k$ and a stride of $s$, we have $\max_{i=1, \ldots, n} |V_i|=  \left \lceil \frac{k}{s} \right\rceil^2$ and so
\[
\smooth_{\text{conv}} = \left \lceil \frac{k}{s} \right\rceil, \quad \lipp_{\text{conv}} = \left \lceil \frac{k}{s} \right\rceil.
\]

\paragraph{Batch of inputs}
For batch of inputs, the smoothness constants of the non-linear and bilinear parts do not change by Lemmas~\ref{lem:smooth_elt_wise} and~\ref{lem:bilin_smooth_elt_wise}. The Lipschitz-constant of the linear part of the biaffine function is modified using Lemma~\ref{lem:smooth_elt_wise} item 3. Namely for a batch of size $\batchsize$, the fully connected layers or the convolutional layers have a linear part whose Lipschitz constant is given by $\lipp_{\biaffine} = \sqrt{\batchsize}$.

\subsection{Activation functions}
The Lipschitz and smoothness constants of an element-wise activation $\activ_{t}$ function are defined by the Lipschitz and smoothness constant of the scalar function $\activelt_t$ from which it is defined.
Denote by $f(x) :=\log(1+\exp(x))$, we have $f'(x) = (1+\exp(-x))^{-1}$, $f''(x) = (2+2\cosh(x))^{-1}$, $f'''(x) = -\sinh(x) /(2(1+\cosh(x)^2))$.
\paragraph{Soft-plus}
For $\activ$ defined by element-wise application of $\activelt(x) = f(x)$, we get 
\begin{align*}
\lip_{\operatorname{softplus}} = 1, \qquad
\smooth_{\operatorname{softplus}}  = 1/4.
\end{align*}

\paragraph{Sigmoid}
For $\activ$ defined by the element-wise application of  $\activelt(x) = f'(x)$, we get 
\begin{align*}
	\lip_{\operatorname{sig}}  = 1/4, \qquad
	\smooth_{\operatorname{sig}}  = 1/10.
\end{align*}

\paragraph{ReLU}
For $\activ$ defined by the element-wise application of $\activelt(x) = \max(0,x)$, we get
\[
\lip_{\Relu} =1, \quad \smooth_{\Relu} \ \mbox{not defined},
\]
since the function is not continuously differentiable.

\paragraph{Soft-max layer}
A soft-max layer takes as input $x\in \reals^d$ and outputs $f(x) = \exp(x)/(\exp(x)^\top \ones_d)$ where $\exp(x)$ is the element-wise application of $\exp$. Its gradient is given by
\[
\nabla f(x) =  \frac{\diag(\exp(\labpred))}{\exp(\labpred)^\top \ones_\dimlabel} - \frac{\exp(\labpred)\exp(\labpred)^\top}{(\exp(\labpred)^\top \ones_\dimlabel)^2}.
\]
Its second-order information can be computed as for the batch-normalization layer, we get then 
\begin{align*}
	\lip_{\operatorname{softmax}}  = 2,\qquad
	\smooth_{\operatorname{softmax}}  = 4.
\end{align*}

\subsection{Normalization layers}
\begin{proposition}
	The batch normalization operation $\normal_{\batchnorm}:\reals^{\diminput\batchsize} \rightarrow \reals^{\diminput\batchsize}$ defined as in~\eqref{eq:batchnorm} is 
	\begin{enumerate}[nosep, label=(\roman*)]
		\item bounded by $\bound_{\batchnorm} = \diminput \batchsize$,
		\item Lipschitz-continuous with a constant $\lip_{\batchnorm} = 2\batchreg^{-1/2}$,
		\item smooth with a constant $\smooth_{\batchnorm} = 2 \diminput\batchsize^{-1/2}\batchreg^{-1}$.
	\end{enumerate}
\end{proposition}
\begin{proof}
	The batch-normalization layer as defined in~\eqref{eq:batchnorm} is the composition $\normal = \normal_2 \circ \normal_1$ of a centering step 
	\[
	\normal_1(\state) = \Vect\left(\Latent - \Latent\frac{\ones_{\batchsize}\ones_{\batchsize}^\top}{\batchsize}\right)
	\]
	and a normalization step 
	\[
	\normal_2(\tilde \state) = \Vect\left(\diag\left(\left(\frac{1}{\batchsize}\diag(\tilde \Latent \tilde \Latent^\top) +\regbatchnorm\ones_{\diminput}\right)^{-1/2}\right) \tilde \Latent\right),
	\]
	where here and thereafter $\Latent, \tilde \Latent \in \reals^{\diminput \times \batchsize}$, $\state = \Vect(\Latent), \tilde \state = \Vect(\tilde \Latent)$.
	
	The centering step is an orthonormal projection, i.e., $\normal_1(\state) = \Vect(\Latent\Pi_\batchsize) = (\Pi_\batchsize \otimes \idm_{\diminput})\state$ where $\Pi_\batchsize = \idm_\batchsize - \frac{\ones_{\batchsize}\ones_{\batchsize}^\top}{\batchsize}$ is an orthonormal projector and so is $(\Pi_\batchsize \otimes \idm_\batchsize)$. Therefore we have $\lip_{\normal_1} \leq 1$ and $\smooth_{\normal_1} = 0$. For the normalizations step denote for $x \in\reals^\batchsize$, and $\bar x =(x_1;\ldots;x_\diminput)\in\reals^{\batchsize\diminput}$ with $x_i \in \reals^m$,
	\begin{align*}
	f(x)  = \sqrt{\frac{1}{\batchsize}\|x\|_2^2 +\regbatchnorm}, \qquad 
	g(x)  = \left(\frac{x_i}{f(x)}\right)_{i=1,\ldots,\batchsize}, \qquad
	\bar g(\bar x)  = (g(x_1); \ldots; g(x_\diminput)) \in \reals^{\batchsize \diminput},
	\end{align*}
	such that $ \normal_2(\tilde \state)  = T_{m, d}\bar g(T_{d, m}  \tilde \state)$, where $T_{d, m}$ is the linear operator such that $T_{d, m}\Vect(\Latent) = \Vect(\Latent^\top)$ for any $\Latent \in \reals^{d \times m}$.
	First we have that
	\begin{align*}
	\|\bar g(\bar x)\|_2 \leq \diminput \max_{i\in \{1,\ldots, d\}}\|g(x_i)\|_2 \leq \diminput \batchsize^{1/2},
	\end{align*}
	such that 
	\[
	\bound_{\normal_2} \leq \diminput \batchsize^{1/2}.
	\]
	Then the gradients can be computed as
	\begin{align*}
	\nabla f(x) & = \frac{x}{mf(x)} = \frac{g(x)}{m} \in \reals^m, \\
	\nabla g(x) & = \frac{f(x) \idm_\batchsize - \nabla f(x) x^\top}{f(x)^2} = \frac{m f(x)^2 \idm_\batchsize -xx^\top}{mf(x)^3} \in \reals^{m\times m},\\
	\nabla \bar g(\bar x) & =  \diag(\nabla g(x_1), \ldots, \nabla g(x_m)) \in \reals^{md \times md},
	\end{align*}
	where for a sequence of matrices $X_1, \ldots X_\horizon \in \reals^{d \times p}$ we denote by 
	\[
	\diag(X_1, \ldots, X_\horizon) = \left(\begin{matrix}
	X_1 & 0 & \ldots & 0 \\
	0 & \ddots & \ddots & \vdots \\
	\vdots & \ddots & \ddots & 0 \\
	0 & \ldots & 0& X_\horizon 
	\end{matrix}\right) \in \reals^{d\horizon \times p \horizon},
	\]
	the corresponding block diagonal matrix.
	Therefore we get
	\begin{align*}
	\|\nabla g(x)\|_{\opnormidx} & \leq \frac{mf(x)^2 + \|x\|_2^2}{mf(x)^3}
	\leq \frac{2\batchsize^{-1}\|x\|^2 +\regbatchnorm}{(\batchsize^{-1}\|x\|_{\opnormidx}^2 + \regbatchnorm)^{3/2}}  
	\leq c\regbatchnorm^{-1/2}, \\
	\|\nabla \bar g(\bar x)\|_2 & \leq c\regbatchnorm^{-1/2},
	\end{align*}
	where $c = 2/(3/2)^{3/2} \approx 1.1$ and we used that the spectral norm of the block-diagonal matrix is given by the maximal spectral norm of its block diagonal components.
	Since $T_{m,d}$, $T_{d, m}$ are orthonormal operators, we get 
	\begin{align*}
	\lip_{\normal_2} \leq 2\regbatchnorm^{-1/2}.
	\end{align*}
	The second order tensor of $g$ reads
	\begin{align*}
	\nabla^2 g(x) & = \frac{3}{\batchsize^2f(x)^5} x\boxtimes x\boxtimes x -\frac{1}{\batchsize f(x)^3}\left(\sum_{i=1}^m x \boxtimes e_i \boxtimes e_i + e_i \boxtimes x \boxtimes e_i + e_i \boxtimes e_i \boxtimes x\right) \in \reals^{\batchsize \times \batchsize \times \batchsize}, \\
	\nabla^2 \bar g(\bar x) & = \diag^3(\nabla^2 g(x_1), \ldots, \nabla^2 g(x_d)),
	\end{align*} 
	where $e_i\in \reals^\batchsize$ is the i\textsuperscript{th} canonical vector in $\reals^\batchsize$ and for a sequence of tensors $\mathcal{X}_1,\ldots, \mathcal{X}_d$ we denote by $ \mathcal{X} =
	\diag^3(\mathcal{X}_1, \ldots, \mathcal{X}_d) \in \reals^{dm \times dm\times dm}
	$ the tensor whose diagonal is composed of the tensors $\mathcal{X}_1,\ldots \mathcal{X}_d$  such that $\mathcal{X}_{i+(m-1)p,j+ (m-1)p,k + (m-1)p} = (\mathcal{X}_p)_{ijk}$ and 0 outside the diagonal.
	We get then by definition of the tensor norm,
	\begin{align*}
	\|\nabla^2 g(x)\|_{\tensnormidx} \leq \frac{3\|x\|_2^3}{\batchsize^2f(x)^5} + \frac{3\|x\|_2}{\batchsize f(x)^3} = \frac{3\|x\|_2(\|x\|_2^2 + \batchsize f(x)^2)}{\batchsize^2f(x)^5} & = \frac{3\|x\|_2(2\|x\|_2^2 +\batchsize\batchreg)}{\batchsize^2(\batchsize^{-1}\|x\|_2^2+\batchreg)^{5/2}} \\
	& 
	\leq \frac{3}{\batchsize^{-1/2}} \frac{\sqrt{c}(2c+1)}{(c+1)^{5/2}}(\batchsize\batchreg)^{-1},
	\end{align*}
	where $c =(1+\sqrt{5})/4$ such that $3\frac{\sqrt{c}(2c+1)}{(c+1)^{5/2}} \approx 1.6$.
	Therefore we get $\|\nabla \bar g(\bar x)\|_{\tensnormidx} \leq \diminput \max_{i\in \{1,\ldots, \diminput\}}\|\nabla^2 g(x_i)\|_{\tensnormidx}$ and
	\[
	\smooth_{\normal_2} \leq 2\diminput\batchsize^{-1/2}\batchreg^{-1}.
	\]
\end{proof}

\subsection{Pooling layers}
We consider pooling layers for which the patches do not coincide such that they amount to a (potentially non-linear) projection.
\paragraph{Average pooling}
The  average pooling layer is a linear operation. If the patches do not coincide, it is a projection with Lipschitz constant one. 
\begin{align*}
\lip_{\avg} = 1, \qquad
\smooth_{\avg}  = 0.
\end{align*}

\paragraph{Max-pooling}
Given an image $\mathcal{I} \in \reals^{C\times H \times B}$ with $C$ channels each composed of a matrix of height $H$ and breadth $B$, the max pooling layer extracts  $\nbpatch$ patches of the form $P^{i,j} = E_{row,i}^\top \mathcal{I}[c, \cdot, \cdot] E_{col,j}$ where $E_{row,i} \in \reals^{H \times K}$ and $E_{col,j} \in \reals^{B \times K}$ extract rows and columns of $I[c, \cdot, \cdot]$ respectively. On each of this patch their maximum value is taken as the output, namely, the output image reads $\tilde{\mathcal{I}}_{c, i ,j} = \max_{k, l}  P^{i,j}_{k,l}$.
It is naturally non-continuously differentiable and  it is $1$-Lipschitz continuous if the patches do not coincide.
\[
\lip_{\maxpool} = 1\, \qquad \smooth_{\maxpool} \quad \mbox{not defined}.
\]
\subsection{Auto-encoders, composition of chains of computations}
For $\horizon$ vectors $(\ctrl_1;\ldots; \ctrl_\horizon) \in \reals^{\sum_{t=1}^\horizon \dimctrl_t}$ and $1\leq s\leq t\leq \horizon$, we denote 
$
\ctrl_{s:t} = (\ctrl_s;\ldots;\ctrl_t) \in \reals^{\sum_{r=s}^{t} \dimctrl_{r}}.
$
For $\horizon$ functions $\dyn_t:\reals^{\dimstate_{t-1}} \times \reals ^{\dimctrl_t} \rightarrow \reals^{\dimstate_t}$, we can split the chain of computations of the $\horizon$ functions $\dyn_t$ into smaller chains of computations.  Namely, for $1\leq s\leq t\leq \horizon$, we denote the output of the chain of computations defined by $\dyn_s,\ldots\dyn_t$ as
\begin{align*}
	\dyn_{s\rightarrow t}(\state_{s-1}, \ctrls_{s:t}) &= \state_t \\
	\mbox{s.t.} \quad  \state_r & = \dyn_r(\state_{r-1}, \ctrl_r) \quad \mbox{for} \ r\in \{s, \ldots, t\}.
\end{align*}
In particular, we have 	$\dyn_t = \dyn_{t\rightarrow t}$. The output of the chain of computations of the $\horizon$ fucntions $\dyn_t$ can then be split as 
\begin{align*}
	\dyn_{1\rightarrow\horizon}(\state_0, \ctrl_{1:\horizon}) & = \dyn_{t+1\rightarrow \horizon}(\dyn_{1\rightarrow t}(\state_0, \ctrl_{1:t}), \ctrl_{t+1:\horizon})  \quad \mbox{for any} \ t\in \{1, \ldots \horizon-1\} .
\end{align*}
On the other hand, the composition of two chains of computations can readily be seen as a chain of computations. Namely, for two chains of computations $\chain$ and $\chainaux$ with computations $(\dyn_t^\chain)_{t=1}^{\horizon_\chain}$ and $(\dynaux_t^\chainaux)_{t=1}^{\horizon_\chainaux}$, parameters $\ctrls$ and $\auxctrls$ respectively, the composition of $\chain$ and $\chainaux$ is 
\[
\chainauxx(\state_0, \ctrls) = \chainaux(\chain(\state_0, \ctrls), \auxctrls).
\]
It is a chain of $\horizon_\chain+\horizon_{\chainaux}$ computations 
\[
\dynauxx_t  =\begin{cases}
	\dyn_t &\quad \mbox{for} \ t\in \{1, \ldots, \horizon_\chain\}\\
	\dynaux_{t-\horizon_\chain}  &\quad \mbox{for} \ t\in \{\horizon_\chain+1,  \horizon_\chain + \horizon_{\chainaux}\}
\end{cases},
\]
with input $\state_0$ and parameters $\auxxctrls = (\ctrls; \auxctrls) \in \reals^{\sum_{t=1}^{\horizon_\chain} \dimctrl_t^\chain + \sum_{t=1}^{\horizon_{\chainaux}} \dimctrl_t^\chainaux}$ such that 
\[
\auxxctrl_t  =\begin{cases}
	\ctrl_t &\quad \mbox{for} \ t\in \{1, \ldots, \horizon_\chain\}\\
	\auxctrl_{t-\horizon_\chain}  &\quad \mbox{for} \ t\in \{\horizon_\chain+1,  \horizon_\chain + \horizon_{\chainaux}\}
\end{cases}.
\]

\subsection{Residual Networks}\label{ssec:residual}
Recall the architecture of a residual network
\begin{align*}
\state_t & =	\nonlin_t(\biaffine_t(\state_{t-1}, \ctrl_t) + \state_{t-2}) \quad \mbox{for} \: t=1, \ldots, \horizon  \\
\state_0 & = \rand, \quad \state_{-1} = 0, \nonumber
\end{align*}
where we assume $\biaffine_t: \reals^{\dimstate_{t-1}} \times  \reals^{\dimctrl_t}  \rightarrow \reals^{\diminter_t}$ such that $\state_{t-1} \in \reals^{\dimstate_{t-1}}, \state_{t-2} \in \reals^{\diminter_t}$ and
\[
\biaffine_t(\state_{t-1}, \ctrl_t ) = \mathcal{B}_t[\state_{t-1}, \ctrl_t, \cdot] + B_t^\ctrl \ctrl_t + B_t^\state \state_{t-1} + \linearcste_t,
\]
where $\mathcal{B} = (B_{t, 1}, \ldots, B_{t, \diminter_t})$ is a tensor.
They can be expressed in terms of the variable $\bar \state_t = (\state_t, \state_{t-1})$ as
\begin{align}\label{eq:dyn_res_net}
\bar \dyn_t(\bar \state_{t-1}, \ctrl_t) & = \bar \nonlin_t(\bar \biaffine_t(\bar \state_{t-1}, \ctrl_t)),
\end{align}
where $\bar \biaffine_t$ is defined as
\begin{align*}
\bar \biaffine_t(\bar \state_{t-1}, \ctrl_t ) & = \bar \bilinear_t(\bar \state_{t-1}, \ctrl_t ) + \bar {\bilinear_t}^\ctrl(\ctrl_t) + \bar {\bilinear_t}^\state(\bar \state_{t-1}) + \bar {\linearcste_t} \\
&  = \bar{\mathcal{B}}_t[\bar \state_{t-1}, \ctrl_t, \cdot] + \bar B_t^\ctrl \ctrl_t + \bar B_t^\state \bar \state_{t-1} +  \bar \linearcste_t, \\
 \bar {\mathcal{B}}_t & = (\bar B_{t,1}, \ldots, \bar B_{t, \diminter_t}, \underbrace{0_{\dimctrl_t \times (\dimstate_{t-1} + \diminter_t)}, \ldots, 0_{\dimctrl_t \times  (\dimstate_{t-1} + \diminter_t)}}_{\dimstate_{t-1}}), \\
 \bar B_{t, j} & = \left(\begin{matrix}
 B_{t, j}, 
 0_{\dimctrl_t \times \diminter_t}
 \end{matrix}\right) \quad \mbox{for} \ j \in \{1, \ldots, \dimstate_{t-1}\}, \\
 \bar B_t^\ctrl & = \left(\begin{matrix}
 B_t^\ctrl \\
 0_{\dimstate_{t-1} \times \dimctrl_t}
 \end{matrix}\right), \\
 \bar B_t^\state& = \left(\begin{matrix}
 B_t^\state& \id_{\diminter_t} \\
 \id_{\dimstate_{t-1}} & 0_{\dimstate_{t-1} \times \diminter_t}
 \end{matrix}\right), \\
 \bar {\linearcste_t} & = \left(\begin{matrix}
 \linearcste_t \\
 0_{\dimstate_{t-1}}
 \end{matrix}\right).
\end{align*}
Denoting $\bar \intervar_t =(\intervar_{t, 1}, \intervar_{t, 2})= \bar \biaffine_t(\bar \state_{t-1}, \ctrl_t)$, we have
\begin{align*}
\bar \nonlin_t(\bar \intervar_t) & = (\nonlin(\intervar_{t, 1}), \intervar_{t, 2}).
\end{align*} 
We can derive the smoothness constants of the layers of a residual network expressed as in~\eqref{eq:dyn_res_net} as
\begin{align*}
\smooth_{\bar{\bilinear_t}} & = \smooth_{\bilinear_t}, \qquad \lipp_{\bar{\bilinear^\ctrl_t}} = \lipp_{\bilinear^\ctrl_t}, \qquad \lipp_{\bar{\bilinear^\state_t}} \leq \lipp_{\bilinear^\state_t} + 1, \qquad \|\bar{ \linearcste_t} \|_{\normidx} = \|\linearcste_t\|_{\normidx}, \\
\bound_{{\bar \nonlin_t}} & \leq  (1 + \bound_{\nonlin_t}), \qquad \lip_{\bar{\nonlin_t}} \leq \max(1, \lip_{\nonlin_t}), \qquad \smooth_{\bar {\nonlin_t}} = \smooth_{\nonlin_t}.
\end{align*}
Proposition~\ref{prop:smoothdeep} can then be applied in this setting.

\subsection{Implicit functions}\label{ssec:implicit}
The smoothness constants of an implicit function are given in the following lemma. They can easily be refined by considering smoothness properties w.r.t. to each of the variables $\impfix$ and $ \impvar$ of the function $\impfunc$ defining the problem.
\begin{lemma}
	Let $\impfunc: (\impfix, \impvar) \rightarrow \impfunc(\impfix, \impvar) \in \reals$ for $\impfix \in \reals^{\dimimpfix}, \impvar \in \reals^\dimimpvar$ be  s.t. $\impfunc(\impfix, \cdot)$ is $\mu_\impfunc$-strongly convex for any $\impfix$. Denote $\Impfunc(\impfix) = \argmin_{\impvar \in \reals^\dimimpvar} \impfunc(\impfix, \impvar)$. Provided that $\impfunc$ has a $\smooth_\impfunc$-Lipschitz gradient and a $\smoothess_\impfunc$-Lipschitz Hessian, the smoothness constants of $\Impfunc$ are bounded as 
	\[
	\lip_\Impfunc \leq L_\impfunc \mu_\impfunc^{-1}, \qquad \qquad \smooth_\Impfunc \leq \smoothess_\impfunc\mu_\impfunc^{-1}(1+\lip_\Impfunc)(1+L_\impfunc\mu_\impfunc^{-1})  \leq  \smoothess_\impfunc\mu_\impfunc^{-1}(1+L_\impfunc\mu_\impfunc^{-1})^2.
	\]
\end{lemma}
\begin{proof}
	By the implicit function theorem, $\Impfunc(\impfix)$ is uniquely defined and its gradient is given by 
	\begin{align*}
	\nabla \Impfunc(\impfix) & = - \nabla_{\impfix} \impgrad(\impfix, \Impfunc(\impfix))\nabla_\impvar \impgrad(\impfix, \Impfunc(\impfix))^{-1} 
	= - \nabla_{\impfix, \impvar}^2\impfunc(\impfix, \Impfunc(\impfix)) \nabla_{\impvar, \impvar}^2 \impfunc(\impfix, \Impfunc(\impfix))^{-1},
	\end{align*}
	where $\impgrad(\impfix, \impvar) = \nabla _\impvar \impfunc(\impfix, \impvar)$. The Lipschitz constant of $\Impfunc$ follows from that. For the smoothness we compute its second order information and bound the corresponding tensors. Note that the same results can be obtained by simply splitting the functions in appropriate terms. We have
	\begin{align*}
		\nabla \Impfunc(\impfix) & = h(\impfix, \Impfunc(\impfix)),  \\
		\mbox{where} \qquad h(\impfix, \impvar) & = -\nabla_{\impfix} \impgrad(\impfix, \impvar)\nabla_\impvar \impgrad(\impfix, \impvar)^{-1} 
		= - \nabla_{\impfix, \impvar}^2\impfunc(\impfix, \impvar) \nabla_{\impvar, \impvar}^2 \impfunc(\impfix, \impvar)^{-1}.
	\end{align*}
	Using Lemma~\ref{lem:invcomp}, we get
	\begin{align*}
		\nabla^2 \Impfunc(\impfix) = \ & \nabla_\impfix h(\impfix, \Impfunc(\impfix)) + \nabla_\impvar h(\impfix, \Impfunc(\impfix))[\nabla \Impfunc(\impfix), \cdot, \cdot] \\
		= \ & - \nabla_{\impfix\impfix}^2 \impgrad(\impfix, \Impfunc(\impfix))[\cdot, \cdot, \nabla_\impvar \impgrad(\impfix, \Impfunc(\impfix))^{-1}] \\
		& - \nabla^2 \impgrad_{\impfix\impvar}(\impfix, \Impfunc(\impfix))[\cdot, \nabla_\impvar \impgrad(\impfix, \Impfunc(\impfix))^{-1}\nabla_\impfix \impgrad(\impfix, \Impfunc(\impfix))^\top, \nabla_\impvar \impgrad(\impfix, \Impfunc(\impfix))^{-1}] \\
		& - \nabla_{\impvar\impfix}^2 \impgrad(\impfix, \Impfunc(\impfix))[\nabla \Impfunc(\impfix), \cdot, \nabla_\impvar \impgrad(\impfix, \Impfunc(\impfix))^{-1}] \\
		& 
		- \nabla^2 \impgrad_{\impvar\impvar}(\impfix, \Impfunc(\impfix))[\nabla \Impfunc(\impfix),\nabla_\impvar \impgrad(\impfix, \Impfunc(\impfix))^{-1}\nabla_\impfix \impgrad(\impfix, \Impfunc(\impfix))^\top , \nabla_\impvar \impgrad(\impfix, \Impfunc(\impfix))^{-1}].
	\end{align*}
	The result follows by using Facts~\ref{fact:lip_smooth} and~\ref{fact:tensor_norm}. We observe that second derivatives of $\impgrad$ correspond to third derivatives of $\impfunc$, whose norms are bounded by $\smoothess_\impfunc$ by assumption. Moreover we have that $\|\nabla_\impvar \impgrad(\impfix, \Impfunc(\impfix))^{-1}\|_2 = \|\nabla_{\impvar, \impvar}^2 \impfunc(\impfix, \impvar)^{-1} \|_2 \leq \mu_\impfunc^{-1}$ by assumption. 
\end{proof}

The  approximation error of the gradient when using an approximate minimizer inside the expression of the gradient is provided in the following lemma. It follows from smoothness considerations.
\approximplicitgrad*
\begin{proof}
	Denote $ h(\impfix, \impvar)  = -\nabla_{\impfix} \impgrad(\impfix, \impvar)\nabla_\impvar \impgrad(\impfix, \impvar)^{-1} 
	= - \nabla_{\impfix, \impvar}^2\impfunc(\impfix, \impvar) \nabla_{\impvar, \impvar}^2 \impfunc(\impfix, \impvar)^{-1}$ such that $\widehat \nabla \hat \Impfunc(\impfix) =  h(\impfix, \hat \Impfunc(\impfix))$ and $\nabla \Impfunc(\impfix)  = h(\impfix, \Impfunc(\impfix))$. The approximation error is given by computing the smoothness constant of $h(\impfix, \cdot)$ for any $\impfix$. We bound the gradient of $h(\impfix, \cdot)$ (same results can be obtained by considering differences of the functions). From Lemma~\ref{lem:invcomp}, we have
	\[
	\nabla_\impvar h(\impfix, \impvar) = -\nabla_{\impvar\impfix}^2 \impgrad(\impfix, \impvar)[\cdot, \cdot, \nabla_\impvar \impgrad(\impfix, \impvar)^{-1}] - \nabla^2 \impgrad_{\impvar\impvar}(\impfix, \impvar)[\cdot, \nabla_\impvar \impgrad(\impfix, \impvar)^{-1}\nabla_\impfix \impgrad(\impfix, \Impfunc(\impfix))^\top , \nabla_\impvar \impgrad(\impfix, \impvar)^{-1}].
	\]
	The result follows by using Facts~\ref{fact:lip_smooth} and~\ref{fact:tensor_norm}. We observe that second derivatives of $\impgrad$ correspond to third derivatives of $\impfunc$, whose norms are bounded by $\smoothess_\impfunc$ by assumption. Moreover we have that $\|\nabla_\impvar \impgrad(\impfix, \Impfunc(\impfix))^{-1}\|_2 = \|\nabla_{\impvar, \impvar}^2 \impfunc(\impfix, \impvar)^{-1} \|_2 \leq \mu_\impfunc^{-1}$ by assumption. 
\end{proof}

\begin{lemma}\label{lem:invcomp}
	Let $\impgrad: (\impfix, \impvar) \rightarrow \impgrad(\impfix, \impvar) \in \reals^\dimimpvar$ for $\impfix\in \reals^\dimimpfix, \impvar \in \reals^\dimimpvar$ such that $\nabla_{\impvar}\impgrad(\impfix, \impvar) \in \reals^{\dimimpvar \times \dimimpvar}$ is positive definite for all $\impfix\in \reals^\dimimpfix, \impvar \in \reals^\dimimpvar$.
	Denoting $h(\impfix, \impvar) = \nabla_\impfix \impgrad(\impfix, \impvar) \nabla_\impvar \impgrad(\impfix, \impvar)^{-1}  \in \reals^{\dimimpfix \times \dimimpvar}$ we have
	\begin{align*}
	\nabla_\impfix h(\impfix, \impvar) = \nabla_{\impfix\impfix}^2 \impgrad(\impfix, \impvar)[\cdot, \cdot, \nabla_\impvar \impgrad(\impfix, \impvar)^{-1}] + \nabla^2 \impgrad_{\impfix\impvar}(\impfix, \impvar)[\cdot, \nabla_\impvar \impgrad(\impfix, \impvar)^{-1}\nabla_\impfix \impgrad(\impfix, \impvar)^\top, \nabla_\impvar \impgrad(\impfix, \impvar)^{-1}], \\
	\nabla_\impvar h(\impfix, \impvar) = \nabla_{\impvar\impfix}^2 \impgrad(\impfix, \impvar)[\cdot, \cdot, \nabla_\impvar \impgrad(\impfix, \impvar)^{-1}] + \nabla^2 \impgrad_{\impvar\impvar}(\impfix, \impvar)[\cdot, \nabla_\impvar \impgrad(\impfix, \impvar)^{-1}\nabla_\impfix \impgrad(\impfix, \impvar)^\top, \nabla_\impvar \impgrad(\impfix, \impvar)^{-1}].
	\end{align*}
\end{lemma}
\begin{proof}
	This follows from the product rule, Fact~\ref{fact:matgrad}, Lemma~\ref{lem:inv_der} and Fact~\ref{fact:tensormat_rule}.
\end{proof}
\begin{lemma}\label{lem:inv_der}
	Let $g:\reals^d \rightarrow S_n^{++}$ be differentiable and $h(x) = (g(x))^{-1}$. Then $\nabla h(x) = \nabla g(x)[\cdot, g(x)^{-1}, g(x)^{-1}]$.
\end{lemma}
\begin{proof}
	Let $x\in \reals^d$ and $\delta \in \reals^d$. Consider first $d=1$, such that $\nabla g(x) \in \reals^{n \times n}$.
	\begin{align*}
	h(x + \delta) &= (g(x) + \delta \nabla g(x) + o(\delta))^{-1}   = g(x)^{-1} - \delta g(x)^{-1}\nabla g(x)g(x)^{-1} + o(\delta).
	\end{align*}
	So in this case $\nabla h(x) = g(x)^{-1}\nabla g(x)g(x)^{-1} \in \reals^{n \times n}$.
	The result follows for $n=d$ by concatenating this result in a tensor such that for $d>1$, $\nabla h(x) = \nabla g(x)[\cdot, g(x)^{-1}, g(x)^{-1}]$. Alternatively it can directly be seen from the following first order approximation for $d>1$,
	\begin{align*}
		h(x + \delta) &= (g(x) + \nabla g(x)[\delta, \cdot, \cdot] + o(\|\delta\|_2))^{-1}  = g(x)^{-1} - g(x)^{-1}\nabla g(x)[\delta, \cdot, \cdot]g(x)^{-1} + o(\|\delta\|_2).
	\end{align*}
\end{proof}

\section{Optimization complexity proofs}\label{sec:optim_proofs}
\subsection{Smoothness of the objective}
\smoothobj*
\begin{proof}
	Consider $\obj, \chain$ to be twice differentiable. Same results can be obtained by considering differences of gradients. We get for $\ctrls \in \reals^\dimvar$,
	\[
	\nabla^2 (\obj\circ\chain)(\ctrls) = \nabla^2\chain(\ctrls)[\cdot, \cdot, \nabla \obj(\chain(\ctrls))] + \nabla \chain(\ctrls) \nabla^2\obj(\chain(\ctrls)) \nabla \chain(\ctrls)^\top.
	\]
	The norm of  $\nabla \obj(\chain(\ctrls))$ can either be directly bounded by $\lip_\obj$ or by using that for any $\ctrls , \ctrls' \in C$, $\|\nabla \obj(\chain(\ctrls))\|_{\normidx} \leq \|\nabla \obj(\chain(\ctrls')) \|_2 + \smooth_\obj \|\chain(\ctrls) - \chain(\ctrls')\|_{\normidx}$. By choosing $\ctrls' \in \argmin_{\ctrls \in \set}\|\nabla \obj(\chain(\ctrls))\|_{\normidx}$ and bounding the second term by the diameter of $\set$, we get a bound on $\sup_{\ctrls \in \set} \|\nabla \obj(\chain(\ctrls))\|_{\normidx}$. The result follows using Fact.~\ref{fact:lip_smooth} and the definitions of the norms used to bound $\lip_f$, $\smooth_f$ for a given function $f$.
\end{proof}

\smoothball*
\begin{proof}
	The smoothness properties of $\chain_{\state_0}$ on $\set'$ are given by considering $\hat \chain_{\state_0}(\Delta) = \chain_{\state_0}(\ctrls^* + \Delta) = \chain_{\state_0}(\ctrls)$ where $\Delta = \ctrls - \ctrls^*$ with $\|\Delta_t\|_{\normidx}\leq \setparam'$. The shifted chain of computations is given by 
	\begin{align*}
	 \hat\chain_{\state_0, t}(\Delta) & = \nonlin_t(\biaffine_t(\hat \chain_{\state_0, t-1}(\Delta), \ctrl^*_t + \Delta_t ))
	\end{align*}
	This means that $\hat \chain_{\state_0}(\Delta)$ is a chain of compositions defined by the same non-linearities $\nonlin_t$ and bi-affine functions $\hat \biaffine_t$ modified as 
	\begin{align*}
	\hat \biaffine_t(\state_{t-1}, \Delta) & = \biaffine_t(\state_{t-1}, \ctrl_t^* + \Delta_t) = \bilinear_t( \state_{t-1}, \Delta_t) + \bilinear^\ctrl_t(\Delta_t) + \hat \bilinear^\state_t(\state_{t-1}) + \hat{\linearcste_t},
	\end{align*}
	where 
	\[
	\hat \bilinear_t^\state(\state_{t-1}) = \bilinear^\state_t(\state_{t-1}) + \bilinear_t(\ctrl_t^*, \state_{t-1}) \qquad \hat{\linearcste_t} = \linearcste_t + \bilinear_t^\ctrl(\ctrl_t^*).
	\]
\end{proof}

\section{Detailed network}\label{sec:net_examples}
\paragraph{VGG network}
The VGG Network is a benchmark network for image classification with deep  networks. The objective is to classify images among $1000$ classes. Its architecture is composed of 16 layers described below.  We drop the dependency to the layers in their detailed formulation. We precise the number of patches $\nbpatch$ of the pooling or convolution operation, which, multiplied by the number of filters $\nbfilter$ gives the output dimension of these operations. 

For a fully connected layer we precise the output dimension $\diminput_{\operatorname{out}}$.

\begin{enumerate}[topsep=1ex,itemsep=-1ex,partopsep=1ex,parsep=1ex, leftmargin=*, start=0]
	\item $\rand_i \in \reals^{\nbpatch\nbfilter}$ with $\nbpatch =224{\times} 224$ and $\nbfilter=3$,
	\item $\dyn_1(\state, \ctrl) = 
	\activ_{\Relu}(\biaffine_{\conv}(\state, \ctrl))$ \newline
	with $\nbpatch_{\conv} = 224{\times }224$, $\nbfilter_{\conv}=64$,
	\item $\dyn_2(\state, \ctrl) =
	\pool_{\maxpool}(\activ_{\Relu}(\biaffine_{\conv}(\state, \ctrl)))$\newline
	with $\nbpatch_{\conv} = 224{\times }224$, $\nbfilter_{\conv}=64$,
	$\nbpatch_{\maxpool} = 112{\times }112$, $\nbfilter_{\maxpool} = 64$,
	\item $\dyn_3(\state, \ctrl) = 
	\activ_{\Relu}(\biaffine_{\conv}(\state, \ctrl))$ \newline
	with $\nbpatch_{\conv} = 112{\times }112$, $\nbfilter_{\conv}=128$
	\item $\dyn_4(\state, \ctrl) = \pool_{\maxpool}(\activ_{\Relu}(\biaffine_{\conv}(\state, \ctrl)))$\newline
	with $\nbpatch_{\conv} = 112{\times }112$, $\nbfilter_{\conv}=128$, 
	$\nbpatch_{\maxpool} = 56{\times }56$, $\nbfilter_{\maxpool} = 128$,
	\item $\dyn_5(\state, \ctrl) = \activ_{\Relu}(\biaffine_{\conv}(\state, \ctrl))$ \newline
	with $\nbpatch_{\conv} = 56{\times }56$, $\nbfilter_{\conv}=256$,
	\item $\dyn_6(\state, \ctrl) = \activ_{\Relu}(\biaffine_{\conv}(\state, \ctrl))$  \newline
	with $\nbpatch_{\conv} = 56{\times }56$, $\nbfilter_{\conv}=256$,
	\item $\dyn_7(\state, \ctrl) = \pool_{\maxpool}(\activ_{\Relu}(\biaffine_{\conv}(\state, \ctrl)))$  \newline
	with $\nbpatch_{\conv} = 56{\times }56$, $\nbfilter_{\conv}=256$,
	$\nbpatch_{\maxpool} = 28{\times}28$, $\nbfilter_{\maxpool} = 256$,
	\item $\dyn_8(\state, \ctrl) = \activ_{\Relu}(\biaffine_{\conv}(\state, \ctrl))$\newline
	with $\nbpatch_{\conv} = 28{\times }28$, $\nbfilter_{\conv}=512$,
	\item $\dyn_9(\state, \ctrl) = \activ_{\Relu}(\biaffine_{\conv}(\state, \ctrl))$\newline
	with $\nbpatch_{\conv} = 28{\times }28$, $\nbfilter_{\conv}=512$,
	\item $\dyn_{10}(\state, \ctrl) = \pool_{\maxpool}(\activ_{\Relu}(\biaffine_{\conv}(\state, \ctrl)))$\newline
	with $\nbpatch_{\conv} = 28{\times }28$, $\nbfilter_{\conv}=512$, 
	$\nbpatch_{\maxpool} = 14{\times}14$, $\nbfilter_{\maxpool} = 512$,
	\item $\dyn_{11}(\state, \ctrl) = \activ_{\Relu}(\biaffine_{\conv}(\state, \ctrl))$ \newline
	with $\nbpatch_{\conv} = 14{\times }14$, $\nbfilter_{\conv}=512$,
	\item $\dyn_{12}(\state, \ctrl) = \activ_{\Relu}(\biaffine_{\conv}(\state, \ctrl))$ \newline
	with $\nbpatch_{\conv} = 14{\times }14$, $\nbfilter_{\conv}=512$
	\item $\dyn_{13}(\state, \ctrl) = \pool_{\maxpool}(\activ_{\Relu}(\biaffine_{\conv}(\state, \ctrl)))$ \newline
	with $\nbpatch_{\conv} = 14{\times }14$, $\nbfilter_{\conv}=512$, 
	$\nbpatch_{\maxpool} = 7{\times}7$, $\nbfilter_{\maxpool} = 512$,
	\item $\dyn_{14}(\state, \ctrl) = \activ_{\Relu}(\biaffine_{\full}(\state, \ctrl))$ \newline
	with $\diminput_{\operatorname{out}} = 4096$,
	\item $\dyn_{15}(\state, \ctrl) = \activ_{\Relu}(\biaffine_{\full}(\state, \ctrl))$ \newline
	with $\diminput_{\operatorname{out}} = 4096$,
	\item $\dyn_{16}(\state, \ctrl) = \activ_{\softmax}(\biaffine_{\full}(\state, \ctrl))$ \newline
	with $\diminput_{\operatorname{out}} = 1000$.
	\item $\obj(\labpred) = \sum_{i=1}^{n}\loss_{\log}(\labpred_i, \lab_i)/n $ for $k = 1000$ classes.
\end{enumerate}

\end{document}